\pdfoutput=1
\documentclass{article}

\usepackage{arxiv}

\usepackage[utf8]{inputenc} % allow utf-8 input
\usepackage[T1]{fontenc}    % use 8-bit T1 fonts
\usepackage{url}            % simple URL typesetting
\usepackage{booktabs}       % professional-quality tables
\usepackage{amsfonts}       % blackboard math symbols
\usepackage{nicefrac}       % compact symbols for 1/2, etc.
\usepackage{microtype}      % microtypography

\usepackage[utf8]{inputenc}   % LaTeX, comprends les accents !
\usepackage[T1]{fontenc}      % Police contenant les caractÃ¨res franÃ§ais
\usepackage{comment}

\usepackage[tbtags]{amsmath}
\usepackage{amsthm}
\usepackage{bm}
\allowdisplaybreaks
\usepackage{amssymb,mathrsfs}
\usepackage{amsfonts}
\usepackage{upgreek}
\usepackage{geometry}
\usepackage[authoryear]{natbib}
\usepackage{graphicx}
\usepackage{float}
\usepackage{color}
\usepackage[colorlinks = true, citecolor = blue]{hyperref}
\usepackage{stmaryrd}

\usepackage[inline]{enumitem}
\usepackage{url}

\usepackage{graphicx} %insertion d'images
\usepackage{tikz}
\usepackage{wrapfig}
\usepackage{pgfplots}
\usepackage{xcolor}
\usepackage{bbm}
\usepackage{ifthen}
\usepackage{xargs}
\usepackage[textwidth=1.8cm]{todonotes}

\usepackage{aliascnt}
\usepackage{thmtools}
\usepackage{cleveref,nameref}
\usepackage{autonum}
\makeatletter
\newtheorem{theorem}{Theorem}
\crefname{theorem}{theorem}{Theorems}
\Crefname{Theorem}{Theorem}{Theorems}

\newaliascnt{llemma}{theorem}
\newtheorem{llemma}[theorem]{Lemma}
\aliascntresetthe{llemma}
\crefname{llemma}{lemma}{lemmas}
\Crefname{LLemma}{Lemma}{Lemmas}

\newaliascnt{pproposition}{theorem}
\newtheorem{pproposition}[theorem]{Proposition}
\aliascntresetthe{pproposition}
\crefname{pproposition}{proposition}{propositions}
\Crefname{Pproposition}{Proposition}{Propositions}

\newaliascnt{ccorollary}{theorem}
\newtheorem{ccorollary}[theorem]{Corollary}
\aliascntresetthe{ccorollary}
\crefname{ccorollary}{corollary}{corollaries}
\Crefname{Ccorollary}{Corollary}{Corollaries}

\newtheorem{assumptionS}{\textbf{HS}\hspace{-3pt}}
\Crefname{assumptionS}{\textbf{HS}\hspace{-3pt}}{\textbf{HS}\hspace{-3pt}}
\crefname{assumptionS}{\textbf{HS}}{\textbf{HS}}
\newtheorem{assumption}{\textbf{H}\hspace{-3pt}}
\Crefname{assumption}{\textbf{H}\hspace{-3pt}}{\textbf{H}\hspace{-3pt}}
\crefname{assumption}{\textbf{H}}{\textbf{H}}

\Crefname{assumptionSC}{\textbf{SC}\hspace{-3pt}}{\textbf{SC}\hspace{-3pt}}
\crefname{assumptionSC}{\textbf{SC}}{\textbf{SC}}

\newtheorem{assumptionF}{\textbf{F}\hspace{-3pt}}
\Crefname{assumptionF}{\textbf{F}\hspace{-3pt}}{\textbf{F}\hspace{-3pt}}
\crefname{assumptionF}{\textbf{F}}{\textbf{F}}

\Crefname{assumptionQC}{\textbf{QC}\hspace{-3pt}}{\textbf{QC}\hspace{-3pt}}
\crefname{assumptionQC}{\textbf{QC}}{\textbf{QC}}

\Crefname{assumptionBR}{\textbf{BR}\hspace{-3pt}}{\textbf{BR}\hspace{-3pt}}
\crefname{assumptionBR}{\textbf{BR}}{\textbf{SC}}

\Crefname{assumptionR}{\textbf{R}\hspace{-3pt}}{\textbf{R}\hspace{-3pt}}
\crefname{assumptionR}{\textbf{R}}{\textbf{R}}

\newtheorem{assumptionA}{\textbf{A}\hspace{-3pt}}
\Crefname{assumptionA}{\textbf{A}\hspace{-3pt}}{\textbf{A}\hspace{-3pt}}
\crefname{assumptionA}{\textbf{A}}{\textbf{A}}

\newtheorem{assumptionMD}{\textbf{MD}\hspace{-3pt}}
\Crefname{assumptionMD}{\textbf{MD}\hspace{-3pt}}{\textbf{MD}\hspace{-3pt}}
\crefname{assumptionMD}{\textbf{MD}}{\textbf{MD}}

\Crefname{assumptionMA}{\textbf{MA}\hspace{-3pt}}{\textbf{MA}\hspace{-3pt}}
\crefname{assumptionMA}{\textbf{MA}}{\textbf{MA}}

\Crefname{assumptionB}{\textbf{B}\hspace{-3pt}}{\textbf{B}\hspace{-3pt}}
\crefname{assumptionB}{\textbf{B}}{\textbf{B}}

\DeclareMathAlphabet{\mathpzc}{OT1}{pzc}{m}{it}

\renewcommand{\bar}[1]{\overline{#1}}

\DeclareFontFamily{U}{matha}{\hyphenchar\font45}
\DeclareFontShape{U}{matha}{m}{n}{
      <5> <6> <7> <8> <9> <10> gen * matha
      <10.95> matha10 <12> <14.4> <17.28> <20.74> <24.88> matha12
      }{}
\DeclareSymbolFont{matha}{U}{matha}{m}{n}
\DeclareFontSubstitution{U}{matha}{m}{n}

\DeclareFontFamily{U}{mathx}{\hyphenchar\font45}
\DeclareFontShape{U}{mathx}{m}{n}{
      <5> <6> <7> <8> <9> <10>
      <10.95> <12> <14.4> <17.28> <20.74> <24.88>
      mathx10
      }{}
\DeclareSymbolFont{mathx}{U}{mathx}{m}{n}
\DeclareFontSubstitution{U}{mathx}{m}{n}

\DeclareMathDelimiter{\vvvert}{0}{matha}{"7E}{mathx}{"17}

\usepackage{cancel}

\newcommand{\calBLmu}{{\calBLmu}}
\newcommand{\calBLnu}{{\calBLnu}}
\newcommand{\calBLkappa}{{\calBLkappa}}

\newcommand{\tlambda}{\tilde \lambda}

\DeclareMathOperator*{\argmin}{\arg\!\min}

\newcommand{\ie}{\textit{i.e.}}

\newcommand{\calO}{\mathcal{O}}

\newcommand{\bbE}{\mathbb{E}}

\newcommand{\rmD}{\mathrm{D}}

\newcommand{\rmT}{\mathrm{T}}

\newcommand{\bfA}{\mathbf{A}}

\newcommand{\bfV}{\mathbf{V}}

\newcommand{\scrH}{\mathscr{H}}
\newcommand{\scrR}{\mathscr{R}}

\def\msa{\mathsf{A}}

\def\msk{\mathsf{K}}
\def\msks{\mathsf{K}^{\star}}
\def\mss{\mathsf{S}}

\def\msc{\mathsf{C}}

\def\msh{\mathsf{H}}
\def\bmsh{\bar{\msh}}
\def\msm{\mathsf{M}}

\def\msx{\mathsf{X}}

\def\msy{\mathsf{Y}}

\newcommand{\mcb}[1]{\mathcal{B}(#1)}

\def\mcy{\mathcal{Y}}
\def\mcx{\mathcal{X}}

\def\mcf{\mathcal{F}}

\def\rset{\mathbb{R}}

\def\nset{\mathbb{N}}
\def\nsets{\mathbb{N}^*}

\def\rmd{\mathrm{d}}

\def\rmc{\mathrm{c}}

\def\rmU{\mathrm{U}}
\def\rmY{\mathrm{Y}}

\def\frX{\mathfrak{X}}
\def\frg{\mathfrak{g}}

\newcommand{\abs}[1]{\left\vert #1 \right\vert}
\newcommand{\absLigne}[1]{\vert #1 \vert}

\newcommandx{\psr}[3][3=]{\left\langle#1,#2 \right\rangle_{#3}}
\newcommandx{\normr}[2][2=]{ \left\Vert#1 \right\Vert_{#2}}
\newcommandx{\psrLigne}[3][3=]{\langle#1,#2 \rangle_{#3}}
\newcommandx{\normrLigne}[2][2=]{ \Vert#1 \Vert_{#2}}

\newcommandx{\norm}[2][1=]{\ifthenelse{\equal{#1}{}}{\left\Vert #2 \right\Vert}{\left\Vert #2 \right\Vert^{#1}}}
\newcommand{\normLigne}[2][1=]{\ifthenelse{\equal{#1}{}}{\Vert #2 \Vert}{\Vert #2\Vert^{#1}}}

\newcommand{\parenthese}[1]{\left(#1 \right)}
\newcommand{\parentheseLigne}[1]{(#1 )}
\newcommand{\parentheseDeux}[1]{\left[ #1 \right]}
\newcommand{\parentheseDeuxLigne}[1]{[ #1 ]}
\newcommand{\defEns}[1]{\left\lbrace #1 \right\rbrace }
\newcommand{\defEnsLigne}[1]{\lbrace #1 \rbrace }

\newcommand{\proba}[1]{\mathbb{P}\left( #1 \right)}
\newcommand{\PP}{\mathbb{P}}
\newcommand{\probaLigne}[1]{\mathbb{P}( #1 )}
\newcommand\probaMarkovTilde[2][2=]
{\ifthenelse{\equal{#2}{}}{{\widetilde{\mathbb{P}}_{#1}}}{\widetilde{\mathbb{P}}_{#1}\left[ #2\right]}}
\newcommand{\probaMarkov}[2]{\mathbb{P}_{#1}\left( #2\right)}
\newcommand{\PE}{\bbE} %j'ai l'impression qu'il manquait ça
\newcommand{\expe}[1]{\PE \left[ #1 \right]}

\newcommand{\expeLigne}[1]{\PE [ #1 ]}

\newcommand{\plusinfty}{+\infty}
\def\ie{\textit{i.e.}}

\def\eqsp{\;}
\def\eqand{\quad \text{  and  }\quad }

\newcommand{\coint}[1]{\left[#1\right)}
\newcommand{\ocint}[1]{\left(#1\right]}
\newcommand{\ooint}[1]{\left(#1\right)}
\newcommand{\ccint}[1]{\left[#1\right]}

\newcommand{\ocintLigne}[1]{(#1]}

\newcommand{\ball}[2]{\operatorname{B}(#1,#2)}
\newcommand{\cball}[2]{\overline{\operatorname{B}}(#1,#2)}

\newcommand{\boulefermee}[2]{\overline{\mathrm{B}}(#1,#2)}
\def\TV{\mathrm{TV}}
\newcommand\sequence[3][2=,3=]
{\ifthenelse{\equal{#3}{}}{\ensuremath{\{ #1_{#2}\}}}{\ensuremath{\{ #1_{#2}, \eqsp #2 \in #3 \}}}}
\newcommand\sequenceD[3][2=,3=]
{\ifthenelse{\equal{#3}{}}{\ensuremath{\{ #1_{#2}\}}}{\ensuremath{( #1)_{ #2 \in #3} }}}

\newcommand\sequenceDouble[4][3=,4=]
{\ifthenelse{\equal{#3}{}}{\ensuremath{\{ (#1_{#3},#2_{#3}) \}}}{\ensuremath{\{  (#1_{#3},#2_{#3}), \eqsp #3 \in #4 \}}}}

\newcommand{\wrt}{w.r.t.}

\def\iid{i.i.d.}

\def\eg{e.g.}
\def\Id{\operatorname{Id}}

\def\rmD{\mathrm{D}}%%rmd déjà pris

\def\bfe{\mathbf{e}}

\def\bigO{\mathcal{O}}

\def\trace{\operatorname{Tr}}

\def\diam{\operatorname{diam}}

\newcommand{\1}{\mathbbm{1}}

\def\loiGauss{\mathrm{N}}

\def\parallelTransport{\mathrm{T}}
\def\distT{\rho_{\Theta}}
\def\metricM{\mathfrak{g}}

\def\grad{\mathrm{grad}\,}
\def\Hess{\mathrm{Hess}\,}

\def\rmD{\mathrm{D}}

\def\noise{e}

\def\Exp{\mathrm{Exp}}

\def\planT{\rmT}
\def\Cut{\mathrm{Cut}}
\def\ID{\mathrm{ID}}
\def\dupgamma{\dot{\upgamma}}

\def\transpose{\top}

\def\sigmaZ{\sigma_0^2}
\def\sigmaU{\sigma_1^2}
\def\bfb{\mathbf{b}}

\newcommand{\beq}{\begin{equation}}
\newcommand{\eeq}{\end{equation}}

\def\proj{\operatorname{proj}}
\def\Leb{\mathrm{Leb}}
\def\Vdist{V_2}
\def\VHuber{V_1}
\def\thetas{{\theta^{\star}}}
\def\thetaspi{{\theta^{\star}_{\pi}}}

\def\vt{\mathpzc{V}}

\def\bnu{\bar{\nu}}
\def\bupeta{\bar{\upeta}}
\def\tb{\tilde{b}}
\def\cupeta{\check{\upeta}}
\def\gVar{g} %fonctº lisse sur la variété
\def\gPlan{\mathrm{g}} %fonctº lisse sur l'esp Tangent
\def\gInter{\mathpzc{g}} %fonctº lisse sur intervalle réel

\newcommand{\ceil}[1]{\left\lceil #1 \right\rceil}
\def\btheta{\bar{\theta}}
\def\tsigma{\tilde{\sigma}}
\def\tf{\tilde{f}}
\def\MKer{Q_{\upeta}}
\def\bu{\bar{u}}

\def\planTsT{\planT_{\thetas}\Theta}
\def\bouletan{\bar{\mathbb{B}}}
\def\SPD{\mathrm{Sym}}

\usepackage{caption}
\usepackage{subfig}
\title{On Riemannian Stochastic Approximation Schemes with Fixed Step-Size}

\author{Alain Durmus \\
    Centre Borelli, UMR 9010\\
	\'Ecole Normale Supérieure Paris-Saclay \\
    \texttt{alain.durmus@ens-paris-saclay.fr}\\
	\And \\
    \textbf{Pablo Jim\'enez} \\
	CMAP, UMR 7641\\
	\'Ecole Polytechnique \\
	\texttt{pablo.jimenez-moreno@polytechnique.edu} \\
	\And\\
    \textbf{\'Eric Moulines}\\
	CMAP, UMR 7641\\
	\'Ecole Polytechnique\\
	\texttt{eric.moulines@polytechnique.edu}\\
	\And \\
    \textbf{Salem Said}\\
	Laboratoire IMS, UMR 5218\\
	CNRS, Universit\'e de Bordeaux\\
	\texttt{salem.said@u-bordeaux.fr}
}

\begin{document}
%\runningauthor{Durmus, Jiménez, Moulines, Said}

\maketitle
%\twocolumn[

% The \author macro works with any number of authors. There are two commands
% used to separate the names and addresses of multiple authors: \And and \AND.
%
% Using \And between authors leaves it to LaTeX to determine where to break the
% lines. Using \AND forces a line break at that point. So, if LaTeX puts 3 of 4
% authors names on the first line, and the last on the second line, try using
% \AND instead of \And before the third author name.

\begin{abstract}%
This paper studies fixed step-size stochastic approximation (SA) schemes,
including stochastic gradient schemes, in a Riemannian framework. It is
motivated by several applications, where geodesics can be computed
explicitly, and their use accelerates crude Euclidean methods. A fixed
step-size scheme defines a family of time-homogeneous Markov chains, 
parametrized by the step-size. Here, using this formulation, non-asymptotic
performance bounds are derived, under Lyapunov conditions. Then,
for any step-size, the corresponding Markov chain is proved to admit a unique stationary
distribution, and to be geometrically ergodic.  This result gives rise
to a family of stationary distributions indexed by the step-size, which is further shown to 
converge to a Dirac measure, concentrated at the solution of the problem at hand, as 
the step-size goes to $0$.  Finally, the asymptotic rate of this convergence is established,
  through an asymptotic expansion of the bias, and a central limit theorem. %Our results are illustrated by ...
\end{abstract}

\section{INTRODUCTION}

This paper deals with the study of fixed step-size Stochastic Approximation
(SA) algorithms~\citep{robbins:monro:1951,kushner:yin:2003,pj}, defined on a
Riemannian manifold $\Theta$ with metric $\metricM$. Specifically,
consider the problem
\begin{equation}
\begin{aligned}
	\label{eq:sa_pb}
	&\text{find } \theta \in \Theta \text{ satisfying }h(\theta) = 0 \eqsp,\\
	&\text{for a vector field $h : \Theta  \to \planT \Theta$} \eqsp,
\end{aligned}
\end{equation}
where $\planT \Theta $ denotes the tangent bundle of $\Theta$, and $h$
is only accessible through an oracle returning noisy estimates. 
The setting where $h = -\grad f$ is of particular interest
for minimizing a smooth function $f : \Theta \to \rset$.
In the Euclidean setting, Stochastic Gradient Descent (SGD) and its
variants are now common methods for solving this problem
\citep{bottou:2010,bottou:bousquet:2008}. However, it should be
stressed that \eqref{eq:sa_pb} encompasses several other applications
in stochastic optimization, reinforcement learning or maximum
likelihood estimation, such as online Expectation Maximization
algorithms \citep{cappe:moulines:2009}, policy gradient
\citep{baxter:bartlett:2001} or Q-learning \citep{jaakkola:1993}.
Minimization over a Riemannian manifold or its general formulation
\eqref{eq:sa_pb} arises in many applications: Principal Component
Analysis \citep{edelman:arias:smith:1998}, dictionary recovery
\citep{sun:qu:wright:2017}, matrix completion
\citep{boumal:absil:2011}, smooth semidefinite programs
\citep{boumal:voroninski:bandeira:2016}, tensor factorization
\citep{ishteva:et:al:2011}, and Riemannian barycenter estimation
\citep{saidmanton:2019,arnaudon:2012}. This has motivated the development of a
comprehensive framework for stochastic optimization problems on
Riemannian manifolds. One of the first contributions in this field is 
\cite{bonnabel2013stochastic}, which derives
asymptotic convergence results for SA on Riemannian manifolds. 
Non-asymptotic results are obtained by \cite{sra:2016} for a
geodesically convex function $f$. This study has been followed and
completed by \cite{zhang2016riemannian,sato2019riemannian} which
introduce and analyze a Riemannian counterpart of the Stochastic
Variance Reduced Gradient (SVRG) algorithm.  Since then, many existing
methods or results from the Euclidean case have been considered in a
Riemannian setting. For example, \cite{khuzani2017stochastic}
suggest a Riemannian stochastic primal-dual algorithm and most
recently \cite{flammarion:2018} study an averaged version of
Riemannian SGD. 
% Finally,
% \cite{criscitiello2019escapingsaddles,sun2019escaping} shows that
% introducing appropriate noise in Riemannian gradient  escaping from saddle
% points.
% Riemannian stochastic optimization
%moulines non-asymptotic SA \cite{karimi:2019}

%\cite{bonnabel2013stochastic} obtains the
%first asymptotic convergence result for stochastic gradient descent in this setting, which is further
%extended by Tripuraneni et al. (2018); Zhang et al. (2016); Khuzani and Li (2017). If the problem is
%non-convex, or the Riemannian Hessian is not positive definite, one can use second order methods
%to escape from saddle points. Boumal et al. (2016a) shows that Riemannian trust region method
%converges to a second order stationary point in polynomial time (see, also, Kasai and Mishra, 2018; Hu
%et al., 2018; Zhang and Zhang, 2018). But this method requires a Hessian oracle, whose complexity is
%d times more than computing gradient.

In this paper, we are interested in the study of fixed step-size SA methods of the form
\begin{equation} \label{eq:scheme}
	\begin{aligned}
  &\theta_{n+1} = \proj_\mss\parentheseDeux{\Exp_{\theta_n}\defEns{\upeta H_{\theta_{n}}(X_{n+1})}}\eqsp, \eqsp \\
  &\text{where} \eqsp H_{\theta_{n}}(X_{n+1}) = h (\theta_n) + \noise_{\theta_n}(X_{n+1}) \eqsp.
  \end{aligned}
\end{equation}
In \eqref{eq:scheme}, $\upeta >0$ is a step-size,
$(X_n)_{n \in\nsets}$ is an $(\mcf_n)_{n \in \nset}$-adapted process,
defined on a filtered probability space, with values in a
measurable space $(\msx,\mcx)$, and
$\noise:\Theta \times \msx \to \planT \Theta$ is a measurable function,
such that $\theta \mapsto e_{\theta}(x)$ is a
vector field over $\Theta$, for any $x \in \msx$. In addition,
$\Exp_\theta:\planT_\theta \Theta \to \Theta$ is the Riemannian
exponential mapping and $\proj_{\mss} : \Theta \to \mss$ is a
projection-like operator onto a subset
$\mss \subset \Theta$.  This recursion is a natural extension of
Euclidean SA, akin to the Robbins-Monroe algorithm, in a Riemannian
setting. 

In the Euclidean setting, the study of fixed step-size SA, and in particular SGD, has recently
attracted much attention, see \eg~\cite{ma2018power,vaswani2019fast,dieuleveut:2017,bach2020effectiveness,bachmoulines2011}. Indeed,
first of all, the step-size $\upeta$ is the only parameter to tune, in
contrast to the case where a decreasing sequence of step-sizes is used
in \eqref{eq:scheme}. Furthermore, % choosing a suitable step-size
% $\upeta$, the computational complexity of fixed-step size SGD can be
% better than in the decreasing step-size setting, since
the forgetting of the 
 initial condition is exponentially fast
\citep{nedic2001convergence,needell:ward:srebo;2014}.

We aim to show, in a general Riemannian framework,
that the use of \eqref{eq:scheme} provides a good solution for
\eqref{eq:sa_pb}. To this end, we establish non-asymptotic and
asymptotic properties of $(\theta_n)_{n\in\nset}$, in the limit
$\upeta \to 0$. 
Our contributions can be summarized as follows.
\begin{enumerate}[wide, labelwidth=!, labelindent=0pt,label=(\arabic*),noitemsep,nolistsep]
\item We derive non-asymptotic bounds, for the convergence of
  $(\theta_n)_{n\in\nset}$ to approximate solutions of \eqref{eq:sa_pb}, under
  general Lyapunov assumptions and mild assumptions on the manifold $\Theta$ and the subset $\mss$.
\item Under additional regularity conditions, we show that
  $(\theta_n)_{n\in\nset}$, as a Markov chain, admits a unique
  stationary distribution $\mu^{\upeta}$ and is geometrically ergodic, \ie~converges to $\mu^{\upeta}$ exponentially fast. 
\item We study the limiting behavior of the family
  $(\mu^{\upeta})_{\upeta >0}$ as $\upeta \to 0$. In particular, we
  show that if \eqref{eq:sa_pb} admits a unique solution $\thetas$ and
  other suitable conditions hold, this family converges to the Dirac
  measure at $\thetas$. In addition, we asymptotically quantify this
  convergence, through a central limit theorem. Precisely, we
  prove that after a $\upeta^{-1/2}$-rescaling, this family of
  stationary distributions converges weakly to a normal distribution
  as $\upeta \to 0$. These results illustrate the exponential
  forgetting of initial condition of the scheme and that, at
  stationarity, the iterates $(\theta_n)_{n\in\nset}$ stay in a
  $\bigO(\upeta^{1/2})$-neighborhood of $\thetas$. In addition, they
  can be understood as generalizations to Riemannian spaces of
  \citet[Theorem 1]{pflug:1986} and  \citet[Theorem 4]{dieuleveut:2017}.
\item We apply our results to SGD. In particular, we
  establish the first non-asymptotic convergence bounds for strongly
  geodesically convex functions, without boundedness assumptions on the
manifold $\Theta$. 
\item Finally, we introduce and prove the convergence of an SGD scheme to compute 
	the Riemannian barycenter, also known as the Karcher mean, of 
	distributions on Hadamard manifolds. To the authors' 
	knowledge, our contribution on this topic is one of the few without 
	boundedness assumptions on the distribution.
\end{enumerate}
In the derivation of our results, we use crucially the fact that
$(\theta_n)_{n\in\nset}$ defines a Markov
chain in $\Theta$, under mild conditions. This interpretation has been successfully used in several
papers dealing with the convergence of SA or SGD in Euclidean spaces; see \eg~
\cite{benveniste:metivier:priouret:1990,kushner:1981,fort:1999,pflug:1986}.

We consider a more general setting and milder conditions
in comparison with most other studies in the field.  Indeed, most papers do
not consider the general SA framework, but only the case $h = -\grad f$,
dealing with SGD and its variants.  To the authors' knowledge, only
\cite{bonnabel2013stochastic,durmus:jimenez:moulines:said:wai:2020}
tackle the general SA problem \eqref{eq:sa_pb}. Our main contribution, compared to these two works, is to
deal with the fixed step-size setting. Besides, our study considers 
general geodesically complete Riemannian manifolds which encompass Hadamard
spaces, which have been the primary focus for \cite{sra:2016,zhang2016riemannian,flammarion:2018}.

Furthermore, a majority of the previous studies %
on SGD in a Riemannian space (see
\eg~\cite{sra:2016,zhang2016riemannian,flammarion:2018,alimisis:2020,han:gao:2020}),
are purely local in nature, because of the assumption that
$(\theta_n)_{n\in\nset}$ stays almost surely in a (fixed and
deterministic) compact and geodesically convex subset of $\Theta$. For
example, note that all the convergence results derived in
\cite{sra:2016} depend on the diameter of the compact in which
$(\theta_n)_{n\in\nset}$ is assumed to stay. This assumption rarely
holds in practice, and is quite difficult to verify in theory. It
strongly limits the applicability of many results in the literature
over the past few years.  On the contrary, our results do not suffer
from this problem, and can all be applied either on a compact or
non-compact Riemannian manifold. As a result, we consider a new SA
method to estimate the Karcher mean of a distribution $\pi$ on
$\Theta$, see \cite{arnaudon:2012,le:2004,sra:2016,iannazzo:2018},
for which we derive non-asymptotic convergence bounds
without boundedness conditions on the support of $\pi$.

\textbf{Notations}
For any $\theta \in \Theta$ and $v,w \in \planT_\theta \Theta$, denote
by $\frg_\theta(v,w)=\psr{v}{w}[\theta]$ and its corresponding norm by
$\frg_\theta(v,v)=\normr{v}[\theta]^2$.
$\distT:\Theta \times \Theta \to \rset_+$ denotes the distance
associated with the Riemannian metric $\frg$.  For any
$\theta_0 \in \Theta,r>0$, set
$\ball{\theta_0}{r}=\defEnsLigne{\theta_1 \in \Theta \,:\,
  \distT(\theta_0,\theta_1)<r}$, the open ball centered at $\theta_0$
with radius $r$. Similarly, we define closed balls in $\Theta$ by
$\cball{\theta_0}{r}=\defEnsLigne{\theta_1 \in \Theta \,:\,
  \distT(\theta_0,\theta_1)\leq r}$.

For a smooth function $g:\Theta \to \rset$, we denote by $\grad g$ its Riemannian gradient \cite[p. 27]{lee:2019} and by $\Hess g$ its Riemannian, or covariant, Hessian \cite[Example 4.22]{lee:2019}. 
For a curve $\upgamma:I \to \Theta, \parallelTransport_{t_0,t_1}^\upgamma : \planT_{\upgamma(t_0)} \Theta \to \planT_{\upgamma(t_1)} \Theta$ stands for the parallel transport map associated to the Levi-Civita connection along $\upgamma$ from $\upgamma(t_0)$ to $\upgamma(t_1)$ \cite[Equation 4.22]{lee:2019}. 
Moreover, for any $\theta\in \Theta$, under the assumption that $\Theta$ is complete, 
consider the Riemannian exponential map $\Exp_{\theta} : \planT_\theta \Theta \to \Theta$,
see \citet[Proposition 5.19]{lee:2019}.
This map projects a vector from
the tangent space $\planT_\theta\Theta$ onto the manifold $\Theta$, following a geodesic curve.

\vspace{-0.3cm}

%\vspace{-0.15cm}

% \vspace{-0.2cm}

%%% Local Variables:
%%% mode: latex
%%% TeX-master: "main"
%%% End:

% \vspace{-0.1cm}
% \vspace{-1.1cm}
% \input{notations}
%\pagebreak
%
\section{CONSTANT STEPSIZE ANALYSIS FOR A CONSTRAINED SCHEME}\label{sec:constant_step}

\vspace{-0.3cm}

%\vspace{-0.15cm}

\subsection{Main Results}
In this section, we study the Stochastic Approximation scheme
\eqref{eq:scheme}, which is constrained on a subset 
$\mss \subset \Theta$. The following assumption on the manifold $\Theta$ and $\mss$ is
considered all along this paper and allows us to rigorously define $\proj_{\mss}$.
\begin{assumptionA}\label{ass:had_or_complete}
Assume one of the following conditions.
\begin{enumerate}[wide, labelwidth=!, labelindent=0pt,label=(\roman*),noitemsep,nolistsep]
\item \label{ass:had_or_complete_i} $\Theta$ is a Hadamard manifold, \ie~a complete, simply connected Riemannian manifold with non-positive sectional curvature. In addition, $\mss$ is a closed geodesically convex subset of $\Theta$ with non-empty interior.
\item \label{ass:had_or_complete_ii} $\Theta$ is a complete, connected Riemannian manifold and $\mss=\Theta$.
\end{enumerate}
\end{assumptionA}
Note that under \Cref{ass:had_or_complete}, the exponential map
$\Exp : \planT \Theta \to \Theta$ is well-defined, see \citet[Theorem
6.19]{lee:2019}.  Under
\Cref{ass:had_or_complete}-\ref{ass:had_or_complete_i},
\citet[Proposition 2.6]{sturm:2003} shows that there exists
$\proj_{\mss} : \Theta \to \mss$ which is the Riemannian counterpart
of the Euclidean projection onto a closed convex subset. More precisely,
$\proj_{\mss}$ is the unique mapping from $\Theta$ to $\mss$ such
that for any $\theta \in \Theta$,
$\distT(\proj_{\mss}(\theta),\theta) = \inf_{\theta' \in \mss}
\distT(\theta',\theta)$. Under
\Cref{ass:had_or_complete}-\ref{ass:had_or_complete_ii}, we simply set
$\proj_{\mss} =
\Id$. % , $\distT^2(\theta,\tilde{\theta}) \geq \distT^2(\theta,\proj_{\mss}(\theta)) + \distT^2(\proj_{\mss}(\theta),\tilde{\theta})$ and $.

Recall that the recursion \eqref{eq:scheme} only uses a noisy estimate $H_\theta$ of the mean field $h(\theta)$, for any $\theta \in \Theta$. We assume the following conditions on the noise to ensure convergence.
\begin{assumptionMD}
  \label{ass:0mean_noise}
  The sequence $(X_n)_{n \in \nsets}$ is  independent and identically distributed (\iid). In addition, for any $\theta \in\Theta$, $\expe{\noise_{\theta}(X_{1})} = 0$ and there exist $\sigmaZ,\sigmaU > 0$ such that 
for any $\theta \in \mss$, $\expeLigne{\normrLigne{\noise_{\theta}\parenthese{X_{1}}}[\theta]^2 } \leq \sigmaZ + 
\sigmaU\, \normr{h(\theta)}[\theta]^2$.
% \begin{equation}
% \expe{\normr{\noise_{\theta}\parenthese{X_{1}}}[\theta]^2 } \leq \sigmaZ + 
% \sigmaU\, \normr{h(\theta)}[\theta]^2 \eqsp.
% \end{equation}
\end{assumptionMD}
\Cref{ass:0mean_noise} is referred to as the martingale difference setting which implies that 
$(\theta_n)_{n \in \nset}$ is a time-homogeneous $(\mcf_n)_{n \in \nset}$-Markov chain,
%For any $\upeta>0$, 
for which we denote by $Q_{\upeta}$ its corresponding Markov kernel.
%of this chain on $\mss \times\mcb{\mss}$.
\begin{assumptionMD}
  \label{ass:MD:topo_prop_chain}
  \begin{enumerate}[wide, labelwidth=!, labelindent=0pt,label=(\roman*),noitemsep,nolistsep]
\item \label{ass:item:feller}
$\PP$-almost surely,  the vector field $\theta \mapsto \noise_{\theta}(X_1)$ is  continuous on $\Theta$. 
\item  \label{ass:item:irreducible_aperiodic}
 For any $\theta \in \Theta$, $\Leb_\theta$ and the distribution of $\noise_\theta(X_1)$ are mutually absolutely continuous, where $\Leb_\theta$  stands for the Lebesgue measure on $\planT_\theta \Theta$.
  \end{enumerate}
\end{assumptionMD}

\Cref{ass:MD:topo_prop_chain} ensures topological and aperiodicity
properties of the Markov chain under consideration. This condition
is used in the study of the limiting behaviour of $(\theta_n)_{n \in\nset}$. Note that the condition
\Cref{ass:MD:topo_prop_chain}-\ref{ass:item:irreducible_aperiodic} is
automatically satisfied adding some Gaussian noise,
\ie, when $\noise_{\theta}(X_i)$ is replaced by
$\noise_{\theta}(X_i) + \mathrm{p}_{\theta}(Z_i)$ where for any
$\theta \in \Theta$, $\mathrm{p}_{\theta}$ is any invertible linear
application from $\rset^d$ to $\planT_{\theta}\Theta$ and
$(Z_i)_{i\in\nsets}$ is a sequence of \iid~$d$-dimensional Gaussian
random variables with zero-mean and covariance matrix identity.

%Under \Cref{ass:0mean_noise}-\ref{ass:item:martin}, 
%, as highlighted by \Cref{lem:feller,lem:irreducible_aperiodic}. 

% \begin{assumptionA} \label{ass:markov}
% \begin{enumerate}[wide, labelwidth=!, labelindent=0pt,label=(\roman*),noitemsep,nolistsep]
% \end{enumerate}
% \end{assumptionA}

To ensure recurrence of $(\theta_n)_{n \in \nset}$, we assume the existence of a Lyapunov function $V : \Theta \to \rset_+$ for the mean vector field $h$.
\begin{assumption}\label{ass:lyap}
\begin{enumerate}[wide, labelwidth=!, labelindent=0pt,label=(\roman*),noitemsep,nolistsep]
\item \label{ass:lyap_contractive} For any $\theta\in \Theta$, $V\circ\proj_\mss(\theta) \leq V(\theta)$.
\item \label{ass:lyap_grad_lips} $V$ is continuously differentiable on $\Theta$ and its Riemannian gradient $\grad V$ is geodesically $L$-Lipschitz, \ie, there exists $L \geq 0$ such that for any $\theta_0,\theta_1 \in \Theta$, and geodesic curve $\upgamma:[0,1]\to \Theta$ such that $\upgamma(0)=\theta_0$ and $\upgamma(1)=\theta_1$, 
\begin{equation}\label{eq:llipschitz}
\normr{\grad V(\theta_1)- \parallelTransport_{01}^{\upgamma} \grad V(\theta_0)}[\theta_1] \leq L \ell(\upgamma) \eqsp ,
\end{equation}  
where $\ell(\upgamma)=\normrLigne{\dot{\upgamma}(0)}[\theta_0]$ is the length of the geodesic.
\item \label{ass:item:lyap_proper} $V$ is proper on $\mss$,
  \ie, for any $M \geq 0$, there exists a compact set
  $\msk \subset \mss$ such that for any $\theta \in \mss \setminus \msk$, $V(\theta) > M$. 
\end{enumerate}
\end{assumption}
\begin{assumption}
  \label{ass:lyap_meanfield}
  There exist $C_1 \geq 0$ and $C_2>0$ such that for any $\theta \in \mss$, $\normr{h(\theta)}[\theta]^2 + C_2 \psr{\grad V(\theta)}{h(\theta)}[\theta] \leq C_1$.
\end{assumption}
In addition, to quantify the convergence of $(\theta_n)_{n \in\nset}$ in a neighborhood of a solution of \eqref{eq:sa_pb}, we consider the following condition for some compact set $\msks \subset \mss$.
% \begin{assumption}[$\msks$] \label{ass:lyap_minorization_0}  
% %There exists $\theta^{\star} \in \mss$ such that the following hold. 
%  There exists $c_h>0$  such that for any $\theta \in \mss$, $\psr{\grad V(\theta)}{h(\theta)}[\theta] \leq - c_h\normr{h(\theta)}[\theta]^2\1_{\mss \setminus \msks}(\theta)$.
% \end{assumption}
\begin{assumption}[$\msks$] \label{ass:lyap_minorization}  
%There exists $\theta^{\star} \in \mss$ such that the following hold. 
 % \begin{enumerate}[wide, labelwidth=!, labelindent=0pt,label=(\roman*),noitemsep,nolistsep]
%\item \label{ass:lyap_attractive}
  There exists $\lambda>0$  such that for any $\theta \in \mss$, $\psr{\grad V(\theta)}{h(\theta)}[\theta] \leq - \lambda V(\theta)\1_{\mss \setminus \msks}(\theta)$.
% \begin{equation}
%  \eqsp.
% \end{equation}
%  \item   \label{ass:lyap_minorization_ii}  
 % \end{enumerate}
\end{assumption}
Note that under \Cref{ass:lyap_minorization}$(\emptyset)$, if $h(\theta)=0$, %$\normr{h(\theta)}[\theta]= 0$,
then $V(\theta)=0$ since $V$ is a nonnegative function. 

It is relevant to recognize that \Cref{ass:lyap},
\Cref{ass:lyap_meanfield} and \Cref{ass:lyap_minorization} boil down
to standard stability and recurrence conditions; see
\eg~\cite{benveniste:metivier:priouret:1990,duflo:1997}.  In the
Euclidean case when we assume the uniqueness of a solution $x^\star$,
a common choice for $V$ is $x \mapsto \norm{x-x^\star}^2$. However,
the square distance is no longer a suitable candidate in non-compact Riemannian
settings, and therefore selecting a
Lyapunov function adapted to the manifold $\Theta$ and the geometry of
the mean field $h$ is all the more important.  Note that
\Cref{ass:lyap}-\ref{ass:item:lyap_proper} is automatically satisfied
if $\mss$ is compact.
%Note also that we assume that  \Cref{ass:lyap_minorization}$(r)$  holds for $\thetas\in \Theta$ such that $h(\thetas)=0$.
In addition, in most cases $\msks$ and $V$ are chosen such that
$\msks = \emptyset$ or $\{ \theta \in \mss \, : \, \normrLigne{h(\theta)}[\theta] \leq
\varepsilon \}$ for some $\varepsilon \geq 0$, $-C_2 \psrLigne{h(\theta)}{\grad V(\theta)}[\theta] \geq \normrLigne{h(\theta)}[\theta]^2$ for some $C_2 >0$ and any $\theta \in\Theta$, and therefore \Cref{ass:lyap_meanfield} is satisfied with $C_1=0$.

The use of Lyapunov functions is really common and widespread to
analyze stochastic approximation schemes, see
\cite{kushner:yin:2003,kushner:1981,duflo:1997}.  However, compared to
the Euclidean setting, the square distance cannot be used in many
situations because it does not satisfy
\Cref{ass:lyap}-\ref{ass:lyap_grad_lips}. This brought us to consider
a different Lyapunov function and therefore develop an adapted
framework for the Riemannian case; see \Cref{sec:two-exampl-lyap}
hereafter for more details.

We start with our first result which is established along with all the other statements of this section in the supplement \Cref{app:constant_step}.			%\pablo{ici}
%\alaini{faire premier résultat (a) plutôt en $\psr{\grad V}{h}$ sans H2 puis avec H2 (b) et (c)}
\begin{theorem}\label{theo:drift_lyap}
 Assume \Cref{ass:had_or_complete}, \Cref{ass:0mean_noise},
  \Cref{ass:lyap}-\ref{ass:lyap_contractive}-\ref{ass:lyap_grad_lips},  \Cref{ass:lyap_meanfield}.
  \begin{enumerate}[wide, labelwidth=!, labelindent=0pt,label=(\alph*),noitemsep,nolistsep]
  \item \label{theo:drift_lyap0} Suppose in addition that for any $\theta \in \mss,  \psr{\grad V(\theta)}{h(\theta)}[\theta] \leq 0$. Then, for any $\upeta\in (0,\bupeta]$, $\theta_0 \in \mss$, and $n \in \nsets$,
{\small  \begin{equation}\label{eq:theo:drift_lyap0_h}
%	  \begin{aligned}
  n^{-1} \sum_{k=0}^{n-1} \expe{-\psr{\grad V (\theta_k)}{h(\theta_k)}[\theta_k]}  \leq
 2V(\theta_0)/(n \upeta) +  \upeta  b \eqsp,
\end{equation}
}

\vspace{-0.4cm}
\noindent
where $(\theta_n)_{n \in\nset}$ is defined by \eqref{eq:scheme} starting from $\theta_0$, 
$\bupeta = [2 C_2 L (1+\sigmaU)]^{-1} , b=2L\{\sigmaZ +C_1 (1+ \sigmaU)\}$.
  \end{enumerate}
  Suppose in addition that \Cref{ass:lyap_minorization}$(\msks)$ holds for
  some compact set $\msks \subset \mss$.
%  , and define
%  $\norm{V}_{\msks} = \sup \defEnsLigne{V(\theta) \,:\, \theta \in
%    \msks}$ if $\msks \neq \emptyset$ and
%  $\norm{V}_{\msks} = 0$ otherwise.
%  \alain{remplacer $h$ par $V$}
  \begin{enumerate}[wide, labelwidth=!, labelindent=0pt,label=(\alph*),noitemsep,nolistsep,resume]  
  \item \label{theo:drift_lyap1}  Then for any $\upeta \in (0,\bar{\upeta}]$, $\theta_0 \in \mss$, and $n \in \nsets$,
    {\small
  \begin{equation}\label{eq:theo::drift_lyap1_h}
n^{-1} \sum_{k=0}^{n-1} \PE[\1_{\mss \setminus \msks}(\theta_k)V(\theta_k)] \leq 
V(\theta_0)/(a n\upeta)+ \upeta b/(2a) \eqsp, 
\end{equation}
}

\vspace{-0.5cm}
\noindent
where $a=\lambda/2$.
%where $(\theta_n)_{n \in\nset}$ is defined by \eqref{eq:scheme} starting from $\theta_0$, $a=\lambda/2,b=L\defEnsLigne{\sigmaZ + (1+\sigmaU)\parentheseDeux{C_1+C_2\norm{V}_{\msks}}}$ and $\bar{\upeta} = \lambda / [2(1+\sigmaU)C_2L]$.
\item \label{theo:drift_lyap2}
Define $\norm{V}_{\msks} = \sup
\defEnsLigne{V(\theta) \,:\, \theta \in
\msks}$ if $\msks \neq \emptyset$ and
$\norm{V}_{\msks} = 0$ otherwise.
Then for any $\upeta \in (0,\bar{\upeta}]$, $\theta_0 \in \mss$, and any $n \in \nsets$,
%\begin{equation}\label{eq:thep_drift_lyap2_1step}
%\PE[V(\theta_1)] \leq \{1-\upeta a\}V(\theta_0) 
%+ \upeta(a \normr{V}[\msks]+b\upeta/2) \eqsp,
%\end{equation}
%and, for any $n \in \nsets$,
  \begin{equation}\label{eq:theo:drift_lyap2_h}
\PE[V(\theta_n)] \leq \defEns{1 - \upeta a }^{n}V(\theta_0) + \norm{V}_{\msks} + \upeta b/(2a) \eqsp.
\end{equation}
\end{enumerate}
\end{theorem}
%\alain{faire un petit commentaire} 
Note that \Cref{theo:drift_lyap}
gives, in the case $\msks=\emptyset$, non-asymptotic bounds of
order $\upeta$ on
$n^{-1} \sum_{k=0}^{n-1} \PE[-\psr{\grad
  V(\theta_k)}{h(\theta_k)}[\theta_k]]$,
$n^{-1} \sum_{k=0}^{n-1} \PE[V(\theta_k)]$ and $\PE[V(\theta_n)]$ 
as $n \to \plusinfty$. In addition, the
forgetting of the initial condition in \eqref{eq:theo:drift_lyap0_h}
and \eqref{eq:theo::drift_lyap1_h} is linear \wrt~$n$, contrary to
\eqref{eq:theo:drift_lyap2_h} where it is exponential.  A statement
similar to \Cref{theo:drift_lyap}-\ref{theo:drift_lyap1} holds only
assuming
\Cref{ass:lyap}-\ref{ass:lyap_contractive}-\ref{ass:lyap_grad_lips}
and replacing \Cref{ass:lyap_minorization}$(\msks)$ by the condition that there exists $\lambda>0$ such that for any
$\theta \in \mss$,
$\psrLigne{\grad V(\theta)}{h(\theta)}[\theta] \leq - \lambda
\normrLigne{h(\theta)}[\theta]^2\1_{\mss \setminus \msks}(\theta)$. This
result is postponed to the supplement
\Cref{theo:alter_theo:drift_lyap}-\Cref{sec:an-altern-crefth}. \Cref{theo:drift_lyap}-\ref{theo:drift_lyap0}
is a generalization of \citet[Lemma 7]{hosseini2019alternative} for SGD
under a general Lyapunov condition and milder assumptions. We show in \Cref{sec:app_sgd}, how this generalization can be applied to SGD
to obtain better convergence guarantees. Finally, in the same Section,
we show that
\Cref{theo:drift_lyap}-\ref{theo:drift_lyap1}-\ref{theo:drift_lyap2}
can be used to derive non-asymptotic convergence bounds for SGD
applied to a geodesically strongly convex function, without any
boundedness assumptions on $\Theta$.

The study of the asymptotic behavior of $(\theta_n)_{n \in \nset}$ is
the second step towards understanding the quality of the approximation
to the solution of \eqref{eq:sa_pb}. We now show, under suitable
assumptions and for $\upeta \leq \bupeta$ given in \Cref{theo:drift_lyap}, first, that the
chain is ergodic and admits a unique invariant distribution, and
second, that this measure converges weakly to the Dirac measure at
some point $\thetas$, as the stepsize of the scheme goes to
zero. In other words, the family of stationary distributions
$(\mu^\upeta)_{\upeta \in (0,\bupeta]}$ concentrates around $\thetas$
as $\upeta \to 0$. Possible approximations of $\thetas$ are therefore
derived from sampling from $\mu^\upeta$ or taking its Riemannian
barycenter, for a small enough $\upeta$. If the sequence
$(\theta_n)_{n \in \nset}$ is ergodic, then as $n\to \plusinfty$ the
marginal distributions of this Markov chain converge to
$\mu^{\upeta}$ and can be used in turn as proxy to solve
\eqref{eq:sa_pb}.  A remaining question is to provide an estimate of
the approximation error as a function of the step-size $\upeta$. This is tackled in
\Cref{sec:unconstrained}.

%\alaini{expliquer pourquoi on s'intéresse à $\theta_n$: comportement asymptotique et conclusion par rapport au problème de départ \eqref{eq:sa_pb}}

%We can now state our first result.
\begin{theorem}
  \label{thm:recurrent_ergodic} 
  Assume \Cref{ass:had_or_complete}, \Cref{ass:0mean_noise},
  \Cref{ass:MD:topo_prop_chain}, \Cref{ass:lyap},
  \Cref{ass:lyap_meanfield} and \Cref{ass:lyap_minorization}$(\msks)$
  for some compact set $\msks \subset \mss$. Let
  $\upeta\in (0,\bar{\upeta}]$ where
  $\bupeta = [2 C_2 L (1+\sigmaU)]^{-1}$. Then,
  $(\theta_n)_{n \in\nset}$ defined by \eqref{eq:scheme} admits a
  unique stationary distribution $\mu^{\upeta}$ and is
  Harris-recurrent.  In addition, there exist $\rho \in \coint{0,1}$
  and $C \geq 0$ such that for any $\theta_0 \in \mss$ and
  $k \in \nset$,
  $\abs{\expeLigne{g(\theta_n)} - \int_{\Theta} g(\theta) \rmd
    \mu^{\upeta}(\theta)} \leq C \rho^n (1+V(\theta_0))$, for any
  measurable function $g : \Theta \to \rset$ satisfying $\sup_{\theta \in \Theta} \{\abs{g}/V\} \leq 1$.
\end{theorem}
Taking $n \to \plusinfty$ in \Cref{theo:drift_lyap}-\ref{theo:drift_lyap2}, we obtain  by \Cref{thm:recurrent_ergodic} that
\begin{equation}
  \label{eq:bound_bias}
\textstyle{\abs{\int_{\Theta} g(\theta) \rmd \mu^{\upeta}(\theta) } \leq  \norm{V}_{\msks} + \upeta b/(2a)} \eqsp,
\end{equation}
for any measurable function $g : \Theta \to \rset$ satisfying
$\sup_{\theta \in \Theta} \{\abs{g}/V\} \leq 1$. In the case
$\norm{V}_{\msks} =0$ (then $V(\theta) = 0$ for any
$\theta \in \msks$), we get
$\int_{\Theta} V(\theta) \rmd \mu^{\upeta}(\theta) \leq \upeta
b/(2a)$. Therefore, this result indicates that the family
$\{\mu^{\upeta} \, :\, \upeta \in \ocint{0,\bupeta}\}$ concentrates in a
$\bigO(\upeta)$-neighborhood of $\msks$ as $\upeta \to 0$. In particular, if $V$ admits
a unique zero $\thetas$ which corresponds in many applications to a
solution of \eqref{eq:sa_pb}, then we can expect that $\{\mu^{\upeta} \, : \, \upeta\in \ocint{0,\bupeta}\}$ converges in distribution to $\updelta_{\thetas}$,  the Dirac measure at $\theta^\star$, as $\upeta \to 0$. 
The specific additional conditions to obtain such a result are the following. 
% \begin{proof}
%   The proof is postponed to \Cref{app:recurrent_ergodic}.
% \end{proof}
% Our goal now is to understand the limiting behavior of  $(\mu^\upeta)_{\upeta \in \ocint{0,\bar{\upeta}}}$ as $\upeta$ goes to zero.% The following result gives the first insight on this question, which is as much as we can expect without 
% renormalizing factors.
\begin{assumption}
  \label{ass:V_positive}
There exists
  $\thetas \in \mss$ such that for any $r >0$,
  \Cref{ass:lyap_minorization}$(\cball{\thetas}{r})$ holds and that
  there exists $c_r>0$ satisfying for any
  $\theta \in \mss \setminus \cball{\thetas}{r}$, $c_r \leq V(\theta)$.
\end{assumption}
Note that assuming \Cref{ass:V_positive} is weaker than assuming \Cref{ass:lyap_minorization}$(\{\thetas\})$
since in the first case the constant $\lambda >0$ in \Cref{ass:lyap_minorization}$(\cball{\thetas}{r})$ may depend on $r$.

As announced previously, we obtain the convergence in distribution of
$\{\mu^{\upeta} \, : \, \upeta \in \ocint{0,\bupeta}\}$.
\begin{theorem}\label{cor:ball_dirac}
Assume \Cref{ass:had_or_complete},  \Cref{ass:0mean_noise}, \Cref{ass:MD:topo_prop_chain}, \Cref{ass:lyap} and \Cref{ass:lyap_meanfield} and let    $\bupeta = [2 C_2 L (1+\sigmaU)]^{-1}$.
\begin{enumerate}[wide, labelwidth=!, labelindent=0pt,label=(\alph*),noitemsep,nolistsep]
\item \label{cor:ball_dirac_i} In addition suppose
  \Cref{ass:lyap_minorization}$(\msks)$ holds for some compact set
  $\msks \subset \mss$ and that there exists $c>0$ such that for any
  $\theta \in \mss \setminus \msks$, $c \leq V(\theta)$. Then
  $\lim_{\upeta \to 0} \mu^\upeta \{\msks\} = 1 $, where
  $\mu^{\upeta}$ is the stationary distribution of $Q_{\upeta}$ for
  $\upeta \in \ocint{0,\bar{\upeta}}$.
\item \label{cor:ball_dirac_ii} In addition suppose \Cref{ass:V_positive} holds.  Then $(\mu^\upeta)_{\upeta \in \ocint{0,\bar{\upeta}}}$ converges weakly to
  $\updelta_{\theta^\star}$, as
  $\upeta \to 0$.
\end{enumerate}
\end{theorem}
% \begin{proof}
%   The proof is postponed to \Cref{app:ball_dirac}.
% \end{proof}

\subsection{Two Examples of Lyapunov Functions}
\label{sec:two-exampl-lyap}
 Having stated the main results of this section, we give two examples of Lyapunov functions $V$ under the following setting for $\Theta$.%and the mean field $h$.
\begin{assumptionA} \label{ass:hadamard_curvature}
 $\Theta$ is a Hadamard manifold. In addition, there exists $\kappa > 0$ such that the sectional curvature of $\Theta$ is bounded below by $- \kappa^2$.
\end{assumptionA}
% \begin{assumption} \label{ass:x*}
%  The restriction of the function $h$ to $\mss \subset \Theta$ has a unique zero $\theta^{\star} \in \mss$.
% \end{assumption}

 A classical choice of Lyapunov function on Euclidean spaces is
$\theta \mapsto \distT^2(\theta,\theta^\star)$, being both strongly
convex and Lipschitz-gradient. However, this function does not
satisfy \Cref{ass:lyap}-\ref{ass:lyap_grad_lips} as soon as
$\Theta$ has non-zero curvature and is non-compact.
In an effort to show the capital
impact of curvature and in order to obtain a valid
Lyapunov function satisfying  the conditions \Cref{ass:lyap}
and \Cref{ass:lyap_minorization}$(\cball{\theta^{\star}}{r})$ for $r >0$, we now introduce the necessary assumptions and
consider a truncated version of $\theta \mapsto
\distT^2(\theta,\theta^\star)$. 

Let
$\msh={\defEnsLigne{\Exp_\theta(t H_\theta(x)) \,:\, \theta \in \mss,
    x \in \msx, t \in [0,\upeta]}}$ be the set of all points reached
from geodesics $\upgamma:[0,1] \to \Theta$ of the form
$\upgamma(0) \in \mss$ and
$\dot{\upgamma}(0)=\upeta H_{\upgamma(0)}(x)$, for any $x \in \msx$.
We assume in our next result that the closure of $\msh$ is compact which is
implied for example in the case where $\mss$ is compact and  
$(\theta,x) \mapsto  H_{\theta}(x)$ is bounded on $\mss \times \msx$.

\begin{pproposition}\label{prop:sqdistance}
  Assume \Cref{ass:hadamard_curvature} and that the closure
  $\bar{\msh}$ of $\msh$ is compact, denote $\rmD_{\msh} =  \diam(\bar{\msh})$. Consider a smooth function
  $\chi_\msh : \Theta \to \ccint{0,1}$ with compact support satisfying
  $\chi_{\msh}(\theta) =1$ for any $\theta \in \bmsh$ and for any
  $\theta \in \Theta$ such that $\inf_{\theta'\in
    \bmsh}\distT(\theta',\theta)\geq 1$,
  it holds $\chi_{\msh}(\theta) =0$.
  %\alain{faire la preuve si on a le temps}.
  Consider now $\Vdist: \Theta \to \rset_+$ defined for any
  $\theta\in \Theta$ by
  \begin{equation}
    \label{eq:1}
    \Vdist(\theta)= \chi_{\msh}(\theta) \distT^2(\theta^{\star},\theta) + (1-\chi_{\msh}(\theta)) \rmD^2_{\msh} \eqsp.
  \end{equation}
  Then,
  \Cref{ass:lyap}-\ref{ass:lyap_contractive}-\ref{ass:lyap_grad_lips}
  holds with $V \leftarrow \Vdist$ and
  $L \leftarrow C_{\chi} (\rmD_{\msh}+1)(1+\kappa  \coth(\kappa \rmD_{\msh}))$ where
  $C_{\chi}\geq 0$ is a constant only depending on $\chi_{\msh}$. Suppose in
  addition that there exist $r >0, \lambda_{\rho} >0$ such that for any $\theta \in \mss$,
\begin{equation}
  \label{eq:prop:sqdistance}
 -\psr{\Exp^{-1}_\theta(\theta^{\star})}{h(\theta)}[\theta] \leq - \lambda_{\rho} \distT^2(\theta^{\star}, \theta) \1_{\mss \setminus \cball{\theta^{\star}}{r}}(\theta) \eqsp.
	\end{equation}
%  the following conditions hold.
% \begin{enumerate}[wide, labelwidth=!, labelindent=0pt,label=(\roman*),noitemsep,nolistsep]
% \item \label{ass:sqdist_attractive}
% % \item \label{ass:sqdist_growthgrad}
% %   There exist constants $A,B > 0$ such that for any $\theta \in \Theta$, $\normr{h(\theta)}[\theta]^2 \leq A + B \Vdist(\theta)$.
% \end{enumerate}
Then  \Cref{ass:lyap_minorization}$(\cball{\theta^{\star}}{r})$ holds with $\lambda \leftarrow \lambda_{\rho}$.
\end{pproposition}
% \begin{proof}
% The proof is postponed to \Cref{app:sqdistance}.
% \end{proof}
Note that under the setting of \Cref{prop:sqdistance},
$V_2(\theta) \geq c$ for any
$\theta \in \mss \setminus \cball{\thetas}{r}$ by definition, since it
is continuous. % In addition, if $\mss$ is compact and
%$(\theta,x) \mapsto H_{\theta}(x)$ is bounded on $\mss \times \msx$, we obtain that \Cref{ass:lyap} holds. 
Clearly, \Cref{ass:lyap}-\ref{ass:item:lyap_proper}  does not hold
 for  $\Vdist$ if $\mss$ is non-compact, since  $\Vdist$ is constant outside of the support of $\chi_\msh$. For this reason, and to weaken the assumptions of \Cref{prop:sqdistance}, we introduce a ``Huberized'' version of the distance to $\theta^\star$.

\begin{pproposition}\label{prop:huber}
  Assume \Cref{ass:hadamard_curvature}.  Let
  $\delta>0$ and consider $\VHuber : \Theta \to \rset_+$ defined for
  any $\theta \in \Theta$ by
\begin{equation} \label{eq:lyapunovV}
  \VHuber(\theta) = \delta^2\defEns{ \parenthese{{\distT(\theta^{\star},\theta)}/{\delta}}^2 +1 }^{1/2} - \delta^2 \eqsp.
\end{equation}
Then, \Cref{ass:lyap} holds with $V\leftarrow\VHuber$ and $L \leftarrow 1+\kappa \delta$.
Suppose in addition that there exist $r>0, \lambda_{\rho} > 0$ such that for
any $\theta \in \mss$, \eqref{eq:prop:sqdistance} holds.
% \begin{equation}
% \psr{\Exp^{-1}_\theta(\theta^{\star})}{h(\theta)}[\theta] \leq - \lambda \distT^2(\theta^{\star},\theta) \1_{\mss \setminus \ball{\theta^{\star}}{r}}(\theta) \eqsp.
% \end{equation}
% \begin{enumerate}[wide, labelwidth=!, labelindent=0pt,label=(\roman*),noitemsep,nolistsep]
%\item \label{ass:huber_attraction} 
% \item \label{ass:huber_growthgrad}
%   There exist constants $A,B > 0$ such that for any $\theta \in \Theta$, $\normr{h(\theta)}[\theta]^2 \leq A + B \VHuber(\theta)$.
% % \begin{equation}
% % \normr{h(\theta)}[\theta] \leq A + B \VHuber^{1/2}(\theta) \eqsp.
% % \end{equation}
%\end{enumerate}
Then, \Cref{ass:lyap_minorization}$(\cball{\thetas}{r})$ holds and $\lambda \leftarrow \lambda_{\rho}$.
\end{pproposition}
% \begin{proof}
% The proof is postponed to \Cref{app:huber}.
% \end{proof}
This Lyapunov function is still constructed upon the distance function, but as \Cref{prop:huber} shows, it is better suited for non-positive curvature spaces. %Indeed, the proof strategies in this paper rely on computing a Taylor expansion of $V$, which is only possible because $V$ is smooth. 
Note that under the setting of \Cref{prop:huber}, $V_1(\theta) \geq c$ for any $\theta \in \mss \setminus \cball{\thetas}{r}$ by definition since it is continuous.

It is worth mentioning that if either \Cref{prop:sqdistance} %(if
%$\mss$ is compact and $(\theta,x) \mapsto H_{\theta}(x)$ is bounded on $\mss \times \msx$)
or \Cref{prop:huber} can be applied, in order to use 
\Cref{theo:drift_lyap} and \Cref{thm:recurrent_ergodic}
(resp. \Cref{cor:ball_dirac}-\ref{cor:ball_dirac_ii}) the only
condition to verify (relative to the Lyapunov function)
 is \Cref{ass:lyap_meanfield} (resp. are
\Cref{ass:lyap_meanfield} and \Cref{ass:V_positive}).

% $\VHuber$ is smooth at $\thetas$ because it is equivalent to $\distT^2(\thetas,\theta)/2$ when $\distT(\thetas,\theta)\to 0$. In addition, $\VHuber$ behaves like $\distT(\thetas,\theta)$ at $\distT(\thetas,\theta)\to +\infty$, therefore giving that $\Hess \VHuber$ has a bounded operator norm. Consequently, this gives, using \cite[Lemma 10]{durmus:jimenez:moulines:said:wai:2020}, that $\grad \VHuber$ is geodesically Lipschitz. In summary, we vividly encourage the reader to consider this choice of Lyapunov function in the Riemannian setting over $\theta \mapsto \distT^2(\thetas,\theta)$.

%%% Local Variables:
%%% mode: latex
%%% TeX-master: "main"
%%% End:

%
\section{ASYMPTOTIC EXPANSION AND LAW IN THE UNCONSTRAINED
  CASE} \label{sec:unconstrained}

The purpose of this section is to quantify the convergence derived
in \Cref{cor:ball_dirac}-\ref{cor:ball_dirac_ii}. First, we establish
an asymptotic expansion for the bias
$\int_{\Theta} \gPlan(\theta) \rmd \mu^{\upeta}(\theta) -
\gPlan(\thetas)$ \wrt~the step size $\upeta$ for $\gPlan$
belonging to a certain class of smooth functions from $\Theta$ to $\rset$. Our result can be
applied to SGD ($h = -\grad f$) and implies then an asymptotic
expansion of
$\int_{\Theta} \normrLigne{\grad f(\theta)}[\theta]^2 \rmd
\mu^{\upeta}(\theta)$.  Secondly, we establish that the convergence
derived in \Cref{cor:ball_dirac}-\ref{cor:ball_dirac_ii} occurs at a rate $\upeta^{1/2}$, 
through a central limit theorem
for $(\mu^{\upeta})_{\upeta \in \ocint{0,\bupeta}}$. These
two results can be understood as a bias-variance decomposition in
which both terms are of order $\upeta$ and are therefore  weak counterparts of
\citet[Proposition 3, Theorem 5]{dieuleveut:2017}, \citet[Theorem
1]{pflug:1986} in a Riemannian setting. The related proofs are postponed to the supplement  \Cref{app:unconstrained}.

% \begin{wrapfigure}{l}{0pt}
% \begin{tikzpicture}[x={(170:1cm)},y={(55:.7cm)},z={(90:1cm)}]
%   \draw[looseness=.6, dashed] (2.5,-2.5,0) -- (2.5,2.5,0) node[below right] {$\planT_{\thetas}\Theta$} -- (-2.5,2.5,0) -- (-2.5,-2.5,0) -- cycle;
%   \draw[looseness=.6] (2.5,-2.5,-1) node[above right] {\, $\Theta$}
%   to[bend left] (2.5,2.5,-1)
%   to[bend left] coordinate (mp) (-2.5,2.5,-1)
%   to[bend right] (-2.5,-2.5,-1)
%   to[bend right] coordinate (mm) (2.5,-2.5,-1)
%   -- cycle;
%   \draw[looseness=.6] (mm) to[bend left] (0,0,0) to[bend left] (mp);
%   \path[looseness=.2] (mm) to[bend left]
%   node[pos=.2,pin={[pin distance=1cm,pin edge={solid,<-}]below right:$\gamma$}] {} (0,0,0);

%   \draw[->] (0,0,0) -- (3,0,0) node[left] {$N\times\dot{\gamma}$};
%   \draw[->] (0,0,0) -- (0,3,0) node[above right] {$\dot{\gamma}$};
%   \draw[->] (0,0,0) -- (0,0,3) node[right] {$N$};
%   \draw[dotted] (0,0,2) -- (1,0,2) -- (1,0,0);
%   \draw[->] (0,0,0) -- coordinate[pos=.3] (psi) (1,0,2) node[above left] {$\ddot{\gamma}$};
%   \node[left] at (0,0,1.5) {$\kappa_n$};
%   \node[above] at (.5,0,0) {$\kappa_g$};
%   \draw (0,0,.8) to[out=170,in=55] node[above,fill=white,inner sep=1pt,outer sep=2pt] {$\psi$} (psi);
% \end{tikzpicture}
% \end{wrapfigure}

% We observe, therefore, that the Riemannian Stochastic Approximation achieves a bias-variance tradeoff.
\vspace{-0.2cm}
\subsection{Asymptotic Expansion as $\upeta \to 0$}
Here, we assume that
\Cref{ass:had_or_complete}-\ref{ass:had_or_complete_ii} holds,
$\Theta$ is compact and the conditions of
\Cref{cor:ball_dirac}-\ref{cor:ball_dirac_ii} hold.  In addition,
define the covariance tensor field $\Sigma$ on $\Theta$, for any $\theta\in \Theta$ by,
%{\small
\begin{equation} \label{eq:SigmaT}
\Sigma(\theta) = \expe{\noise_{\theta}(X_{1}) \otimes \noise_{\theta}(X_{1})} \eqsp.
\end{equation}
%}
Under appropriate conditions, letting $n \to \plusinfty$ in
\Cref{theo:drift_lyap}-\ref{theo:drift_lyap1}, and
\Cref{theo:alter_theo:drift_lyap} in the supplement, 
show respectively that $\int_{\Theta} V(\theta) \rmd \mu^{\upeta}(\theta)$
and $\int_{\Theta} \normr{h(\theta)}[\theta]^2 \rmd \mu^{\upeta}(\theta)$ are bounded by a 
term of order $\upeta$. We specify this result in the
case where $h = - \grad f$ for a smooth objective function
$f :\Theta \to \rset$. More precisely, we establish in what follows a weak asymptotic expansion for
$\int_{\Theta} \normr{\grad f(\theta)}[\theta]^2 \rmd \mu^{\upeta}(\theta)$,
as $\upeta \to 0$ based on the following result for which
we assume:
\begin{assumptionMD} \label{ass:moment3} $\Sigma$ is a continuous
  $(2,0)$-tensor field  on $\Theta$.
\end{assumptionMD}
Denote the contraction of a covariant 2-tensor $F$ with a
contravariant 2-tensor $G$ on $\Theta$ by $[F: G]$; see
\Cref{app:geometry} \eqref{eq:contraction1}-\eqref{eq:contraction2} in
the supplementary for more details.  For two matrices $A,B$, $[A: B]$
just corresponds to $\trace(A B^{\transpose})$, where $^\top$ denotes the transpose. 
\begin{theorem}\label{prop:bound_error_moment}
  Assume \Cref{ass:had_or_complete}-\ref{ass:had_or_complete_ii}, $h$ is continuous, $h(\thetas) = 0$ and
  $\Theta$ is compact. Assume also
  \Cref{ass:0mean_noise}, \Cref{ass:MD:topo_prop_chain}, \Cref{ass:moment3},
  \Cref{ass:lyap}, \Cref{ass:lyap_meanfield} and \Cref{ass:V_positive}. Let
  $\bupeta = [2 C_2 L (1+\sigmaU)]^{-1} $.  
%\begin{equation}\label{eq:hess_V}
%\abs{\Hess V_\theta (u,u)} \leq \scrV_2 \normr{u}[\theta]^2 \quad \text{  and  } \quad \abs{\nabla \Hess V_\theta (u,u,u)} \leq \scrV_3 \normr{u}[\theta]^3 \eqsp. 
%\end{equation}
Then for any $\upeta \in (0,\bupeta]$ and smooth function $g :\Theta \to \rset$, we have
{\small\begin{equation}\label{eq:bound_error_moment}
-\int_{\Theta} \psrLigne{\grad g (\theta)}{h(\theta)}[\theta] \rmd \mu^{\upeta}(\theta) 
 =
\frac{\upeta}{2} \parentheseDeux{\Hess g : \Sigma}(\thetas)+ \scrR_{g,\upeta} \eqsp,
\end{equation}
}

\vspace{-0.5cm}
\noindent                                         
where $\lim_{\upeta \to 0} \{\absLigne{\scrR_{g,\upeta}}/\upeta \} = 0$.
\end{theorem}
%\begin{proof}
%The proof is postponed to \Cref{app:bound_error_moment}.
%\end{proof}
Applying this result to SGD, \ie~$h  = -\grad f$ and  $g= f$, we obtain that
$\int_{\Theta} \normr{\grad f (\theta)}[\theta]^2 \rmd\mu^{\upeta}(
\theta) = (\upeta/2) \parentheseDeux{\Hess f : \Sigma}(\thetas)+ \scrR_{f,\upeta}$ with
$\lim_{\upeta \to 0} \{\absLigne{\scrR_{f,\upeta}}/\upeta \} =
0$. % This result can be seen as a Riemannian and non-convex counterpart of
%$\underset{\upeta \to 0}{\lim} \{\absLigne{\scrR_{f,\upeta}}/\upeta \} =
%0$. % This result can be seen as a Riemannian and non-convex counterpart of
% \cite[Theorem 4]{dieuleveut:2017}.
% This theorem comes with an interesting application to the case of
% SGD, \ie~when there exists a smooth function $f:\Theta \to \rset_+$
% such that $- \grad f = h$. Consider, for example, that $V=\VHuber$ for
% which, as presented in \Cref{prop:huber}, \eqref{eq:hess_V} holds.
% Then, this \Cref{thm:central_limit} shows that at stationarity, the
% mean Huber distance to the solution $\thetas$, $\expeLigne{\VHuber}$
% is not larger than $\upeta$ times a constant.

\subsection{A Central Limit Theorem on $(\mu^{\upeta})_{\upeta \in \ocint{0,\bupeta}}$}
Now, we assume both
\Cref{ass:had_or_complete}-\ref{ass:had_or_complete_i} and
\Cref{ass:had_or_complete}-\ref{ass:had_or_complete_ii}, meaning
$\mss = \Theta$ and $\Theta$ is a Hadamard manifold. Note that under
this setting $\Exp_{\theta}^{-1}:\Theta \to \planT_\theta \Theta$ is a
well defined diffeomorphism for any $\theta \in\Theta$ by
\cite[Proposition 12.9]{lee:2019}. In addition, we assume that the
other conditions of
\Cref{cor:ball_dirac}-\ref{cor:ball_dirac_ii} hold. % Therefore,
% there exists $\thetas$ such that
% $(\mu^{\upeta})_{\upeta \in \ocint{0,\bupeta}}$ weakly converges to
% $\updelta_{\thetas}$. We now aim at quantifying this convergence. 
% Under these conditions, for any
% $n \in \nset$, \eqref{eq:scheme} corresponds to
% \begin{equation} \label{eq:scheme_unconstrained}
%   \theta_{n+1} = \Exp_{\theta_n}\left(\upeta\{ h (\theta_n)  + \noise_{\theta_n}(X_{n+1})\} \right) \eqsp.
% \end{equation}
Following the approach of \cite{pflug:1986} in Euclidean SA, to find
the asymptotic rate of convergence of the family
$(\mu^\upeta)_{\upeta \in (0,\bar{\upeta}]}$ defined in
\Cref{sec:constant_step}, we establish a central limit theorem in
$\planT_{\thetas} \Theta$, for the family of pushforward measures
$(\bnu^{\upeta})_{\upeta \in \ocint{0,\bupeta}}$ defined for any
$\msa \in \mcb{\planT_{\thetas} \Theta}$ by
{\small\begin{equation}
  \label{eq:def_bnu}
  \bar{\nu}^{\upeta}(\msa)=\mu^{\upeta}\parenthese{\Exp_{\theta^\star}(\upeta^{1/2}\msa)} \eqsp. 
\end{equation}
}

\vspace{-0.5cm}
\noindent                                         
It is shown in \Cref{app:unconstrained} that for any
$\upeta \in \ocint{0,\bupeta}$, $\bnu^{\upeta}$ is the stationary distribution of the
rescaled and projected Markov chain $(\bar{U}_n)_{n \in \nset}$ defined for any $n \in \nset$
by $\bar{U}_n = \upeta^{-1/2} \Exp_{\thetas}^{-1} (\theta_n)$.
% \begin{equation}\label{eq:def_u_bar}
%  \eqsp. 
% \end{equation}
Therefore, since under
\Cref{ass:had_or_complete}-\ref{ass:had_or_complete_i}-\ref{ass:had_or_complete_ii},
for any $u \in \planT_{\thetas} \Theta$,
$\distT(\thetas,\Exp_{\thetas}(u)) = \normr{u}[\thetas]$ by  \cite[Corollary 6.12,Proposition 12.9]{lee:2019}, showing a
central limit theorem for the family
$(\bnu^{\upeta})_{\upeta \in \ocint{0,\bupeta}}$ as
$\upeta \to 0$ shows that asymptotically
$(\mu^{\upeta})_{\upeta \in \ocint{0,\bupeta}}$ concentrates in regions
of diameter $\bigO(\upeta^{1/2})$ around $\thetas$ for the Riemannian
distance.

% Note that $\mu^\upeta$ is the asymptotic distribution
%  of $\theta_n$ as $n \to \infty$. Consequently, to
%   understand the behavior of $\mu^\upeta$ as $\upeta
%    \to 0$ further than \Cref{cor:ball_dirac}, we look
%     to find a a limit of the renormalized measure 
%     $\bar{\nu}^\upeta$. \Cref{thm:central_limit} below
%      states that this limit is a multivariate normal
%       distribution in the tangent space $\planT_\thetas
%        \Theta$. We require the following.
We consider the following assumptions.
\begin{assumptionMD}
  \label{ass:noise_unif_int_h}
   There exist
  $\varepsilon_e >0$, $\tsigma_{0}^2,\tsigma_{1}^2 \geq 0$ such that for any $\theta \in \Theta$, $\PE[\normr{e_{\theta}(X_1)}[\theta]^{2+\varepsilon_e}] \leq \tsigma_0^2 + \tsigma_1^2 V(\theta)$.
\end{assumptionMD}
\begin{assumption} \label{ass:driftmatrix}
 There exist a linear mapping $\bfA: \planT_{\theta^*}\Theta \rightarrow \planT_{\theta^*}\Theta$ and a map $\scrH:\Theta \to \planT_{\theta^*} \Theta$, such that for any $\theta \in \Theta$,
\begin{equation} \label{eq:htaylor}
 h(\theta) = \parallelTransport_{01}^{\upgamma}\left( \bfA \Exp^{-1}_{\thetas}(\theta) + \scrH(\theta)\right)\eqsp,
\end{equation}
where $\thetas$ is defined in \Cref{ass:V_positive}, $\parallelTransport_{01}^{\upgamma}$ denotes parallel transport along the geodesic $\upgamma :[0,1]\rightarrow \Theta$ with $\upgamma(0) = \thetas$ and $\upgamma(1) = \theta$, and $\lim_{\theta \to \thetas} \defEnsLigne{\normr{\scrH(\theta)}[\thetas]/\distT(\thetas,\theta)}=0$.
In addition, % \label{ass:stableA}
 the eigenvalues of the matrix $\bfA$ all have strictly negative real parts. Finally, there exists $C_3 > 0$ such that for any
  $\theta \in \Theta$,
  $\normr{h(\theta)}[\theta] \leq C_3 \distT(\thetas,\theta)$.
\end{assumption}
We show in \Cref{lem:taylor_vector} that \eqref{eq:htaylor} holds in the case $h$ is twice continuously differentiable on $\Theta$ with  
$\bfA = \nabla h(\thetas)$. For ease of notation, we also denote by
$\bfA$ and $\Sigma(\thetas)$ the matrices associated with these two linear applications in some orthonormal basis of
$\planT_\thetas \Theta$.
% of Consider
% $(\bfb_i)_{i\in \{1,\dots,d\}}$, and define the matrices
% $\mathrm{A} = (\rmA^i_j)_{i,j \in \{1,\dots,d\}}$ and
% $\Sigma_* = (\Sigma^{ij}_*)_{i,j \in \{1,\dots,d\}}$, where
% \begin{equation} \label{eq:ASigma}
% \rmA \bfb_i = \sum^d_{i=1} \rmA^i_j \bfb_i
% \quad \text{  and  }\quad \Sigma(\thetas) = \sum^d_{i,j=1} \Sigma_{\star}^{ij} \bfb_i \otimes \bfb_j \eqsp.
% \end{equation}
\Cref{ass:driftmatrix} guarantees the existence and uniqueness of the solution $\bfV\in \rset^{d \times d}$ of the Lyapunov equation $ \bfA \bfV + \bfV \bfA^\top = \Sigma(\thetas)$, see \cite[Theorem~2.2.1]{horn:johnson:1994}.
% \begin{equation}
%  \label{eq:lyapunov_cov}
%  \eqsp.
% \end{equation}
% \begin{equation} \label{eq:lyapunov_cov}
%  \eqsp,
% \end{equation}

% Introduce, for any smooth function with compact support $f:\planT_\thetas \Theta \to \rset$, $u \in \planT_\thetas \Theta$ the infinitesimal generator $\calL$, by
% \begin{equation} \label{eq:calL}
%   \calL f (u) = \sum^d_{i,j=1} \rmA^i_j u^j \partial_i f (u) + (1/2) \sum^d_{i,j=1} \Sigma^{ij}_\star\partial^2_{ij} f (u) \eqsp,
% \end{equation}
% where $(\partial_i)_{i \in \{1,\dots ,d \}}$ and $(u^i)_{i \in \{1,\dots,d\}}$ are the partial derivatives and the coordinates with respect to $(\bfb_i)_{i \in \{1,\dots,d\}}$.
% Then, under \Cref{ass:stableA}, the unique invariant distribution of a linear diffusion process with values on $\planT_\thetas \Theta$ and generator $\calL$ is $\loiGauss(0,\rmV)$, where the matrix $\rmV \in \rset^{d\times d}$ is given in \eqref{eq:lyapunovV}.
We also assume that $V$ can be compared to a function of the distance
on $\Theta$ which leads to the strengthening of \Cref{ass:V_positive}.
\begin{assumption}
  \label{ass:lyap_comparison_distance}
  There exists $\thetas$ such that
  \Cref{ass:lyap_minorization}$(\{\thetas\})$ holds and there exists
  $\phi : \rset_+ \to \rset_+$ such that for any $\theta \in \Theta$,
  $V(\theta) \geq \phi(\distT(\thetas,\theta))$ and for any $r >0$,
  $\inf_{\coint{r,\plusinfty}}\phi >0$. In addition, there exists $\bar{a} >0$,
  such that $\lim_{r\to \plusinfty} \sup_{a \leq \bar{a}} a/\phi(a^{1/2}r) =
  0$. 
\end{assumption}

Note
that the assumption on the growth rate of the
Lyapunov function is verified when $V=
\VHuber$, considered in
\Cref{prop:huber}. In this case, we can  take $\phi(r) = \delta^2[1+(r/\delta)^2]^{1/2} - \delta^2$. 

% In our next result, we prove a central limit 
% theorem for the rescaled invariant 
% distribution $\bnu^\upeta$ when $\upeta \to 0$.
\begin{theorem}\label{thm:central_limit}
  Assume
  \Cref{ass:had_or_complete}-\ref{ass:had_or_complete_i}-\ref{ass:had_or_complete_ii},
  \Cref{ass:0mean_noise}, \Cref{ass:MD:topo_prop_chain},
  \Cref{ass:moment3}, \Cref{ass:noise_unif_int_h}, \Cref{ass:lyap},
  \Cref{ass:lyap_meanfield}, \Cref{ass:driftmatrix} and
  \Cref{ass:lyap_comparison_distance} hold. Suppose in addition that
  $h(\thetas) =0$, $h$ is continuous and let $\bupeta = [2 C_2 L (1+\sigmaU)]^{-1} \wedge (4C_3)^{-1}$.
  Then, the family of distributions
  $(\bnu^\upeta)_{\upeta \in \ocint{0,\bupeta}}$, defined by
  \eqref{eq:def_bnu}, converges weakly to $\loiGauss(0,\bfV)$ as
  $\upeta \to 0$ on $\planT_\thetas \Theta$, where $\bfV$ is the
  unique solution to the Lyapunov equation $ \bfA \bfV + \bfV \bfA^\top = \Sigma(\thetas)$.
\end{theorem}
% Both results in this section derive from a
% Taylor expansion of $Q_\upeta f$, for a smooth 
% and compactly supported function $f:\Theta \to \rset$.
% This expansion depends heavily on controlling
% $V$ and $h$, and their derivatives along
% geodesics $\upgamma:[0,1]\to \Theta$ given,
% for any $t \in [0,1]$ and $\theta \in \mss$ by
% $\upgamma(t)=\Exp_\theta(t\upeta H_\theta(X_1))$.
% \Cref{ass:had_or_complete}-
% \ref{ass:had_or_complete_ii}
% is therefore required so that the geodesics
% remain in $\mss$.

Even though $\loiGauss(0,\bfV)$ is a distribution on $\rset^d$, we
identify $\rset^d$ with $\planT_\thetas \Theta$ using the same
orthonormal basis as before.  As mentioned in \Cref{sec:constant_step},
\Cref{thm:central_limit} complements \Cref{cor:ball_dirac} because it
proves that the asymptotic rate of convergence of $(\mu^\upeta)_{\upeta \in (0,\bupeta]}$ 
to $\updelta_{\thetas}$ is $\upeta^{1/2}$, since
$(\bar{U}_n)_{n \in \nset}$ is rescaled by this factor with respect to
the actual SA scheme $(\theta_n)_{n\in \nset}$. Finally
\Cref{thm:central_limit} can be seen as a Riemannian counterpart of
\citet[Theorem 1]{pflug:1986}.  In the following section, we illustrate
our results on SGD.

\section{APPLICATION TO SGD} \label{sec:app_sgd}
We assume throughout this section that 
\Cref{ass:had_or_complete}-\ref{ass:had_or_complete_i}-\ref{ass:had_or_complete_ii} holds. 
We apply the results of \Cref{sec:constant_step} and
\Cref{sec:unconstrained}, to the unconstrained stochastic
gradient scheme, \ie~$(\theta_n)_{n \in\nset}$ defined by
\eqref{eq:scheme} with $h = -\grad f$ and $\mss = \Theta$.  Proofs are postponed to the supplement, \Cref{app:applications_sgd}.

%\subsection{Geodesically Strongly Convex and Smooth Function}
\textbf{Geodesically Strongly Convex and Smooth Function}~
First, the objective function 
$f:\Theta \rightarrow \rset$ is subject to the following assumptions.
% is a twice continuously differentiable function.% Clearly, \eqref{eq:sgdschemecss} is the same as  \eqref{eq:scheme}, with $\mss = \Theta$ and $h(\theta) = -\grad f(\theta)$. It is assumed that the cost function $f$ is strongly geodesically convex.%~\cite{boumal:2020}.
\begin{assumptionF}
  \label{ass:f_grad_lip}
  $f : \Theta \to \rset$ is twice continuously differentiable and $\grad f$ is geodesically $L_f$-Lipschitz, see \eqref{eq:llipschitz}. 
\end{assumptionF}
%We first specify our results in the case:
\begin{assumptionF} \label{ass:strongconv}
   $f$ is continuously differentiable on $\Theta$ and $\lambda_f$-strongly geodesically convex, for
  some $\lambda_f > 0$, \ie~for any $\theta_1,\theta_2 \in \Theta$,
  $f(\theta_2) \geq f(\theta_1) +
  \psrLigne{\Exp_{\theta_1}^{-1}(\theta_2)}{\grad f(\theta_1)}[\theta_1] +
  \lambda_f\distT^2(\theta_1,\theta_2)$.
\end{assumptionF}
Under \Cref{ass:strongconv}, $f$ admits a unique minimizer denoted by $\thetas$. In addition, we have the following inequalities.
\begin{llemma}
  \label{lem:str_conv}
  Assume \Cref{ass:had_or_complete}-\ref{ass:had_or_complete_i}-\ref{ass:had_or_complete_ii} and \Cref{ass:strongconv}. Then for any $\theta \in \Theta$, we have 
\begin{equation} \label{eq:strongconv}
	\begin{aligned}
		\normr{\grad f(\theta)}[\theta]^2 &\geq \lambda_f (f(\theta) - f(\thetas)) \eqsp \eqsp \text{and} \eqsp,\\
	f(\theta) - f(\thetas) &\geq \lambda_f \distT^2(\theta,\thetas) \eqsp. 
	\end{aligned}
\end{equation}  
\end{llemma}
Under \Cref{ass:f_grad_lip} and \Cref{ass:strongconv},
\Cref{lem:str_conv} implies that
$V(\theta) = f(\theta)-f(\thetas)$ and $h = -\grad f$ satisfy
\Cref{ass:lyap} with $L \leftarrow L_f$, \Cref{ass:lyap_meanfield}
with $C_1 \leftarrow 0, C_2\leftarrow 1$ and
\Cref{ass:lyap_minorization}$(\emptyset)$ with $\lambda \leftarrow \lambda_f$.
A direct application of \Cref{theo:drift_lyap}-\ref{theo:drift_lyap2} leads to the following result.
%, which can be obtained immediately, using Lemma \ref{lem:sgd_sc1}.
\begin{ccorollary} \label{cor:sgd_sc1} Assume
  \Cref{ass:had_or_complete}-\ref{ass:had_or_complete_i}-\ref{ass:had_or_complete_ii},
  \Cref{ass:0mean_noise}, \Cref{ass:f_grad_lip},
  \Cref{ass:strongconv}. Consider $(\theta_n)_{n \in\nset}$ defined by
  \eqref{eq:scheme} with $h = -\grad f$. Let
  $\bar{\upeta} = [2L_f(1+\sigmaU)]^{-1}$ and
  $\upeta \in \ocint{0,\bar{\upeta}}$. For any $\theta_0 \in \Theta$,
  and $n \in \nset$,
{\small	\begin{equation}\label{eq:last_regret_sc}
		\PE[f(\theta_n) - f(\thetas)] \leq (1-\upeta\lambda_f/2)^n(f(\theta_0) - f(\thetas)) 
					     + 2\upeta\,L_f\sigmaZ/\lambda_f \eqsp.
                                           \end{equation}
                                         }
                                         
\vspace{-0.5cm}
\noindent                                         
Then, setting $\upeta = \bupeta \wedge [\varepsilon \lambda_f/\{4\sigma_0^2L_f\}]$, for $\varepsilon \in\ooint{0,1}$, and $	n = \ceil{[\log(1/\varepsilon)-\log(f(\theta_0)-f(\thetas))]/\log(1-\upeta\lambda_f/2)}$,
% \begin{equation}
%   \label{eq:11}
%   \begin{aligned}
% 	n &= \ceil{[\log(1/\varepsilon)-\log(f(\theta_0)-f(\thetas))]/\log(1-\upeta\lambda_f/2)} %\eqsp,
%    \end{aligned}
% \end{equation}
we get $ \PE[f(\theta_n) - f(\thetas)] \leq \varepsilon$. 
\end{ccorollary}
\Cref{cor:sgd_sc1} shows that \eqref{eq:scheme} has a computational
complexity of order $\bigO(\log(1/\varepsilon)\varepsilon^{-1})$ to minimize $f$,
without any boundedness assumptions on $\Theta$, contrary to
\cite{sra:2016}. In addition, \Cref{lem:str_conv} also implies that
\Cref{ass:driftmatrix} and \Cref{ass:lyap_comparison_distance} hold if
$f$ is three times continuously differentiable and therefore
\Cref{thm:central_limit} can be applied.

\textbf{Geodesically Convex Function with Bounded Gradient}~ Consider the following assumption.
\begin{assumptionF} \label{ass:quasiconv} $f$ is twice continuously differentiable. Further, there exists $\tlambda_f >0$
  such that for any $\theta \in \Theta$,
  $ -\psrLigne{\Exp^{-1}_\theta(\thetas)}{\grad f(\theta)}[\theta] \geq \tlambda_f
  V_1(\theta) $, where $V_1$ is defined by \eqref{eq:lyapunovV} with
  $\delta =1$. In addition, there exists $C_f > 0$ such that for any
  $\theta \in \Theta$,
  $\normr{\grad f(\theta)}[\theta]^2 \leq C_f
  (\distT^2(\thetas,\theta) \wedge 1)$.
\end{assumptionF}
Note that a function $f$ satisfying \Cref{ass:quasiconv} is strictly
geodesically convex but not necessarily strongly geodesically convex. By introducing \Cref{ass:quasiconv},
we can relax the condition \Cref{ass:f_grad_lip} using the following
result.
\begin{llemma}
  \label{lem:str_huberized}
  Assume \Cref{ass:hadamard_curvature} and
  \Cref{ass:strongconv}.  Suppose in addition that $f$ is twice continuously differentiable and  
  there exists $M_f > 0$ such that  for any 
  $\theta \in \Theta$, $\normrLigne{\grad f(\theta)}[\theta]^2 \leq M_f \distT^2(\thetas,\theta)$. 
  Let $\tilde{f} = \defEnsLigne{f - f(\thetas) + 1}^{1/2}$. Then $\tilde{f}$ satisfies 
  \Cref{ass:quasiconv} with $C_f \leftarrow (M_f/4)[1\wedge \lambda_f]$ and 
  $\tlambda_f \leftarrow  \lambda_f /(2M_f^{1/2})$. 
\end{llemma}
Note that the condition introduced in \Cref{lem:str_huberized} is
a relaxation of the condition that $\grad f$ is geodesically
Lipschitz. Indeed, by \citet[Theorem 5.6.1]{jost:2005}, for
$\thetas \in \Theta$, $\theta \mapsto \distT^2(\thetas,\theta)$
satisfies the conditions of \Cref{lem:str_huberized} but its gradient
is not geodesically Lipschitz. 
% Under \Cref{ass:quasiconv}, let $\thetas$ be the unique global minimum of $f$, and assume that,
% \begin{assumptionQC} \label{ass:lyap_quasiconv}
%   there exists $C_f > 0$ such that $\normr{\grad f(\theta)}[\theta]^2 \leq C_f \left(\distT(\theta,\thetas) \wedge 1\right)$, for any $\theta \in \mss$. 
% \end{assumptionQC}
A non-asymptotic bound is now given in terms of the distance-like function, defined for any $\theta_1,\theta_2 \in \Theta$ by
{\small
  \begin{equation}
  \label{eq:def_D_theta}
  D_\Theta^2(\theta_1,\theta_2) = \distT^2(\theta_1,\theta_2)/(1+\distT^2(\theta_1,\theta_2)) \eqsp. 
\end{equation}
}

\vspace{-1cm}
\noindent                                         
% \begin{equation} \label{eq:distD}
%   D_\Theta(\theta_1,\theta_2) = \distT(\theta_1,\theta_2)/(1+\distT(\theta_1,\theta_2)) \eqsp.
% \end{equation}
%It is well-known this is a true distance function on $\Theta$, satisfying the triangle inequality.
\begin{pproposition} \label{prop_sgdconstrained}
  Assume that  \Cref{ass:hadamard_curvature}, \Cref{ass:0mean_noise}, \Cref{ass:quasiconv} hold. Let $\bupeta =[(8  C_f/ \tlambda_f) (1+\kappa ) (1+\sigmaU)]^{-1}$ and $\upeta \in \ocint{0,\bupeta}$. Consider $(\theta_n)_{n \in\nset}$ defined by
  \eqref{eq:scheme} with $h = -\grad f$ and $\mss = \Theta$. Then,  for any $\theta_0 \in \Theta$ and $n \in \nset^*$,
{ \small  \begin{equation} \label{eq:prop_sgdconstrained}
		   n^{-1} \sum_{k=0}^{n-1} \expe{D_{\Theta}^2(\thetas,\theta_n)} 
		   \leq 4 V_1(\theta_0)/(n \upeta \tlambda_f) 
		    +  4 \upeta (1+\kappa)\sigmaZ /\tlambda_f  \eqsp,
                  \end{equation}
                  }
% \begin{equation}
%   n^{-1} \sum_{k=0}^{n-1} \PE[D_\Theta(\theta_n,\thetas)] \leq (1/\lambda_f) V_1(\theta_0)/(n\upeta) +  \upeta\,(1+\kappa)\sigmaZ
% \end{equation}

                  \vspace{-0.5cm}
\noindent                                         
where $\kappa$ is given in \Cref{ass:hadamard_curvature}, and $V_1$ is defined by \eqref{eq:lyapunovV} with $\delta =1$.
%\alaini{ $\bar{\upeta} = (\lambda_f/C_f)((1+\kappa)(1+\sigmaU))^{-1}$ ??}
\end{pproposition}
To the authors' knowledge, such a bound is novel even in a deterministic setting. 

\textbf{Application to the Riemannian Barycenter Problem}~ 
To conclude our study, we consider the
problem of computing the Riemannian barycenter $\thetas$ of a %discrete
probability distribution $\pi$ on a Hadamard manifold $\Theta$.
First, we look at the discrete case:
\begin{equation}
  \label{eq:def_pi_discrete}
\textstyle{\pi = M_{\pi}^{-1}\sum_{i=1}^{M_\pi} \updelta_{\btheta_i}} \eqsp,
\end{equation}
where
$M_{\pi} \in \nsets$ and
$\{\btheta_i\}_{i=1}^{M_{\pi}} \in \Theta^{M_{\pi}}$. The Riemannian
barycenter $\thetas$ or Karcher mean of $\pi$ \citep{arnaudon:2012}  is the unique global minimum of the 
 function $ f_{\pi}:\theta \mapsto  \sum_{i=1}^{M_{\pi}}\left. \distT^2(\theta, \btheta_i)  / (2M_{\pi}) \right.$.
 % \begin{equation}
 % f_{\pi}:\theta \mapsto  \sum_{i=1}^{M_{\pi}}\left. \distT^2(\theta, \btheta_i)  \middle/ (2M_{\pi}) \right.\eqsp. 
 % \end{equation}
%$f(\theta) = (1/(2M_{\pi})) \sum_{i=1}^{M_{\pi}} \distT^2(\theta, \btheta_i)$. 
% \begin{equation} \label{eq:variancef}
%   f(\theta) =  \frac{1}{2}\,\int_\Theta \distT^2(\theta,x)\pi(\rmd x)
% \end{equation}
% For this function $f$, the stochastic gradient scheme \eqref{eq:sgdschemecss} becomes %reference
% \begin{equation} \label{eq:sgd_barycentre}
%    \theta_{n+1} = \Exp_{\theta_n}\!\left( \upeta H_{\theta_n}(X_{n+1})\right)\hspace{0.25cm};\hspace{0.25cm}
%   H_\theta(x) = -\Exp^{-1}_\theta(X_{n+1})
% \end{equation}
% where $(X_n\,;n\in \nset)$ are independent realisations of $\pi$. 
By \citet[Theorem 5.6.1]{jost:2005}, $\grad f_{\pi}(\theta) = - M_{\pi}^{-1} \sum_{i=1}^{M_{\pi}} \Exp^{-1}_{\theta}(\btheta_i)$ for any $\theta \in\Theta$ and $f_{\pi}$ satisfies
\Cref{ass:strongconv} with $\lambda_f = 1/2$ using \citet[Lemma 10]{durmus:jimenez:moulines:said:wai:2020}. Therefore, by
\Cref{lem:str_huberized}, \Cref{prop_sgdconstrained} can be applied. In addition, %setting $H_{\theta}(X_i) = \Exp^{-1}_{\theta}(X_i)$, $e_{\theta}(X_i) = H_{\theta}(X_i) - \grad f(\theta)$, where $(X_i)_{i\in\nsets}$ is a sequence of \iid~random variables with distribution $\pi$, 
  we get the 
  following result, as an application of  \Cref{prop:sqdistance} and
\Cref{theo:drift_lyap}-\ref{theo:drift_lyap2}.% On the other hand, $f$ will not verify

\begin{pproposition} \label{prop:br1} Assume
  \Cref{ass:hadamard_curvature}. Let $\thetaspi$ be the Riemannian
  barycenter of the probability measure $\pi$ in \eqref{eq:def_pi_discrete} on the Hadamard manifold
  $\Theta$, and let $(\theta_n)_{n\in \nset}$ be given by
  $\theta_{n+1} =
  \Exp_{\theta_n}(\upeta\Exp^{-1}_{\theta_n}(X_{n+1}))$, where
  $(X_n)_{n\in \nsets}$ is a sequence of \iid~random variables with
  distribution $\pi$. Then, for any
  $\upeta \in \ocintLigne{0, 1/(C L^3_{\pi} )}$, $\theta_0 \in \Theta$ and
  $n \in \nset$,
\begin{equation} \label{eq:regret_br}
\PE[\distT^2(\theta_n,\thetaspi)] \leq (1-\upeta/4)^n \distT^2(\theta_0,\thetaspi) + C\upeta\,L_\pi{\rm D}^2 \eqsp,
\end{equation}
where $L_\pi = (1+\rmD)(1+\kappa  \coth(\kappa \rmD))$, $C$ is a universal constant, and ${\rm D} = \max_{i=1,\ldots,M_\pi\,} \distT(\theta_0,\btheta_i)$. 
\end{pproposition}

Secondly, we tackle the general case where $\pi$ is not required to be discrete  or compactly supported.
In this case, the mapping that we are looking to minimize is
\begin{equation}
	\label{def:barycenter_noncompact}
	f_{\pi}:\theta \mapsto (1/2) \int_\Theta \distT^2\parenthese{\theta,\nu} \pi\parenthese{\rmd \nu} \eqsp .
	%= (1/2) \expe{\distT^2\parenthese{\theta,X}} \eqsp,
\end{equation}
The function $f_{\pi}$ is well-defined and finite  under the following assumption.
\begin{assumptionMD}
  \label{ass:second_moment_pi}
  There exists $\theta \in \Theta$ such that $    \int_{\Theta} \distT^2(\theta,\nu) \pi (\rmd \nu) < \plusinfty$.
 %  \begin{equation}
 %    \label{eq:2}
 % \eqsp.
 %  \end{equation}
\end{assumptionMD}
Note that by the triangle inequality, \Cref{ass:second_moment_pi} is equivalent to for any $\theta \in \Theta$ such that $\int_{\Theta} \distT^2(\theta,\nu) \pi (\rmd \nu) < \plusinfty$ and therefore $f_{\pi}$ is finite.
% We are looking to compute the Riemannian barycenter of a distribution $\pi$, \ie~finding a global minimum of the following mapping, defined over $\Theta$,
Using the Lebesgue's dominated convergence theorem and \citet[Theorem 5.6.1]{jost:2005}, we can compute its Riemannian gradient given for any $\theta \in \Theta$ by, $	\grad f_{\pi} (\theta) = - \int_\Theta \Exp^{-1}_\theta \parenthese{ \nu} \pi\parenthese{\rmd \nu}$.
% \begin{equation}
% 	\label{eq:grad_barycenter}
% 	\grad f_{\pi} (\theta) = - \int_\Theta \Exp^{-1}_\theta \parenthese{ \nu} \pi\parenthese{\rmd \nu}  \eqsp. 
% 	%= -\expe{\Exp^{-1}_\theta \parenthese{X}} \eqsp.
% \end{equation}
Then, $f_{\pi}$ satisfies \Cref{ass:strongconv} with $\lambda_f= 1/2$ and admits a unique minimizer $\thetaspi$. 
However, $\grad f$ does not satisfy \Cref{ass:f_grad_lip} in
general. More precisely, it fails to be geodesically Lipschitz, see
\citet[Theorem 5.6.1]{jost:2005}.
%, for example if the support of $\pi$ is not bounded .
In the Euclidean setting, several modifications of SGD have been suggested
to rescale the gradient such as RMSProp, AdaGrad and Adam \citep{hinton:2014,duchi:2011,kingma:2017}. 
Inspired by these methods, we consider the stochastic approximation scheme \eqref{eq:scheme} with $\mss=  \Theta$ and 
{\small \begin{equation}
  \label{eq:def_H_theta}
H_{\theta}(X_{n+1}) =   (1/2)
\Exp^{-1}_{\theta}\parenthese{X_{n+1}^{(1)}} \defEnsLigne{\distT^2(\theta,X_{n+1}^{(2)})/2 + 1}^{-1/2} \eqsp,
                \end{equation}
              }
              
                \vspace{-0.6cm}
                \noindent
                where $X_{n+1} = (X_{n+1}^{(1)},X_{n+1}^{(2)})$ and
                $(X_k^{(1)},X_k^{(2)})_{k\in \nsets}$ is an
                \iid~sequence of pairs of independent random variables
                with distribution $\pi$. The following result establishes non-asymptotic convergence bounds for the resulting recursion. 

\begin{theorem}
	\label{theo:drift_bary}
	Assume
	  \Cref{ass:hadamard_curvature} and \Cref{ass:second_moment_pi}. Let $\thetaspi$ be the Riemannian
	  barycenter of the probability measure $\pi$. 
          Let $(\theta_n)_{n\in \nset}$ be given by \eqref{eq:scheme} with $\mss = \Theta$ and $H$ defined by  \eqref{eq:def_H_theta}.
 Then, for any $n \in \nset$,
 {\small
   \begin{equation}
		\label{eq:drift_bary}
	n^{-1} \sum_{k=0}^{n-1} \expe{D^2_\Theta(\theta_{k},\thetaspi)} 
		\leq \left. 4\VHuber(\theta_0)C_\pi^{1/2} \middle/ (\upeta n) \right. 
		+ 4 \upeta B_\pi \eqsp,
              \end{equation}
            }

            \vspace{-0.5cm}
            \noindent
	where $V_1$ is defined by \eqref{eq:lyapunovV} with $\delta\leftarrow 1$, $\thetas \leftarrow \thetaspi$, $C_\pi = 1+ 2 f_\pi(\thetaspi)$, 
	$B_\pi =(1+\kappa) (f_\pi(\thetaspi) + 1) 
	(f_\pi(\thetaspi) + 2) C_\pi^{-1/2} $
	and $D^2_\Theta$ is defined in \eqref{eq:def_D_theta}.
\end{theorem}

%%% Local Variables:
%%% mode: latex
%%% TeX-master: "main"
%%% End:

\section{NUMERICAL EXPERIMENTS}
      
We consider in our experiments the Karcher mean estimation problem on
$\Theta = \SPD_{50}^{+}(\rset) \subset \rset^{50 \times 50}$, the symmetric definite positive matrix manifold (SPD) equipped with
its affine-invariant metric, see \cite{pennec:2006}. Note that the dimension of $\Theta$ is $1275$.

We first consider the case where $\pi=(15)^{-1} \sum_{i=1}^{15} \updelta_{x_i}$ is a discrete distribution, where
$\{x_i\}_{i=1}^{15}$ are random samples from the Wishart distribution
$\mathbf{W}(50,\Id)$ \ie~with 50 degrees of freedom and scale matrix identity.
The Karcher mean $\thetaspi$ associated with
$\pi$ is estimated using the Matrix Means Toolbox
\citep{bini:iannazzo:2013}.

\Cref{fig:path} represents the behavior
of the squared distance to the barycenter $\thetaspi$ for a single
path and three step-sizes $\upeta \in \{10^{-3},4\times10^{-3},10^{-2}\}$. 
As expected from \Cref{prop:br1}, two regimes can be observed. At
first, the squared-distance to the barycenter exponentially decreases and then
the iterates oscillate in a $\bigO(\upeta^{1/2})$-neighborhood of
$\thetaspi$.  In addition, the rate of convergence in the exponential
decay depends on the step-size.

In \Cref{fig:monte}, we aim at illustrating \eqref{eq:bound_bias},
\Cref{prop:bound_error_moment} and \Cref{thm:central_limit}.  To this
end, 1000 replications of the previous experiment are performed to obtain
$\{(\theta_n^{(i)}) \,: \, i \in \{1,\ldots,1000\}\}$ for $n = \ceil{10/\upeta}$ and 
$\upeta \in \{1,2.8, 4.6,6.4,8.2,10\}\times 10^{-2}$. These samples are used to 
estimate the mean and the variance of $\distT^2(\theta,\thetaspi)$, for $\theta$
following the stationary distribution $\mu^{\upeta}$. 
As expected,
the mean and the variance are both linear \wrt~the step-size $\upeta$,
further confirming that the iterates remain in a neighborhood of
diameter $\calO(\upeta^{1/2})$ to the ground truth.

\begin{figure}
	\includegraphics[width=1.\linewidth]{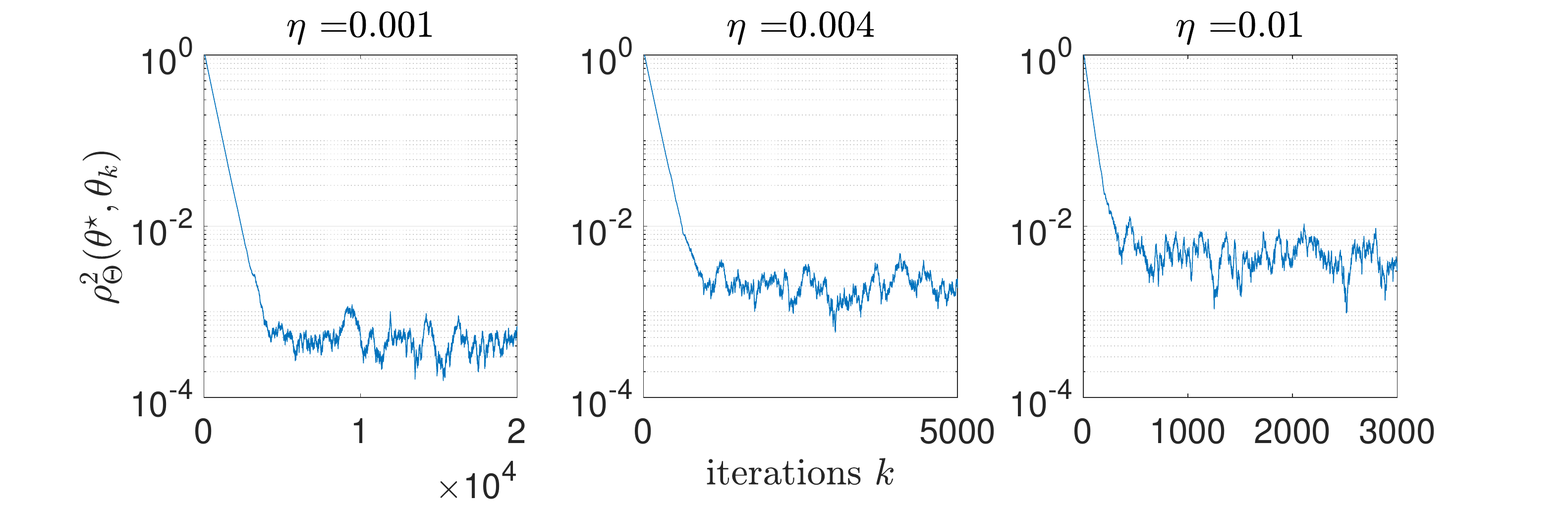}
	\caption{Paths of the algorithm in \Cref{prop:br1}}
	\label{fig:path}
\end{figure}

\begin{figure}
  \centering
	\includegraphics[width=1.\linewidth]{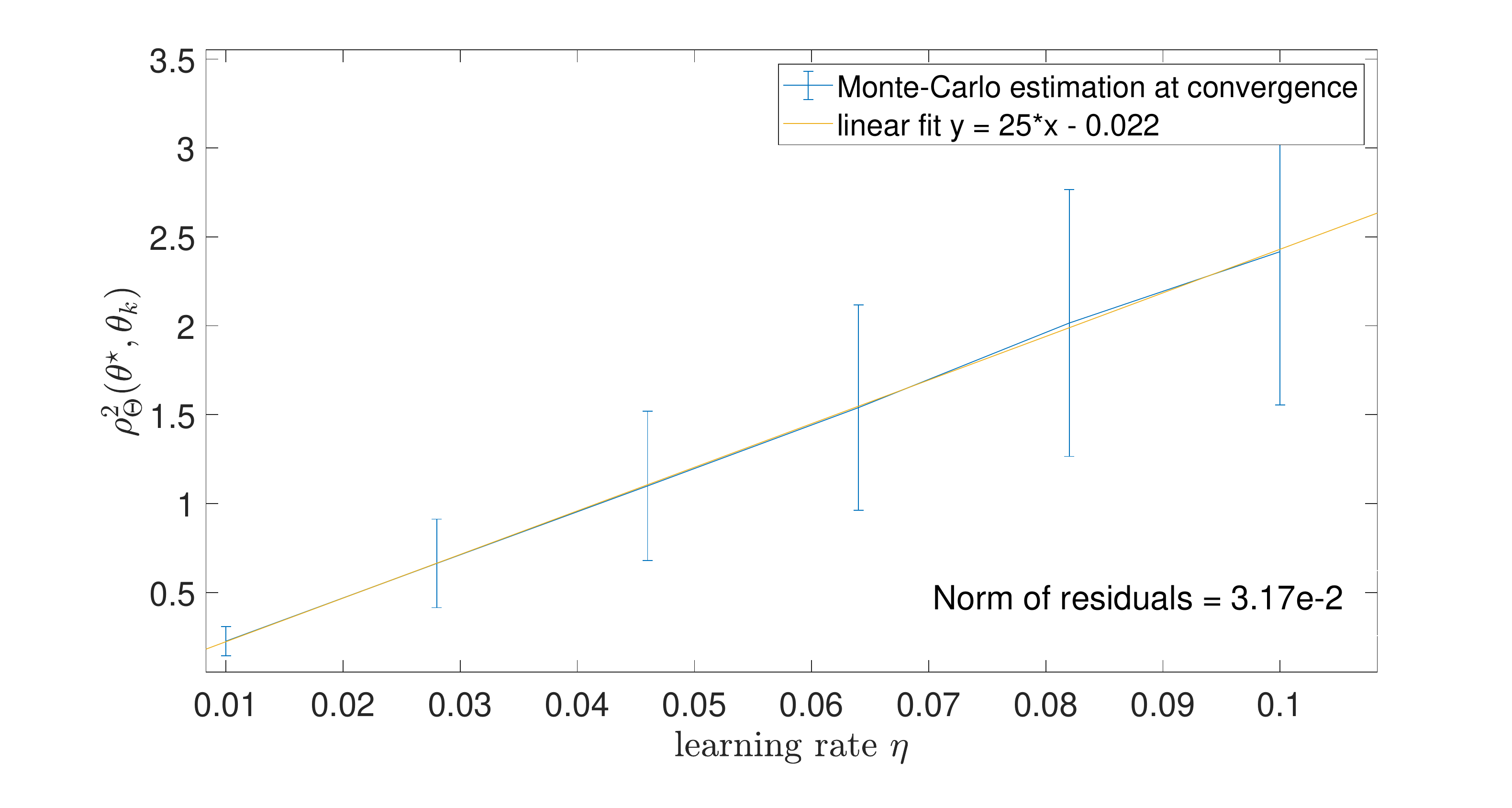}
	 \caption{}%Monte-Carlo estimates of the
	% amplitude of oscillations at convergence}
	\label{fig:monte}
\end{figure}

Secondly, we examine the barycenter problem for $\pi=\mathbf{W}(50,\Id)$, 
following the scheme introduced in \eqref{eq:def_H_theta}. The estimation 
of $\thetaspi$, relative to the new distribution $\pi$, is now done with 
a $100$-batch-size version of our methodology, with $10^6$ iterations and $\upeta=10^{-4}$.

As a counterpart to \Cref{fig:path}, in \Cref{fig:path_unbounded} we are interested 
in the mean values of $(D^2_\Theta(\theta_n,\thetaspi))_{n \in \nset}$ along a single
path for three step-sizes $\upeta \in \{10^{-3},4\times10^{-3},10^{-2}\}$, with 
respective burn-ins $\{13,3.3,1.645\} \times 10^3$. 
As predicted by \Cref{theo:drift_bary}, an initial decrease in $\bigO(n^{-1})$ is 
followed by a plateau in $\bigO(\upeta)$. We can observe that compared to \Cref{fig:path}, averaging smoothes 
oscillations.

Finally, we also perform the experiment corresponding to
\Cref{fig:monte} for the discrete setting to illustrate numerically
that the conclusions of \eqref{eq:bound_bias},
\Cref{prop:bound_error_moment} and \Cref{thm:central_limit} still
hold. However, due to space constraints and since the conclusions are the same than for \Cref{fig:monte}, the corresponding figure is postponed to the supplement \Cref{fig:monte_supp}. 

% In
%  to support
% \Cref{theo:drift_bary}.  We observe the same behavior with the
% distance-like function $D^2_\Theta$ in \Cref{fig:path_unbounded}.

% In \Cref{fig:monte},  we sample 1000 paths for each step-size in the discrete case, to
%  illustrate that the size of oscillations of the square distance is linear with respect to the step-size, therefore the
%  iterates remain in a neighborhood of diameter $\calO(\upeta^{1/2})$ to the ground truth, 
%  confirming the asymptotic rate $\upeta^{1/2}$ found in \Cref{thm:central_limit}.

% These numerical experiments are done with matrices of size $50 \times 50$, 
% the discrete distribution is created with 15 random samples 
% of the Wishart distribution with 50 degrees of freedom and scale matrix identity. 
% The same Wishart distribution is used as $\pi$ for the second experiment. We compute 
% the barycenter $\thetaspi$ in the discrete case, and approximate it in the general 
% case, using the Matrix Means Toolbox on MATLAB \cite{bini:iannazzo:2013}.
% \vspace{-0.5cm}

\begin{figure}
	\includegraphics[width=1.\linewidth]{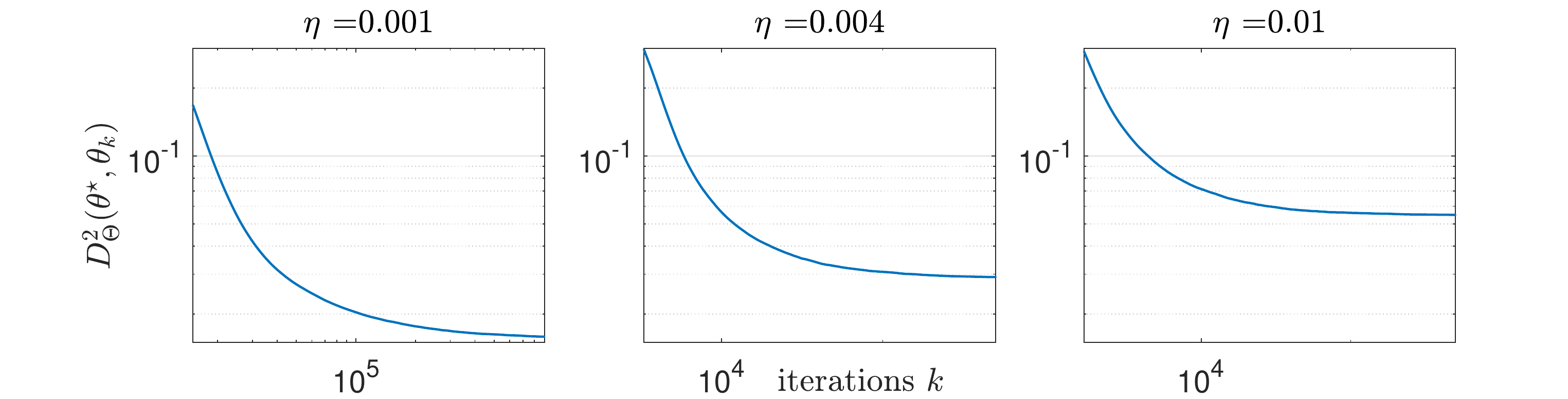}
	\caption{Paths of the algorithm in \Cref{theo:drift_bary}}
	\label{fig:path_unbounded}
\end{figure}

%%% Local Variables:
%%% mode: latex
%%% TeX-master: "main"
%%% End:

%% \bibliographystyle{plainnat}
\newpage
%\section*{Broader Impact}
%This work does not present any foreseeable societal consequence.
%

\subsubsection*{Acknowledgments}

  AD and EM acknowledge support  of the Lagrange Mathematical and Computing Research Center.
\bibliographystyle{apalike}%unsrt
\bibliography{bibliography}

\newpage

\appendix

\section{Supplementary notation}
%\alain{consider $V$-norm}
%We introduce some additional Riemannian notions that are used in the proofs below.
Denote the unit tangent space $\rmU_\theta \Theta = \defEnsLigne{u \in \planT_\theta 
\Theta \,:\, \normrLigne{u}[\theta]=1}$.
The cut-locus of $\theta$, $\Cut(\theta) \subset \Theta$ \cite[p. 308]{lee:2019} 
and the injectivity domain $\ID(\theta) \subset \planT_\theta \Theta$ \cite[p. 310]{lee:2019}
are two notions that inform us about the length-minimizing properties of geodesics, 
and therefore provide the domain of definition of the Riemannian exponential.
On a complete and connected manifold, \cite[Theorem 10.34]{lee:2019} holds, 
meaning the restriction $(\Exp_\theta)_{|\ID(\theta)} : \ID(\theta) \to \Theta$ 
is a diffeomorphism onto its image $\Theta \setminus \Cut(\theta)$. 
We simply denote $\Exp_\theta^{-1} : \Theta \setminus \Cut(\theta) \to \ID(\theta)$ its inverse. 
Under the assumption that $\Theta$ is complete, simply connected and of non-positive sectional curvature, 
\ie~a Hadamard manifold, \cite[Proposition 12.9]{lee:2019} proves that $\Cut(\theta)=\emptyset$
and $\ID(\theta)=\planT_\theta \Theta$ for any $\theta\in \Theta$.

For a measure $\mu$ on a measurable space $(\msy,\mcy)$, denote by
$\mu(g)$ the integral of a measurable function $g : \msy \to \rset$
with respect to $\mu$, when it exists.

%%% Local Variables:
%%% mode: latex
%%% TeX-master: "main"
%%% End:

\section{Proofs of \Cref{sec:constant_step}}\label{app:constant_step}
Under \Cref{ass:had_or_complete} and \Cref{ass:0mean_noise}, for any $\upeta>0$, we denote by  $Q_{\upeta}$ the Markov kernel associated with $(\theta_n)_{n\in\nset}$ defined by \eqref{eq:scheme} given for any $\msa \in \mcb{\mss}$ and $\theta \in \mss$ by
\begin{equation}
  \label{eq:def_Q_upeta}
  Q_{\upeta}(\theta,\msa) = \expe{\1_{\msa}\parenthese{\Exp_{\theta}\defEns{\upeta H_{\theta}(X_{1})}}} \eqsp. 
\end{equation}
Useful notions, definitions and results relative to Markov chain theory are given in \Cref{app:markov}.
%
%We start by introducing lemmas setting a useful framework to analyze weak convergence of Markov chains.
\begin{llemma}
  \label{lem:drift1}
   Assume \Cref{ass:had_or_complete}, \Cref{ass:0mean_noise},
  \Cref{ass:lyap}-\ref{ass:lyap_contractive}-\ref{ass:lyap_grad_lips}. Then for any $\upeta >0$ and $\theta_0 \in \mss$,
  \begin{equation}\label{eq:drift1}
Q_\upeta V(\theta_0) \leq V(\theta_0) + \upeta \psr{\grad V(\theta_0)}{h(\theta_0)}[\theta_0] + L \upeta^2 \parentheseDeux{\normr{h(\theta_0)}[\theta_0]^2 + \sigmaZ + \sigmaU\normr{h(\theta_0)}[\theta_0]^2} \eqsp.
\end{equation}

\end{llemma}

\begin{proof}
  Let $\theta_0 \in \mss$, and $\upeta>0$.  %We first show \eqref{eq:drift1}. 
Consider
\begin{equation} \label{eq:unconstrained}
\theta_{1/2} = \Exp_{\theta_0}\parentheseDeux{ \upeta H_{\theta_0}(X_1)} \eqsp, \eqsp \theta_1=\proj_\mss\parenthese{\theta_{1/2}} \eqsp.
\end{equation}
First, by definition of $Q_\upeta$ and \Cref{ass:lyap}-\ref{ass:lyap_contractive}, we have
\begin{equation} \label{eq:contraction}
Q_\upeta V(\theta_0) = \expe{V(\theta_1)} \leq \expe{V(\theta_{1/2})} \eqsp.
\end{equation}
Second, using \Cref{ass:had_or_complete}, \Cref{ass:lyap}-\ref{ass:lyap_grad_lips}, \cite[Lemma 1]{durmus:jimenez:moulines:said:wai:2020} and \eqref{eq:unconstrained}, we obtain
\begin{equation}
V(\theta_{1/2}) \leq V(\theta_0) + \upeta \psr{\grad V(\theta_0)}{H_{\theta_0}(X_1)}[\theta_0] + (L/2)\upeta^2 \normr{H_{\theta_0}(X_1)}[\theta_0]^2 \eqsp. 
\end{equation}
Plugging this result in \eqref{eq:contraction} and using \Cref{ass:0mean_noise}  completes the proof of  \eqref{eq:drift1}.
\end{proof}
\subsection{Proof of \Cref{theo:drift_lyap}}
\label{sec:proof-crefth_lyap}
\begin{enumerate}[wide, labelwidth=!, labelindent=0pt,label=(\alph*),noitemsep,nolistsep]
\item Using \Cref{lem:drift1} and \Cref{ass:lyap_meanfield} we have for any $\theta_0 \in \mss$ and $\upeta >0$,
\begin{equation}
Q_\upeta V (\theta_0) \leq V (\theta_0) + \upeta \{1-C_2 L \upeta (1+ \sigmaU) \}
\psr{\grad V (\theta_0)}{h(\theta_0)}[\theta_0] + L \upeta^2 [\sigmaZ + C_1 (1 + \sigmaU)] \eqsp.
\end{equation}
Letting $\bupeta = [2 C_2 L (1+\sigmaU)]^{-1}$, then for any $\upeta \in (0,\bupeta]$, we have $1-C_2 L \upeta (1+ \sigmaU) \geq 1/2 $. Therefore, using also that $\psr{\grad V (\theta_0)}{h(\theta_0)}[\theta_0]\leq 0$, we obtain,
\begin{equation}\label{eq:drift_intermed}
Q_\upeta V (\theta_0) \leq V (\theta_0) + (\upeta/2)
\psr{\grad V (\theta_0)}{h(\theta_0)}[\theta_0] + L \upeta^2 [\sigmaZ + C_1 (1 + \sigmaU)] \eqsp.
\end{equation}
Therefore, by the Markov property, for any $k \in \nsets$,
$\upeta \in (0,\bupeta]$ and $\theta_0 \in \mss$ we get,
\begin{equation}
- (\upeta/2) \int_{\Theta} \psr{\grad V (\theta)}{h(\theta)}[\theta] Q_\upeta^{k-1}(\theta_0,\rmd \theta)
\leq Q_\upeta^{k-1} V (\theta_0) - Q_\upeta^{k}V(\theta_0) + L \upeta^2 [\sigmaZ + C_1 (1 + \sigmaU)] \eqsp.
\end{equation}
Summing these inequalities for $k \in \{ 1, \dots, n \}$ concludes the proof of \ref{theo:drift_lyap0} upon using that $V$ is a non-negative function.
\item We prove \eqref{eq:theo::drift_lyap1_h} by 
using \Cref{ass:lyap_minorization}$(\msks)$ in \eqref{eq:theo:drift_lyap0_h} 
and dividing both sides by $\lambda >0$.
\item We start by using \Cref{ass:lyap_minorization}$(\msks)$ in
\eqref{eq:drift_intermed}. For any $\upeta\in (0,\bupeta]$ and $\theta_0 \in \mss$, we have
\begin{equation}\label{eq:drift_lyap2_intermed}
Q_\upeta V(\theta_0) \leq V(\theta_0) \parentheseDeux{1 - (\lambda \upeta /2) 
\1_{\mss \setminus \msks}(\theta_0)} + \upeta^2 b /2 \eqsp,
\end{equation}
where $b = 2 L [\sigmaZ + C_1 (1 + \sigmaU)]$. 
By adding and subtracting $V(\theta_0)(\lambda \upeta /2)\1_{\msks}(\theta_0)$ 
in the right-hand side of \eqref{eq:drift_lyap2_intermed}, we have,
\begin{equation}\label{eq:drift_lyap2_1step}
Q_\upeta V(\theta_0) \leq V(\theta_0) [1-\upeta a]
+ \upeta(b\upeta/2 + a \normr{V}[\msks]) \eqsp,
\end{equation}
where $a = \lambda/2$. Therefore, by a straightforward induction on $n \in \nset$, using the Markov property, we get, for any $n\in \nset$, 
$\upeta \in (0,\bupeta]$ and $\theta_0 \in \mss$,
\begin{align}
\expe{V(\theta_n)} &\leq \{1-\upeta a \}^n V (\theta_0) 
+ \upeta(b\upeta/2 + a \normr{V}[\msks]) \sum_{k=0}^{n-1} [1-\upeta a]^k \\
&\leq \{1- \upeta a \}^n V(\theta_0) + \{\normr{V}[\msks] + (b \upeta/2a)\} \eqsp,
\end{align}
which concludes the proof of \ref{theo:drift_lyap2} and \Cref{theo:drift_lyap}.
\end{enumerate}

\subsection{An alternative to \Cref{theo:drift_lyap}-\ref{theo:drift_lyap1}}
\label{sec:an-altern-crefth}

Consider the following condition for some compact set $\msks \subset \mss$.
\begin{assumptionS}[$\msks$] \label{ass:lyap_minorization_alter}  
  There exists $\lambda>0$  such that for any $\theta \in \mss$, $\psr{\grad V(\theta)}{h(\theta)}[\theta] \leq - \lambda \normr{h(\theta)}[\theta]^2\1_{\mss \setminus \msks}(\theta)$.
\end{assumptionS}

\begin{theorem}
  \label{theo:alter_theo:drift_lyap}
   Assume \Cref{ass:had_or_complete}, \Cref{ass:0mean_noise},
  \Cref{ass:lyap}-\ref{ass:lyap_contractive}-\ref{ass:lyap_grad_lips} and
  \Cref{ass:lyap_minorization_alter}$(\msks)$ hold for
  some compact set $\msks \subset \mss$, and define
  $\norm{h}_{\msks} = \sup \defEnsLigne{\normr{h(\theta)}[\theta] \,:\, \theta \in
    \msks}$ if $\msks \neq \emptyset$ and
  $\norm{h}_{\msks} = 0$ otherwise. Then for any $\upeta \in \ocint{0,\cupeta}$ and $\theta_0 \in \mss$, and $n \in \nsets$,
  \begin{equation}\label{eq:theo::drift_lyap1_h_alter}
n^{-1} \sum_{k=0}^{n-1} \PE[\1_{\mss \setminus \msks}(\theta_k)\normr{h(\theta_k)}[\theta_k]^2] \leq V(\theta_0)/(an\upeta)+ \upeta \tb/a \eqsp,
\end{equation}
where $(\theta_n)_{n \in\nset}$ is defined by \eqref{eq:scheme} starting from $\theta_0$, $\cupeta = \lambda/[2(1+\sigma^2_1)L]$, $a = \lambda/2$ and $\tb=L((1+\sigma_1^2)\norm{h}_{\msks} + \sigma_0^2 )$.
\end{theorem}
\begin{proof}
  By \Cref{lem:drift1} and \Cref{ass:lyap_minorization_alter}$(\msks)$,  for any $\upeta \in (0,\bar{\upeta}]$ and $\theta_0 \in \mss$, we have
    \begin{equation}\label{eq:drift1_2}
Q_\upeta V(\theta_0) \leq V(\theta_0) - \upeta \lambda  \normr{h(\theta_0)}[\theta_0]^2 \1_{\mss \setminus \msks}(\theta_0) + L \upeta^2 \parentheseDeux{\normr{h(\theta_0)}[\theta_0]^2 + \sigmaZ + \sigmaU\normr{h(\theta_0)}[\theta_0]^2} \eqsp.
\end{equation}
Therefore, by the Markov property, for any $k \in \nsets$, $\upeta \in (0,\bar{\upeta}]$ and $\theta_0 \in \mss$, we get
\begin{multline}
  \label{eq:6}
  (\upeta \lambda/2)  \int_{\Theta} \{ \1_{\mss \setminus \msks}(\theta) \normr{h(\theta)}[\theta]^2\} Q_{\upeta}^{k-1}(\theta_0,\rmd \theta)\\
  \leq  Q_\upeta^{k-1} V(\theta_0) - Q_\upeta^{k}V(\theta_0) +L\upeta^2((1+\sigma_1^2)\norm{h}_{\msks} + \sigma_0^2 ) \eqsp. 
\end{multline}
Summing these inequalities for $k \in \{1,\ldots,n\}$ concludes the proof upon using that $V$ is a non-negative function.
\end{proof}

\subsection{Proof of \Cref{thm:recurrent_ergodic}} \label{app:recurrent_ergodic}

\begin{llemma} \label{lem:feller}
 Assume \Cref{ass:had_or_complete}, \Cref{ass:0mean_noise} and \Cref{ass:MD:topo_prop_chain}-\ref{ass:item:feller}. Then the Markov kernel $Q_{\upeta}$ on $\mss \times \mcb{\mss}$ is Feller, \ie~ for any measurable bounded function $f:\mss \to \rset$, $Q_\upeta f$ is continuous from $\mss$ to $\rset$.
\end{llemma}
\begin{proof}
  The proof is an easy consequence of the Lebesgue dominated convergence theorem, since $h$ is continuous and \Cref{ass:MD:topo_prop_chain}-\ref{ass:item:feller} holds.
\end{proof}
For the next lemma, we introduce $\mu_\mss$, the restriction to $\mss$ of the Riemannian measure $\mu_\Theta$ associated with the volume form on $\Theta$.
\begin{llemma} \label{lem:irreducible_aperiodic}
 Assume  \Cref{ass:had_or_complete}, \Cref{ass:0mean_noise} and \Cref{ass:MD:topo_prop_chain}-\ref{ass:item:irreducible_aperiodic}. Then $Q_{\upeta}$ is $\mu_\mss$-irreducible and aperiodic.
\end{llemma}
\begin{proof}
We consider first the case \Cref{ass:had_or_complete}-\ref{ass:had_or_complete_i}, where $\Theta$ is a Hadamard manifold. 
  Let $\msa \in \mcb{\mss}$ be a Borel set of $\mss$, such that $\mu_\mss(\msa) > 0$. We only need
  to show that for any $\theta_0 \in \Theta$, $Q_{\upeta}(\theta_0,\msa) >0$. Indeed, this gives $\mu_\mss$-irreducibility by definition and implies that  the chain is aperiodic by \cite[Theorem 5.4.4]{meyn:tweedie:2009} since for any $\msa \in \mcb{\mss}$, $\mu_{\mss}(\msa)>0$, $\theta \in \msa$, we have $Q_{\upeta}(\theta,\msa) >0$.
  % the following argument.\pablo{adapter la preuve avec les hypothèses} For any $\msa \in \mcb{\mss}$ such that $\mu_\mss(\msa)>0$, and any $\theta \in \msa$, we have $Q_{\upeta}(\theta,\msa) >0$, therefore there is only one class in the cyclic decomposition of the state space $\mss$ given by .
  
Let $\theta_0 \in \mss$.
By definition of the scheme \eqref{eq:scheme} and $\proj_\mss$, $Q_{\upeta}(\theta_0,\msa)=\probaLigne{\proj_\mss \circ \Exp_{\theta_0}(\upeta \defEnsLigne{h(\theta_0) + \noise_{\theta_0}(X_1)}) \in \msa} \geq \probaLigne{\Exp_{\theta_0}(\upeta \defEnsLigne{h(\theta_0) + \noise_{\theta_0}(X_1)}) \in \msa}$. However, using \Cref{ass:MD:topo_prop_chain}-\ref{ass:item:irreducible_aperiodic}, the law of $\noise_{\theta_0}(X_1)$ has a positive density $\phi:\planT_{\theta_0} \Theta \to (0,+\infty)$ with respect to Lebesgue's measure $\Leb_{\theta_0}$. Denote $(\frg_{ij}(\theta))_{1\leq i, j \leq d }$ the matrix representing the Riemannian metric at $\theta\in \Theta$ in normal global coordinates at $\theta_0$. Expressing $\mu_\mss$ in these coordinates and using \cite[p.404 and Proposition 2.41]{lee:2019},
\begin{align}
&\probaLigne{\upeta \defEnsLigne{h(\theta_0) + \noise_{\theta_0}(X_1)} \in \Exp_{\theta_0}^{-1}(\msa)} = \int_{\Exp^{-1}_{\theta_0}(\msa)} \phi\parenthese{\upeta^{-1} v - h(\theta_0)} \rmd \Leb_{\theta_0} (v)  \\
& \qquad \qquad \qquad = \int_{\msa} \phi\parenthese{\upeta^{-1} \Exp^{-1}_{\theta_0}(\theta) - h(\theta_0)} \defEns{\det(\frg_{ij}(\theta))}^{-1/2}\rmd \mu_\mss (\theta)   >0 \eqsp,
\end{align}
since all quantities in the integral are positive and $\mu_\mss(\msa)>0$.

Now assume \Cref{ass:had_or_complete}-\ref{ass:had_or_complete_ii} and keep the notations of the first case. Then $\Exp_{\theta_0}:\planT_{\theta_0}\Theta \to \Theta$ is no longer a diffeomorphism. However, $(\Exp_{\theta_0})_{|\ID(\theta_0)} : \ID(\theta_0) \to \Theta \setminus \Cut(\theta_0)$ is a diffeomorphism, see \cite[Theorem 10.34]{lee:2019}. Moreover, as $\Cut(\theta_0)$ is a set of measure zero, see again \cite[Theorem 10.34]{lee:2019}, considering $\tilde{\msa}=\msa \setminus \Cut(\theta_0)$ allows the previous proof to give the desired result.
\end{proof}

\begin{proof}[Proof of \Cref{thm:recurrent_ergodic}]
  First, we prove that the chain is
  Harris-recurrent. For that, we start by proving, for any $\theta_0 \in \mss$,
\begin{equation}\label{eq:non-evanescent}
\probaMarkov{}{ \cup_{k \in \nsets} \cap_{N \in \nset} \cup_{n \geq N} \{\theta_n \in
  \boulefermee{\theta^\star}{k}\}} =1 \eqsp,
\end{equation}
where $(\theta_n)_{n \in \nset}$ is defined by \eqref{eq:scheme} and with initial condition $\theta_0$.

\Cref{theo:drift_lyap}-\eqref{eq:theo:drift_lyap2_h}
  implies that for any $\theta_0 \in\Theta$,
  $\sup_{n \in \nset} Q^n_{\upeta} V(\theta_0) < \plusinfty$; since
  $\norm{V}_{\msks}=\sup_{\msks}V < \plusinfty$ because $V$ is
assumed to be continuous. Therefore $\liminf_{n \to \plusinfty} V (\theta_n)$ is integrable by Fatou's lemma. Thus, for any $k\in \nsets$, using Markov's inequality,
\begin{equation}
\proba{\liminf_{n \to +\infty} V(\theta_n) > k} \leq \left. \expe{\liminf_{n\to +\infty} V(\theta_n)} \middle/ k \right. \eqsp.
\end{equation}
However, $\defEnsLigne{\liminf_{n \to +\infty} V(\theta_n) \leq k} = \cap_{N \in \nset} \cup_{n \geq N}\defEnsLigne{\theta_n \in V^{-1}([0,k])}$. Thus, for any $k \in \nsets$,
\begin{equation}
\proba{\cap_{N \in \nset} \cup_{n \geq N}\defEns{\theta_n \in V^{-1}([0,k])}} \geq 1- \left. \expe{\liminf_{n\to +\infty} V(\theta_n)} \middle/ k \right. \eqsp.
\end{equation}
Now, taking the union of these events for any $k \in \nsets$ gives
\begin{equation} \label{eq:non-evan1}
\proba{\cup_{k \in \nsets} \cap_{N \in \nset} \cup_{n \geq N}\defEns{\theta_n \in V^{-1}([0,k])}} = 1 \eqsp.
\end{equation}
Nonetheless, using \Cref{ass:lyap}-\ref{ass:item:lyap_proper}, for any $k\in \nsets$, $V^{-1}([0,k])$ is a subset of a compact set, therefore it is bounded. Thus, for any $k\in \nsets$, there exists $k' \in \nsets$ such that $V^{-1}([0,k]) \subset \boulefermee{\theta^\star}{k'}$. This gives the following,
\begin{equation}
\cup_{k \in \nsets} \cap_{N \in \nset} \cup_{n \geq N}\defEns{\theta_n \in V^{-1}([0,k])} \subset \cup_{k \in \nsets} \cap_{N \in \nset} \cup_{n \geq N}\defEns{\theta_n \in \boulefermee{\theta^\star}{k}} \eqsp.
\end{equation}
Combining this with \eqref{eq:non-evan1} gives \eqref{eq:non-evanescent}.
%we get by \Cref{ass:lyap}-\ref{ass:item:lyap_proper} that for any
%$\theta_0 \in \mss$,
%$$\probaMarkov{}{\cup_{k \in \nsets >0} \cap_{N \in \nset} \cup_{n \geq N} \{\theta_n \in
%  \boulefermee{\theta^\star}{k}\}} =1 \eqsp,$$ where $(\theta_n)_{n \in \nset}$ is defined by \eqref{eq:scheme} and with initial condition $\theta_0$. 

\Cref{eq:non-evanescent} gives that the chain is non-evanescent \cite[Section 9.2.1]{meyn:tweedie:2009}. 
Since $Q_{\upeta}$ is Feller (see \Cref{lem:feller}),
this result and \cite[Theorem 9.2.2]{meyn:tweedie:2009} imply that $Q_{\upeta}$ is Harris recurrent.

We now show that $Q_{\upeta}$ is $\tilde{V}$-uniformly geometrically ergodic (see \Cref{app:markov}) setting $\tilde{V} = 1 + V$. First, by
 \Cref{theo:drift_lyap} and \eqref{eq:drift_lyap2_1step} obtained in the proof above,
 we have that for any
$\theta_0\in \mss, \upeta \in (0,\bar{\upeta}]$,
\begin{equation}
Q_\upeta \tilde{V} (\theta_0) \leq (1- \upeta a) \tilde{V} (\theta_0) + \upeta( \upeta b/2 +  a(1+\normr{V}[\msks]) ) \eqsp, 
\end{equation}
where $a,b,\bar{\upeta}$ and $\normr{V}[\msks]$ are defined in
\Cref{theo:drift_lyap}. Then,
by \Cref{ass:lyap}-\ref{ass:item:lyap_proper} there exists $\tilde{r} >0$, such that for any $\theta_0 \in \mss$,
\begin{equation}
  Q_\upeta \tilde{V} (\theta_0) \leq (1- a \upeta/2 ) \tilde{V}(\theta_0)
  + \upeta( \upeta b/2 +  a(1+\normr{V}[\msks]) ) \1_{\boulefermee{\theta^\star}{\tilde{r}}}(\theta_0) \eqsp.
\end{equation}
%
% there exist $\tilde{a},\tilde{b},\tilde{r}$ such that for any $\theta_0 \in \Theta$,
%\begin{equation}
%  Q_{\upeta} V(\theta_0) \leq (1-\tilde{a}\upeta) V(\theta_0)+\upeta b \1_{\boulefermee{\theta^{\star}}{\tilde{r}}}(\theta_0) \eqsp. 
%\end{equation}
%\alain{completer et justifier etc que ça soit propre}
Then, since $Q_{\upeta}$ is Feller by \Cref{lem:feller} and $\mu_{\mss}$-irreducible by \Cref{lem:irreducible_aperiodic}, using \cite[Proposition 6.2.8 (ii)]{meyn:tweedie:2009}, $\boulefermee{\theta^{\star}}{r}$ is petite since it is compact by the Hopf-Rinow theorem \cite[Theorem 1.7.1]{jost:2005} and $\mss$ has non-empty interior by \Cref{ass:had_or_complete}. Therefore, an application of \cite[Theorem 16.0.1]{meyn:tweedie:2009} proves that the chain is
$\tilde{V}$-uniformly geometrically ergodic.
\end{proof}

\subsection{Proof of \Cref{cor:ball_dirac}} \label{app:ball_dirac}
\begin{llemma}\label{lem:majoration_mu_v}
Assume \Cref{ass:had_or_complete}, \Cref{ass:0mean_noise} \Cref{ass:MD:topo_prop_chain},   \Cref{ass:lyap}, \Cref{ass:lyap_meanfield} and \Cref{ass:lyap_minorization}$(\msks)$ hold for some compact set $\msks \subset \mss$. Then for any $\upeta \in (0,\bar{\upeta}]$,
\begin{equation}
  \label{eq:4}
  \mu^{\upeta} [V  \1_{\mss \setminus \msks}]  \leq 2\upeta  L\defEnsLigne{\sigmaZ + C_1(1+\sigmaU)}/\lambda \eqsp,
\end{equation}
where $\bupeta = [2 C_2 L (1+\sigmaU)]^{-1}$. %and $\norm{V}_{\msks} = \sup \defEnsLigne{V(\theta) \,:\, \theta \in
%    \msks}$ is $\msks \neq \emptyset$ and $\norm{V}_{\msks} =0$ otherwise.
\end{llemma}
\begin{proof}
For any $\upeta \in (0,\bar{\upeta}]$ and $M \geq 0$, setting
  $V_M = M \wedge V$,  \eqref{eq:drift_lyap2_intermed}
  implies using Jensen inequality, for any $\theta_0 \in \Theta$,
\begin{equation}
Q_\upeta V_M \parenthese{\theta_{0}} \leq (1-\upeta a \1_{\mss \setminus \msks}(\theta_0)) V_M \parenthese{\theta_{0}} + \upeta^2 b /2 \eqsp,
\end{equation}
where $\bupeta = [2 C_2 L (1+\sigmaU)]^{-1}$ , $b=2L\{\sigmaZ +C_1 (1+ \sigmaU)\}$
and $a=\lambda/2$.
Using that $\mu^{\upeta}$ is invariant for $Q_{\upeta}$ by \Cref{thm:recurrent_ergodic} and $V_M$ is bounded, we get $  \mu^{\upeta} [V_M \1_{\mss \setminus \msks}]  \leq \upeta b/(2a) $.
By the monotone convergence theorem, taking $M \to \plusinfty$, we have $  \mu^{\upeta} [V  \1_{\mss \setminus \msks}]  \leq \upeta b/(2a)$,
which concludes the proof.
\end{proof}
\begin{proof}[Proof of \Cref{cor:ball_dirac}]
\begin{enumerate}[wide, labelwidth=!, labelindent=0pt,label=(\alph*),noitemsep,nolistsep]
\item Using \Cref{lem:majoration_mu_v} and $V(\theta) \geq c >0$ for any $\theta \in \mss \setminus \msks$, we obtain
\begin{equation}
\mu^\upeta \defEns{\mss \setminus \msks} \leq \upeta b / (2ac) \eqsp,
\end{equation}
which concludes the proof of \ref{cor:ball_dirac_i} taking the limit $\upeta \to 0$.

\item Let $(\upeta_n)_{n \in \nset}$ be a sequence converging to zero such that for any $n \in \nset$, 
$\upeta_n \in (0,\bupeta]$.
  We start by proving that $(\mu^{\upeta_n})_{n \in \nset}$ is tight. Let $\varepsilon>0$.
On one hand, let $r>0$ and $\msk_0=\boulefermee{\theta^\star}{r}$. Then, using \Cref{cor:ball_dirac}-\ref{cor:ball_dirac_i}, there exists $N\in \nset$ such that for any $n\geq N$, $\mu^{\upeta_n}(\msk_0)\geq 1- \varepsilon$. On the other hand, $( \mu^{\upeta_n})_{n \in \{0,\dots, N-1\}}$ is tight, \ie~there exists a compact set $\tilde{\msk}\subset \Theta$ such that for any $n\in \{1,\dots,N-1\}$, $\mu^{\upeta_n}(\tilde{\msk}) \geq 1- \varepsilon$.
Finally, taking $\msk= \msk_0 \cup \tilde{\msk}$ gives the tightness of $(\mu^{\upeta_n})_{n \in \nset}$.
Now, let $\mu$ be a limit point of $(\mu^{\upeta_n})_{n \in \nset}$. Using \Cref{cor:ball_dirac}-\ref{cor:ball_dirac_i}, and Lebesgue's dominated convergence theorem letting $r\to 0$, gives $\mu(\defEnsLigne{\theta^\star})=1$, \ie~$\mu = \updelta_{\theta^\star}$. In conclusion, for any $(\upeta_n)_{n \in \nset}$ converging to zero, $(\mu^{\upeta_n})_{n \in \nset}$ converges weakly to the Dirac at $\theta^\star$.
\end{enumerate}
\end{proof}
\subsection{Proof of \Cref{prop:sqdistance}} \label{app:sqdistance}

First, we check \Cref{ass:lyap}-\ref{ass:lyap_contractive}. Using \cite[Proposition 2.6]{sturm:2003}, $\proj_\mss$ is a contraction \wrt~$\distT$, which implies that for any $\theta \in \Theta$,
\begin{equation}
\distT^2(\theta^{\star},\proj_\mss(\theta)) =
\distT^2(\proj_\mss(\theta^{\star}),\proj_\mss(\theta)) \leq
\distT^2(\theta^{\star},\theta) \eqsp.
\end{equation}
This implies, since $\mss\subset \msh$, that
\begin{equation}
  \label{eq:2}
  \Vdist(\proj_{\mss}(\theta)) = \distT^2(\theta^{\star},\proj_\mss(\theta)) \leq \chi_{\msh}(\theta) \distT^2(\theta^{\star},\theta) + (1-\chi_{\msh}(\theta)) \diam^2(\bar{\msh}) =  \Vdist(\theta) \eqsp,
\end{equation}
which gives \Cref{ass:lyap}-\ref{ass:lyap_contractive}.

To prove \Cref{ass:lyap}-\ref{ass:lyap_grad_lips}, we calculate the operator norm of the Hessian of $\Vdist$
and conclude by \cite[Lemma 10]{durmus:jimenez:moulines:said:wai:2020}.
  Using \Cref{ass:hadamard_curvature} and \cite[Theorem
  5.6.1]{jost:2005}, $\theta \mapsto \distT^2(\thetas,\theta)$ is
  smooth and its gradient on $\Theta$ is given  by 
  $\theta \mapsto -2\Exp^{-1}_\theta(\theta^{\star})$.
Therefore, for any $\theta \in \Theta$,
\begin{equation}
\grad \Vdist (\theta) = [\distT^2(\thetas,\theta)-\rmD_\msh^2] \grad \chi_\msh (\theta) 
- 2 \chi_\msh(\theta)\Exp_\theta^{-1}(\thetas) \eqsp.
\end{equation}
Using now \Cref{ass:hadamard_curvature}, \cite[Theorem 5.6.1]{jost:2005} and Cauchy-Schwarz's inequality 
brings, for any $\theta \in \Theta, v \in \planT_\theta \Theta$,
\begin{align}
\normr{(\Hess \Vdist)_\theta (v,v)}[\theta] &\leq 2 \kappa \distT(\thetas,\theta) 
\coth (\kappa \distT(\thetas,\theta)) \chi_{\msh}(\theta) \normr{v}[\theta]^2
+ 4 \distT(\thetas,\theta) \normr{\grad \chi_\msh (\theta)}[\theta] \normr{v}[\theta]^2 \\
&\quad + \normr{(\Hess \chi_\msh)_\theta (v,v)}[\theta] \abs{\distT^2(\thetas,\theta)- \rmD_\msh^2} \eqsp.
\end{align}
However, one can choose $\chi_\msh$ such that for any $\theta \in \Theta$ satisfying
$\inf_{\theta' \in \msh} \distT(\theta',\theta)\geq 1$, it holds that
$\chi_\msh(\theta)=0$.
Therefore, for any $\theta \in \Theta$, $\distT(\thetas,\theta) \chi_\msh(\theta) \leq \rmD_\msh + 1$.
%; in which case
%we also have $\grad \chi_\msh (\theta)=0$ and $(\Hess \chi_\msh)_\theta=0$.
Since $\chi_\msh$ is smooth with compact support, there exists a constant $M>0$ 
%(that can be shown to be universal) 
such that for any $\theta \in \Theta$ and $v \in \planT_\theta \Theta$,
\begin{equation}
\normrLigne{\grad \chi_\msh (\theta)}[\theta]\leq M \eqand 
\normrLigne{(\Hess \chi_\msh)_\theta(v,v)}[\theta]\leq M \normr{v}[\theta]^2 \eqsp.
\end{equation}
Therefore, combining these expressions brings for any $\theta \in \Theta$ and $v \in \planT_\theta \Theta$,
\begin{equation}
\normr{(\Hess \Vdist)_\theta (v,v)}[\theta] \leq 6(M+1) (\rmD_\msh+1)[1+ \kappa \coth(\kappa \rmD_\msh)] \normr{v}[\theta]^2 \eqsp,
\end{equation} 
thus proving by \cite[Lemma 10]{durmus:jimenez:moulines:said:wai:2020} and
setting $C_\chi=6(M+1)$, that \Cref{ass:lyap}-\ref{ass:lyap_grad_lips} holds with 
$L \leftarrow C_\chi (1+ \rmD_\msh)[1+ \kappa \coth(\kappa \rmD_\msh)]$.

%\Cref{ass:lyap}-\ref{ass:lyap_grad_lips} follows from  the fact that $\Hess \Vdist$ is
%smooth by \Cref{ass:hadamard_curvature}-\cite[Theorem 5.6.1]{jost:2005} and has compact support by
%definition, and an application of \cite[Lemma 10]{durmus:jimenez:moulines:said:wai:2020}.

We now turn  on checking \Cref{ass:lyap_minorization}$(\cball{\thetas}{r})$. 
  Since $\grad \chi_{\msh}(\theta) = 0$ for any $\theta \in \mss$,  we get
  that $\Vdist$ is smooth and for any
  $\theta\in \mss, \grad \Vdist
  (\theta)=-2\Exp^{-1}_\theta(\theta^{\star})$ Therefore
\Cref{ass:lyap_minorization}$(\cball{\thetas}{r})$ holds by \eqref{eq:prop:sqdistance}.

\subsection{Proof of \Cref{prop:huber}} \label{app:huber}

First, we check \Cref{ass:lyap}-\ref{ass:lyap_contractive}. Using \cite[Proposition 2.6]{sturm:2003}, $\proj_\mss$ is a contraction \wrt~$\distT$, which implies that $\theta \in \Theta$,
\begin{equation}
\distT(\theta^{\star},\proj_\mss(\theta)) =
\distT(\proj_\mss(\theta^{\star}),\proj_\mss(\theta)) \leq
\distT(\theta^{\star},\theta) \eqsp.
\end{equation}
Then the proof of  \Cref{ass:lyap}-\ref{ass:lyap_contractive} is completed using that $x\mapsto \delta^2\defEnsLigne{(x/\delta)^2+1}^{1/2}-\delta^2$ is increasing.
% , thus
% \begin{equation}
% \VHuber(\proj_\mss(\theta)) \leq \VHuber(\theta) \eqsp.
% \end{equation}

Next, using \Cref{ass:hadamard_curvature}, \cite[Lemma 16]{durmus:jimenez:moulines:said:wai:2020}, we have for any $\theta \in \Theta , v\in \planT_\theta \Theta \setminus \{ 0 \}$,
\begin{equation}
0 < \Hess \VHuber(\theta)(v,v) \leq (1+\kappa \delta) \normr{v}[\theta]^2 \eqsp.
\end{equation}
Therefore, using \cite[Lemma 10]{durmus:jimenez:moulines:said:wai:2020}, \Cref{ass:lyap}-\ref{ass:lyap_grad_lips} holds for $L=1+\kappa \delta$. It is easy to see that as $\distT(\thetas,\theta)\to \infty$, $\VHuber(\theta) \to \plusinfty$, meaning \Cref{ass:lyap}-\ref{ass:item:lyap_proper} holds by the Hopf-Rinow theorem \cite[Theorem 1.7.1]{jost:2005}.

Regarding \Cref{ass:lyap_minorization}$(\cball{\thetas}{r})$, using \cite[Lemma 16]{durmus:jimenez:moulines:said:wai:2020}, we have for any $\theta \in \Theta$,
\begin{equation} \label{eq:Vgrad}
 \grad \VHuber(\theta) = - \left. {\Exp^{-1}_\theta(\theta^{\star})} \middle/ \defEns{\parenthese{\distT(\theta^{\star},\theta)/\delta}^2+1}^{1/2} \right. \eqsp,
\end{equation}
Therefore for any $\theta \in \Theta$,  we get 
\begin{align}\label{eq:gradientterm}
 \psr{\grad \VHuber(\theta)}{h(\theta)}[\theta] &= -\left. \psr{\Exp^{-1}_{\theta}(\theta^{\star})}{h(\theta)}[\theta] \middle/ \defEns{\parenthese{\distT(\theta^{\star},\theta)/\delta}^2+1}^{1/2} \right. \eqsp.
\end{align}
Then, under the condition \eqref{eq:prop:sqdistance}, we obtain
\begin{equation}\label{eq:gradterm}
\begin{aligned}
\psr{ \grad \VHuber(\theta) }{h(\theta)}[\theta] &\leq \left. -\lambda_{\rho} \distT^2(\theta^{\star},\theta) \1_{\mss \setminus \cball{\theta^{\star}}{r}}(\theta) \middle/ \defEns{\parenthese{\distT(\theta^{\star},\theta) \middle/ \delta}^2  +1 }^{1/2}  \right. \\ 
&\leq -\lambda_{\rho} \VHuber(\theta)\1_{\mss \setminus \cball{\theta^{\star}}{r}}(\theta) \eqsp,
\end{aligned}
\end{equation}
where we used that 
\begin{equation} \label{eq:Vinequality}
 \VHuber(\theta) \leq \left. {\distT^2(\theta^{\star},\theta)} \middle/ \defEns{\parenthese{\distT(\theta^{\star},\theta)/\delta}^2+1}^{1/2} \right. \eqsp,
\end{equation}
since for any $a >0$ and $x \geq 0$, $(ax^2+1)^{1/2}-1 = a \int_0^x t\{at^2+1\}^{-1/2}\rmd t \leq ax^2/\{ax^2+1\}^{1/2}$.
%Finally, \Cref{ass:lyap_minorization}$(\cball{\thetas}{r})$-\ref{ass:lyap_minorization_ii} is straightfoward by definition. 
% Finally, for any $\theta \in \Theta$,  $\VHuber(\theta)=v \circ \distT(\theta^{\star},\theta)$. Thus, that for any $\varepsilon>0$, \Cref{ass:lyap}-\ref{ass:lyap_minorization}$(r)$ holds for $c=v(\varepsilon)$.

%%% Local Variables:
%%% mode: latex
%%% TeX-master: "main_supplement"
%%% End:

\section{Proofs of \Cref{sec:unconstrained}} \label{app:unconstrained}
%\pabloi{A mon avis, le \Cref{lem:taylor_kernel_1} n'utilise pas plus d'hypotheses que ça,
%  et n'utilise pas $\bupeta$ non plus, ni les resultats precedents... STP verifier}
For any $K \in \rset_+$, consider a smooth function with compact support  $\chi_K : \rset_+ \to [0,1]$ such that $\chi_K(t) = 1$ for any $t \leq K$ and $\chi_K(t) = 0$ for any $t \geq K+1$. 
\begin{llemma}\label{lem:taylor_kernel_1}
Assume \Cref{ass:had_or_complete}-\ref{ass:had_or_complete_ii} and \Cref{ass:0mean_noise}. %and there exists $C_3 \geq 0$ such that for any $\theta\in\Theta$, $\normr{h(\theta)}[\theta] \leq C_3 \distT(\thetas,\theta)$ for some $\thetas \in \Theta$. % \Cref{ass:lyap},
%\Cref{ass:lyap_minorization}-\ref{ass:lyap_attractive} for some $r>0$ and
%let $\bupeta = [2 C_2 L (1+\sigmaU)]^{-1}$. 
\begin{enumerate}[wide, labelwidth=!, labelindent=0pt,label=(\alph*),noitemsep,nolistsep]
\item \label{lem:item:taylor_kernel_1:a} Then, for any smooth function with compact support $\gVar:\Theta \to \rset$, %$\gPlan:\planT_\thetas \Theta \to \rset$, $\gInter:I \to \rset$ 
any $\upeta>0$ and $\theta_0 \in \Theta$,% we have,
\begin{equation}\label{eq:taylor_q}
Q_{\upeta}\gVar(\theta_0) 
= \gVar (\theta_0) + \upeta \psr{\grad \gVar(\theta_0)}{h(\theta_0)}[\theta_0] + (\upeta^2/2) \parentheseDeux{\Hess \gVar : \Sigma + h\otimes h}(\theta_0) + (\upeta^2/6) \scrR_{\gVar,\upeta}(\theta_0) \eqsp,
\end{equation}
where for any $K>0$,
\begin{align}
  \label{eq:control_remainder_q_a}
&\abs{\scrR_{\gVar,\upeta}(\theta_0)} \leq 8 \upeta   \expe{\normrLigne{\nabla \Hess \gVar}[\upgamma,\infty] \1_{\msa_{\theta_0}^\complement}\normr{H_K}[\theta_0]^3}  %\1_{\msk_K}(\theta_0)
+ 16 \normrLigne{\Hess \gVar}[\infty] \expe{\normr{Y_K}[\theta_0]^2} \eqsp, \\
  \label{eq:def_hk}
&H_K=h(\theta_0)+\noise_{\theta_0}(X_1)\chi_K(\normrLigne{\noise_{\theta_0}(X_1)}[\theta_0]) \eqsp, \quad
  Y_K =\noise_{\theta_0}(X_1)\{1-\chi_K(\normrLigne{\noise_{\theta_0}(X_1)}[\theta_0])\}\eqsp, \qquad \\
  & \Vert \Hess \gVar\Vert_\infty=\sup \defEnsLigne{\absLigne{\Hess\gVar_{\theta}(u,u)}\,:\, 
    \theta\in \Theta, u \in \rmU_{\theta} \Theta} \eqsp,\\
  & \Vert \nabla \Hess \gVar\Vert_{\upgamma,\infty}=
\sup \defEnsLigne{\absLigne{\nabla \Hess\gVar_{\upgamma(t)}(u,u,u)}\,:\, 
t \in [0,1], u \in \rmU_{\upgamma(t)} \Theta} \eqsp,
\end{align}
$\msa_{\theta_0} = \defEnsLigne{\normrLigne{H_K}[\theta_0] \leq \normrLigne{Y_K}[\theta_0]}$ and $\upgamma:[0,1] \to \Theta$ 
is defined for any $t \in\ccint{0,1}$ by $\upgamma(t)=\Exp_{\theta_0}(t\upeta H_{\theta_0}(X_1))$.
\item \label{lem:item:taylor_kernel_1:b}Assume in addition that there exist $C_3>0$ and $\thetas \in \Theta$ 
such that for any $\theta \in \Theta$,
$\normrLigne{h(\theta)}[\theta] \leq C_3 \distT(\thetas,\theta)$. 
Then, for any smooth function with compact support $g : \Theta \to \rset$, 
any $\upeta \in \ocintLigne{0,(4C_3)^{-1}}$ and $\theta_0 \in \Theta$,
\eqref{eq:taylor_q} holds, with for any $K>0$,
\begin{equation}\label{eq:control_remainder_q_b}
\abs{\scrR_{\gVar,\upeta}(\theta_0)} \leq 8 \upeta  \1_{\msk_K}(\theta_0) \expe{\normrLigne{
\nabla \Hess \gVar}[\upgamma,\infty] \1_{\msa_{\theta_0}^\complement}\normr{H_K}[\theta_0]^3} 
+ 16 \normrLigne{\Hess \gVar}[\infty] \expe{\normr{Y_K}[\theta_0]^2} \eqsp,
\end{equation}
where we take the notation of \ref{lem:item:taylor_kernel_1:a} and
$\msk_K$ is a compact subset of $\Theta$.
\end{enumerate}
\end{llemma}

\begin{proof}
\begin{enumerate}[wide, labelwidth=!, labelindent=0pt,label=(\alph*),noitemsep,nolistsep]
\item Let $\gVar:\Theta \to \rset$ be a smooth function with compact support and $\theta_0 \in \Theta$. Using \eqref{eq:scheme}, \Cref{ass:had_or_complete}-\ref{ass:had_or_complete_ii} and the definition of $Q_{\upeta}$ \eqref{eq:def_Q_upeta}, we have
\begin{equation} \label{eq:Qeta1}
  Q_\upeta \gVar({\theta_0}) = \expe{ \gVar\defEns{\Exp_{\theta_0}[\upeta H_{\theta_0}(X_1)]}} \eqsp.
\end{equation}
Consider the geodesic $\upgamma:[0,1] \to \Theta$ defined for any $t \in\ccint{0,1}$ by $\upgamma(t)=\Exp_{\theta_0}(t\upeta H_{\theta_0}(X_1))$.
 For any $t\in [0,1]$, let $\gInter(t) = (\gVar\circ \upgamma) (t)$. We compute now its derivatives 
 to derive a Taylor expansion.
Using \cite[Proposition 4.15-(ii) and Theorem 4.24-(iii)]{lee:2019}, we have for any $t \in [0,1]$,
\begin{equation}
\gInter'(t) = \rmD_t (\gVar\circ \upgamma)(t) = \psrLigne{\grad \gVar (\upgamma(t))}{\dot{\upgamma}(t)}[\upgamma(t)] \eqsp.
\end{equation}
By definition of the Hessian \cite[Example 4.22]{lee:2019} and using
$\rmD_t \dot{\upgamma}(t)=0$,
\Cref{prop:ten_cov_der}-\eqref{eq:ten_cov_der}-\ref{prop:item:ten_cov_der_extension},
we get for any $t \in [0,1]$,
\begin{equation}
\gInter''(t)= [\rmD^2_t \gInter](t)  = \Hess \gVar_{\upgamma(t)}(\dot{\upgamma}(t),\dot{\upgamma}(t)) \eqsp,
\end{equation}
In addition, using $\rmD_t \dot{\upgamma}(t)=0$ and
\Cref{prop:ten_cov_der}-\eqref{eq:ten_cov_der}-\ref{prop:item:ten_cov_der_extension},
we obtain for any $t \in [0,1]$,
\begin{equation}
  \gInter^{(3)}(t)= [\rmD^3_t \gInter](t) = \nabla \Hess \gVar_{\upgamma(t)} (\dot{\upgamma}(t),\dot{\upgamma}(t),\dot{\upgamma}(t))  \eqsp,
\end{equation}
where $\nabla \Hess \gVar$ is the total covariant derivative of
$\Hess \gVar$ \cite[Proposition 4.17]{lee:2019}.  Finally, for any
$K >0$, consider the two random tangent vectors at ${\theta_0}$
defined in \eqref{eq:def_hk}.  Now, writing the first-order Taylor
expansion of $\gInter:[0,1] \to \rset$, at $t = 1$ on the event
$\msa_{\theta_0}=\defEnsLigne{\normrLigne{H_K}[{\theta_0}]\leq
  \normrLigne{Y_K}[{\theta_0}]}$, the second-order one on the
complement, and summing both expansions, we get
\begin{equation} \label{eq:Qeta2}
\begin{aligned}
\gVar\parenthese{\Exp_{\theta_0} (\upeta H_{\theta_0}(X_1))} &= \gVar ({\theta_0}) + \upeta \psrLigne{\grad \gVar ({\theta_0})}{H_{\theta_0}(X_1)}[{\theta_0}] \\ 
&\quad+ (\upeta^2/2) \Hess \gVar_{\theta_0} (H_{\theta_0}(X_1),H_{\theta_0}(X_1))  + \scrR_{\gVar,\upeta} ({\theta_0},X_1)/6 \eqsp,
\end{aligned}
\end{equation}
where the remainder term is given by
\begin{equation}\label{eq:Qeta4}
\begin{aligned} 
  \scrR_{\gVar,\upeta}({\theta_0},X_1) &= \1_{\msa_{\theta_0}^\complement} 
  \int_0^1 \nabla \Hess \gVar_{\upgamma(t)} (\dot{\upgamma}(t),\dot{\upgamma}(t),\dot{\upgamma}(t)) \rmd t  \\
   &\quad + \1_{\msa_{\theta_0}} \parentheseDeux{\int_0^1 \Hess 
   \gVar_{\upgamma(t)}(\dot{\upgamma}(t),\dot{\upgamma}(t)) \rmd t 
   - 3\upeta^2 \Hess \gVar_{{\theta_0}} (H_{{\theta_0}}(X_1),H_{{\theta_0}}(X_1)) }\eqsp.
\end{aligned}
\end{equation}
%Since $\gVar$ is smooth with compact support, using the dominated convergence theorem proves that 
%for any $x\in \msx, \theta \mapsto \scrR_{\gVar,\upeta}({\theta},x)$ is continuous. Moreover, 
We bound the remainder as follows. Since $\gVar$ has compact support, 
$\Hess \gVar$ and $\nabla \Hess \gVar$ have an operator norm uniformly 
bounded over $\Theta$, which we express in the following way. 
For any $\theta \in \Theta$, consider the unit tangent space at 
$\theta$, $\rmU_\theta \Theta=\defEnsLigne{v \in \planT_\theta \Theta \,:\, \normrLigne{v}[\theta]=1}$, 
let $\normrLigne{\Hess \gVar}[\infty]=\sup \defEnsLigne{\absLigne{\Hess \gVar_\theta (v,v)} 
\, : \, \theta \in \Theta, v \in \rmU_\theta \Theta}$ and 
$\normrLigne{\nabla \Hess \gVar }[\upgamma,\infty]=
\sup \defEnsLigne{ \absLigne{\nabla \Hess \gVar_{\upgamma(t)} (v,v,v)} 
\,:\, t \in [0,1], v \in \rmU_{\upgamma(t)} \Theta} $. 
Then, using \cite[Corollary 5.6-$(b)$]{lee:2019}, and $\dot{\upgamma}(0)=\upeta H_{\theta_0}(X_1)$,
\begin{align}
\abs{\scrR_{\gVar,\upeta}({\theta_0},X_1) } &\leq
\1_{\msa_{\theta_0}^\complement} \Vert \nabla \Hess \gVar \Vert_{\upgamma,\infty} \int_0^1 \normr{\dot{\upgamma}(t)}[\upgamma(t)]^3 \rmd t \\
& \quad + \1_{\msa_{\theta_0}} \Vert \Hess \gVar \Vert_{\infty} \parentheseDeux{\int_0^1 \normr{\dot{\upgamma}(t)}[\upgamma(t)]^2 \rmd t + 3\upeta^2 \normr{H_{{\theta_0}}(X_1)}[{\theta_0}]^2} \\
&= \1_{\msa_{\theta_0}^\complement} \Vert \nabla \Hess \gVar \Vert_{\upgamma,\infty} \upeta^3 \normr{H_{{\theta_0}}(X_1)}[{\theta_0}]^3  + 4 \1_{\msa_{\theta_0}}  \Vert \Hess \gVar \Vert_{\infty} \upeta^2 \normr{H_{{\theta_0}}(X_1)}[{\theta_0}]^2 \eqsp.
\end{align}
Moreover, using that $H_K+Y_K = H_{\theta_0}(X_1)$ and the definition of $\msa_{\theta_0}$,
\begin{align}\label{eq:control_remainder_x}
\abs{\scrR_{\gVar,\upeta}({\theta_0},X_1) } %&\leq
%\1_{\msa_{\theta_0}^\complement} 8 \Vert \nabla \Hess \gVar \Vert_{\upgamma,\infty} \upeta^3 \normr{H_K}[{\theta_0}]^3 \\
%&\quad + \1_{\msa_{\theta_0}} 16 \Vert \Hess \gVar \Vert_{\infty} \upeta^2 \normr{Y_K}[{\theta_0}]^2 \\
&\leq  8\1_{\msa^\complement_{\theta_0}} \Vert \nabla \Hess \gVar \Vert_{\upgamma,\infty} \upeta^3 \normr{H_K}[{\theta_0}]^3 +  16 \Vert \Hess \gVar \Vert_{\infty} \upeta^2 \normr{Y_K}[{\theta_0}]^2 \eqsp.
\end{align}
Now, using \Cref{ass:0mean_noise},
\begin{equation} \label{eq:Qeta5}
\mathbb{E}\left[\psrLigne{\grad {\gVar}({\theta_0})}{H_{\theta_0}(X_1)}[{\theta_0}\,]\right] = \psrLigne{\grad {\gVar}({\theta_0})}{h({\theta_0})}[{\theta_0}]  \eqsp.
\end{equation}
In addition, since
\begin{equation}
\Hess {\gVar}_{{\theta_0}}(H_{\theta_0}(X_1),H_{\theta_0}(X_1)) = \left[\Hess {\gVar}:H_{\theta_0}(X_1)\otimes H_{\theta_0}(X_1)\right] \eqsp,
\end{equation}
it follows by a further application of \Cref{ass:0mean_noise}, that
\begin{equation} \label{eq:Qeta6}
\mathbb{E}\left[\Hess {\gVar}_{{\theta_0}}(H_{\theta_0}(X_1),H_{\theta_0}(X_1))\right] = 
%\mathbb{E}\left[\Hess {\gVar}:H_{\theta_0}(X_1)\otimes H_{\theta_0}(X_1)\right]
\left[\Hess {\gVar}:h \otimes h + \Sigma\right]({\theta_0}) \eqsp,
\end{equation}
where $\Sigma({\theta_0})$ is defined in \eqref{eq:SigmaT}. 
Using that $\normrLigne{H_K}[\theta_0] \leq K + \normrLigne{h(\theta_0)}[\theta_0]$, 
and \Cref{ass:0mean_noise} in \eqref{eq:control_remainder_x}, we obtain that for any
$\theta_0\in \Theta, \PE[\absLigne{ \scrR_{\gVar,\upeta}(\theta_0,X_1)}] < \plusinfty$. 
%
%Define $\scrR_{\gVar,\upeta}({\theta_0}) = 
%\upeta^{-2}\expeLigne{\scrR_{\gVar,\upeta}({\theta_0},X_1)}$.
Then, by \eqref{eq:Qeta2}, \eqref{eq:Qeta5} and \eqref{eq:Qeta6}, it follows from \eqref{eq:Qeta1},
\begin{equation} \label{eq:Qeta7}
Q_\upeta {\gVar}({\theta_0}) = {\gVar}({\theta_0}) + \upeta \psrLigne{\grad {\gVar}({\theta_0})}{h({\theta_0})}[{\theta_0}]  
+ (\upeta^2/2) \left[\Hess {\gVar}:h \otimes h + \Sigma\right]({\theta_0}) + \upeta^2 \scrR_{{\gVar},\upeta}({\theta_0}) /6 \eqsp,
\end{equation} 
where we define
$\scrR_{\gVar,\upeta}({\theta_0}) =
\upeta^{-2}\expeLigne{\scrR_{\gVar,\upeta}({\theta_0},X_1)}$.  The
desired bound on the remainder in \eqref{eq:control_remainder_q_a}, is a
simple consequence of \eqref{eq:control_remainder_x}.

\item In addition to the results of \ref{lem:item:taylor_kernel_1:a} and specifically 
\eqref{eq:control_remainder_q_a}, we need
to prove that, since $g$ has compact support,
there exists a compact set $\msk_K\subset \Theta$ such that
$\normrLigne{\nabla \Hess \gVar}[\upgamma,\infty] \1_{\msa_{\theta_0}^\complement} =0$ for any
$\theta_0 \not \in  \msk_K$. 

Using that
$\normr{h(\theta)}[\theta] \leq C_3 \distT(\thetas,\theta)$, we obtain
that on $\msa_{\theta_0}^\complement$,
$\normr{H_{\theta}(X_1)}[\theta] \leq 2(C_3 \distT(\thetas,\theta) +
K)$. In addition, by \cite[Corollary 6.12]{lee:2019},
$\distT(\theta,\upgamma(t)) = t \upeta
\normr{H_{\theta}(X_1)}[\theta]$ for any $t \in \ccint{0,1}$,
therefore for any $t \in \ccint{0,1}$ and
$\upeta \in \ocintLigne{0,(4C_3)^{-1}}$
\begin{equation}
  \label{eq:3}
  \distT(\thetas,\upgamma(t)) \geq   \distT(\thetas,\theta)- \distT(\theta,\upgamma(t)) 
  \geq (1-2\upeta t C_3)\distT(\thetas,\theta)  - 2\upeta K
  \geq \distT(\thetas,\theta)/2 -K/(2C_3) \eqsp.
\end{equation}
Consider now $R \geq 0$ such that for any
$\theta \not \in \cball{\thetas}{R}$, $g(\theta) = 0$. Then, setting
$\msk_K= \cball{\thetas}{2(R+K/(2C_3))}$, we obtain that for any
$\theta_0 \not \in \msk_K$ and $t \in \ccint{0,1}$,
$\upgamma(t) \not \in \cball{\thetas}{R}$ and therefore,
$\nabla \Hess \gVar_{\upgamma(t)} = 0$, which yields
$\normrLigne{\nabla \Hess \gVar}[\upgamma,\infty]
\1_{\msa_{\theta_0}^\complement} =0$ for any
$\theta_0 \not \in \msk_K$. Finally $\msk_K$ is a compact subset of
$\Theta$ by \cite[Theorem 1.7.1]{jost:2005}.
\end{enumerate}
\end{proof}

\subsection{Proof of \Cref{prop:bound_error_moment}}\label{app:bound_error_moment}
%\pabloi{Affaiblir hyp de \Cref{lem:taylor_kernel_1}, celles de \Cref{prop:bound_error_moment} suffisent} 
%pas sûr
%Consider a non-increasing 
%smooth function
%supported on $[0,2],\chi:[0,\plusinfty) 
%\to [0,1]$ such that for any $x\in
%[0,1],\chi(x)=1$ and define for any $\uprho \in [1,+
%\infty)$ the smooth function with compact support
%$V_\uprho : \theta \mapsto V(\theta)
%\chi(\uprho^{-1}\distT(\thetas,\theta))$. %%%Pas necessaire si \Theta est compact
Let $g :\Theta \to \rset$ be a smooth function. Since we assume  
that $\Theta$ is compact, $g$ is smooth with compact support.
%Therefore, we want to use \Cref{lem:taylor_kernel_1} on $g$,
%for any $\theta \in\Theta$. However, without the assumption that there exists
%$C_3 >0,\thetas \in \Theta$ such that for any $\theta \in \Theta$, 
%$h(\theta) \leq C_3 \distT(\thetas,\theta)$, we still have 
% we do not assume that %$\uprho\in [1,+\infty)$ 
Therefore, using \Cref{lem:taylor_kernel_1}-\ref{lem:item:taylor_kernel_1:a}
for any $\theta \in \Theta$ and $\upeta >0$, we have,
%\begin{equation}
%Q_\upeta V_\uprho (\theta) = V_\uprho(\theta) + \upeta
%\psr{\grad V_\uprho(\theta)}{h(\theta)}[\theta]
%+ (\upeta^2/2) [\Hess V_\uprho : \Sigma+ h \otimes h]
%(\theta) + (\upeta^2/6) \scrR_{V_\uprho,\upeta}(\theta)\eqsp,
%\end{equation}
%where $\scrR_{V_\uprho,\upeta}$ is defined in
%\eqref{eq:control_remainder_q}. On the left term,
%taking the limit when $\uprho \to \plusinfty$ and
%using the monotone convergence theorem gives that
%$\lim_{r \to \infty} Q_\upeta V_\uprho (\theta) =
%Q_\upeta V (\theta)$. On the right term, taking
%the limit when $\uprho \to + \infty$, we only need
%to bound the remainder uniformly in 
%$\uprho\in [1,+\infty)$. However, using
%\eqref{eq:hess_V}, define
%\begin{equation}
%\scrV = \sup \defEns{\max\parenthese{\abs{\Hess V_\theta (u,u)},\abs{\nabla \Hess V_\theta (u,u,u)}} \,:\, \theta \in \Theta,
%u \in \rmU_\theta \Theta} \eqsp.
%\end{equation}
%Moreover, using
%\eqref{eq:control_remainder_q}, we have, for any
%$r\in [1,+\infty)$ and $K>0$,
%\begin{equation}\label{eq:bound_error_v}
%\absLigne{\scrR_{V_\uprho,\upeta}(\theta)} \leq 8 \upeta \scrV \expe{\1_{\msa_\theta^\complement} \normrLigne{H_K}[\theta]^3} + 16 \scrV \expe{\normrLigne{Y_K}[\theta]^2} \eqsp,
%\end{equation}
%where $H_K,Y_K$ and $\msa_\theta$ are given in
%\eqref{eq:def_hk}. Therefore, we have,
\begin{equation}\label{eq:dl_v}
Q_\upeta g (\theta) = g(\theta) + \upeta
\psr{\grad g(\theta)}{h(\theta)}[\theta]
+ (\upeta^2/2) [\Hess g : \Sigma+ h \otimes h]
(\theta) + (\upeta^2/6) \scrR_{g,\upeta}(\theta)\eqsp,
\end{equation}
where using  \eqref{eq:control_remainder_q_a}, Hölder inequality and \Cref{ass:0mean_noise} gives,
\begin{multline}
\abs{\scrR_{g,\upeta}(\theta)} 
\leq 32 \upeta  (\normr{h(\theta)}[\theta]^3 + K^3) \sup \defEnsLigne{ 
\absLigne{\nabla \Hess g_\theta (u,u,u)} \,:\, \theta \in \Theta, u \in \rmU_\theta \Theta} \\
 + 16 \normr{\Hess g}[\infty] \parenthese{\sigmaZ + \sigmaU \normr{h(\theta)}[\theta]^2} \eqsp.
\end{multline}
%which is integrable \wrt~$\mu^\upeta$ on $\Theta$ since $h$ is assumed to be continuous.
Next, let $\upeta \in \ocint{0,\bupeta}$, where
$\bupeta = [2 C_2 L (1+\sigmaU)]^{-1}$. Note that since $\Theta$ is compact, $g$ is smooth, $h$ and $\Sigma$ are continuous, all the functions appearing in \eqref{eq:dl_v} are bounded. % we want to
% integrate \eqref{eq:dl_v} on $\Theta$ with respect to the stationary
% distribution $\mu^\upeta$, defined in \Cref{thm:recurrent_ergodic}.
%First, for any $\theta \in \Theta$, using \eqref{eq:control_remainder_q_a}, \Cref{ass:0mean_noise}
%Therefore, $\scrR_{g,\upeta}$ is bounded and integrable.
% Next, integrability of the other quantities 
% in \eqref{eq:dl_v} is checked using \Cref{ass:moment3}, since all the mappings 
% are continuous over the compact set $\Theta$.
%Using \Cref{thm:recurrent_ergodic}, for any
%$\upeta \in (0,\bupeta]$, $Q_\upeta$ is $V$-uniformly 
%geometrically ergodic, thus $V$ is $\mu^\upeta$-integrable
%\cite[Chapter 14]{meyn:tweedie:2009}. 
%Therefore, by stationarity of $\mu^\upeta$,
%$Q_\upeta V$ is also
%$\mu^\upeta$-integrable. Looking
%at \eqref{eq:dl_v}, $\psrLigne{\grad V}{h}$ 
%is integrable using the Cauchy-Schwarz
%inequality, that $\normr{h}$ is integrable
%using \Cref{ass:lyap}-\ref{ass:lyap_meanfield}
%and $\dots$ \pablo{$\grad V$ $\mu^\upeta$
%-integrable ?}.
%\Cref{ass:lyap_minorization}$(r)$-\ref{ass:lyap_attractive} 
%for some $r>0$. 
%$[\Hess V_\uprho : \Sigma +
%h \otimes h]$ is $\mu^\upeta$-integrable using 
%\eqref{eq:hess_V}, and that
%$\normrLigne{h}^2$ and $\Vert\Sigma \Vert$
%are integrable thanks to 
%\Cref{ass:lyap}-\ref{ass:lyap_meanfield} and
%\Cref{ass:0mean_noise} respectively. Finally,
%$\scrR_{V_\uprho , \upeta}$ is integrable 
%using \eqref{eq:dl_v} and the triangle
%inequality. Consequently, integrating
%\eqref{eq:dl_v} gives
Therefore, integrating \eqref{eq:dl_v} with respect to $\mu^{\upeta}$ given by  \Cref{thm:recurrent_ergodic} and using that $\mu^\upeta$ is invariant \wrt~$Q_\upeta$, we obtain,
\begin{equation}
-\int_\Theta \psr{\grad g (\theta)}{h(\theta)}
[\theta] \mu^\upeta (\rmd \theta) 
= (\upeta/2) \int_\Theta [ \Hess g : \Sigma
+ h \otimes h ] (\theta) \mu^\upeta (\rmd \theta)
+ (\upeta/6) \int_\Theta \scrR_{g,\upeta}
(\theta) \mu^\upeta (\rmd \theta) \eqsp.
\end{equation}
Using that $\theta \mapsto [ \Hess g : \Sigma + h \otimes h ] (\theta)$
is bounded and continuous over $\Theta$, \Cref{cor:ball_dirac}-\ref{cor:ball_dirac_ii} 
and that $h(\thetas)=0$, by weak convergence of $\mu^\upeta$ to $\updelta_\thetas$ 
when $\upeta \to 0$, we have,
\begin{equation}
\lim_{\upeta \to 0} \int_\Theta [ \Hess g : \Sigma
+ h \otimes h ] (\theta) \mu^\upeta (\rmd \theta) =
[ \Hess g : \Sigma+ h \otimes h ] (\thetas) 
=[ \Hess g : \Sigma ] (\thetas) \eqsp.
\end{equation}
Equivalently, there exists $\scrR_{\Hess g}:\ocintLigne{0,\bupeta} \to \rset$ such that for any 
$\upeta \in \ocintLigne{0,\bupeta}$, we have
\begin{equation}
\int_\Theta [ \Hess g : \Sigma
+ h \otimes h ] (\theta) \mu^\upeta (\rmd \theta) =
[ \Hess g : \Sigma ] (\thetas) + \scrR_{\Hess g}(\upeta)\eqsp,
\end{equation}
where $\lim_{\upeta \to 0}|\scrR_{\Hess g}(\upeta)|=0$.

To conclude, we prove that $\limsup_{\upeta \to 0} \absLigne{\int_\Theta \scrR_{g,\upeta}
(\theta) \mu^\upeta (\rmd \theta)} =0$. Let $K \geq 0$. By \eqref{eq:control_remainder_q_a}, since $\theta_0 \mapsto \expeLigne{\1_{\msa_{\theta_0}^\complement}\normrLigne{H_K}[\theta_0]^3}$ is uniformly bounded over $\Theta$ by definition \eqref{eq:def_hk} and since $h$ is continuous, we have that 
\begin{align}
\limsup_{\upeta \to 0} \abs{\int_\Theta \scrR_{g,\upeta}
(\theta) \mu^\upeta (\rmd \theta)}& \leq 
                                    16 \normrLigne{\Hess \gVar}[\infty] \limsup_{\upeta \to 0} \int_\Theta \expe{\normr{e_{\theta}(X_1)}[\theta]^2\{1-\chi_K(\theta)\}} \mu^\upeta (\rmd \theta) \\
  &\leq   16 \normrLigne{\Hess \gVar}[\infty]  \expe{\normr{e_{\thetas}(X_1)}[\thetas]^2\{1-\chi_K(\thetas)\}} \eqsp,
\end{align}
using \Cref{cor:ball_dirac}-\ref{cor:ball_dirac_ii},  that $\theta \mapsto \PE[\normr{e_{\theta}(X_1)}[\theta]^2]$ and $\chi_K$ are continuous and bounded by \Cref{ass:moment3} since $\PE[\normr{e_{\theta}(X_1)}[\theta]^2] = \trace(\Sigma(\theta))$ for any $\theta \in \Theta$ and $\Theta$ is compact. Taking $K \to \plusinfty$ completes the proof.   

%\pablo{brouillon verifi\'e}

\subsection{Proof of \Cref{thm:central_limit}}\label{app:central_limit}
%\paragraph*{Trying to weaken assumptions}
We introduce an auxiliary chain $(U_n)_{n \in \nset}$
as an intermediate step between 
$(\theta_n)_{n \in \nset}$ and $(\bar{U}_n)_{n \in \nset}$
for which we recall the definition below.
Define for any $\upeta >0, n \in \nset$,
\begin{equation}\label{eq:def_u}
U_n = \Exp_{\theta^\star}^{-1} (\theta_n) \eqand 
\bar{U}_n = \upeta^{-1/2} \Exp_{\thetas}^{-1} (\theta_n) =\upeta^{-1/2}U_n  \eqsp,
\end{equation}
where $(\theta_n)_{n \in \nset}$ is defined by \eqref{eq:scheme} with $\mss=\Theta ~\ie \proj_\mss = \Id$.
Note that $(U_n)_{n \in \nset}$ and $(\bar{U}_n)_{n \in \nset}$ are
Markov chains with state space $\planT_{\theta^{\star}} \Theta$, as
$\Exp_{\theta^\star}$ is a bijection. Conversely, since
$\Exp_{\theta^\star}^{-1}$ and $\upeta^{-1/2}\Exp_{\theta^\star}^{-1}$
are bijections from $\Theta$ to $\planT_{\theta^{\star}}\Theta$ under
\Cref{ass:had_or_complete}-\ref{ass:had_or_complete_i},
$(\theta_n)_{n\in \nset}$ is a deterministic function of $(U_n)_{n \in \nset}$ or $(\bar{U}_n)_{n \in \nset}$.
Therefore, the convergence of these three processes is expected to be the same. This is the content of the following result. Denote by $R_\upeta$ and $\bar{R}_\upeta$ the Markov kernels on $\planT_{\theta^\star} \Theta \times \mcb{\planT_{\theta^\star} \Theta} $, associated with $(U_n)_{n \in \nset}$ and $(\bar{U}_n)_{n \in \nset}$ respectively.

\begin{llemma}\label{lem:u_kernels}
Assume
\Cref{ass:had_or_complete}-\ref{ass:had_or_complete_i}-\ref{ass:had_or_complete_ii}, \Cref{ass:0mean_noise},
  \Cref{ass:MD:topo_prop_chain}, \Cref{ass:lyap},
  \Cref{ass:lyap_meanfield} and \Cref{ass:lyap_minorization}$(\msks)$
  for some compact set $\msks \subset \mss$. Let $\upeta \in (0,\bar{\upeta}]$ where $\bupeta = [2 C_2 L (1+\sigmaU)]^{-1}$. 
For any measurable and bounded function $\gPlan:\planT_{\theta^\star} \Theta \to \rset$ and any $u_0,\bar{u}_0 \in \planT_{\theta^\star} \Theta$, $R_\upeta$ and $\bar{R}_{\upeta}$ satisfy
\begin{equation}\label{eq:u_kernel}
R_\upeta \gPlan (u_0)= Q_\upeta \gVar \parenthese{\Exp_{\theta^\star} (u_0)} \quad\text{ and }\quad
 \bar{R}_{\upeta} \gPlan \parenthese{\bar{u}_0} = R_{\upeta} \gPlan_\upeta (\upeta^{1/2} \bar{u}_0) \eqsp,
\end{equation}
where $\gVar:\theta \mapsto \gPlan[\Exp^{-1}_{\theta^\star}(\theta)]$ and $\gPlan_\upeta: u \mapsto \gPlan(\upeta^{-1/2} u)$ are defined over $\Theta$ and $\planT_{\theta^\star} \Theta$ respectively, and $Q_{\upeta}$ is the Markov kernel associated with $(\theta_n)_{n\in \nset}$.
In addition, $R_\upeta$ and $\bar{R}_\upeta$ both admit a unique stationary distribution $\nu^{\upeta}$ and $\bar{\nu}^{\upeta}$ respectively, defined for any $\msa \in \mcb{\planT_{\theta^\star}\Theta}$ by
\begin{equation}\label{eq:u_invar}
\nu^{\upeta}(\msa)=\mu^{\upeta}\parenthese{\Exp_{\theta^\star}(\msa)} \quad \text{ and }\quad 
\bar{\nu}^{\upeta}(\msa)=\nu^{\upeta}(\upeta^{1/2}\msa) \eqsp.
\end{equation}
Finally, both $R_\upeta$ and $\bar{R}_\upeta$ are Harris-recurrent and geometrically ergodic, \ie~there exist $C,\bar{C}:\planT_\thetas \Theta \to \rset$ and $\rho,\bar{\rho}\in \rset_+^*$ such that for any $u,\bar{u}\in \planT_{\thetas} \Theta$,
\begin{equation}
\normrLigne{\updelta_u R_\upeta - \nu^\upeta}[\TV] \leq C(u) \rho^n \quad \text{  and  }\quad
\normrLigne{\updelta_{\bar{u}} \bar{R}_\upeta - \bar{\nu}^\upeta}[\TV] \leq \bar{C}(\bar{u}) \bar{\rho}^n \eqsp.
\end{equation}
\end{llemma}
\begin{proof}
Let $\gPlan :\planT_{\theta^\star} \Theta \to \rset$ be a measurable and bounded function and $u_0\in \planT_{\theta^\star} \Theta$. Consider $(U_n)_{n\in \nset}$ defined by \eqref{eq:def_u} with $\theta_0=\Exp_{\theta^\star} (u_0)$. Using \eqref{eq:def_u}, we have by definition
\begin{equation}\label{eq:calcul_kernel_u}
\expe{\gPlan(U_1)} = \expe{\gPlan \parenthese{\Exp^{-1}_{\theta^\star} (\theta_1)}} = Q_\upeta \parenthese{\gPlan\circ \Exp^{-1}_{\theta^\star}} \parenthese{\Exp_{\theta^\star} (u_0)} \eqsp.
\end{equation}
Moreover, let $\bar{u}_0 \in \planT_{\theta^\star} \Theta$ and consider $(\bar{U}_n)_{n\in \nset}$ defined by \eqref{eq:def_u} with $U_0=\upeta^{1/2}\bar{u}_0$. Using \eqref{eq:def_u}, we have by definition
\begin{equation}\label{eq:calcul_kernel_bar_u}
\expe{\gPlan(\bar{U}_1) } = \expe{\gPlan\parenthese{\upeta^{-1/2} U_1} } = R_\upeta \gPlan_{\upeta} \parenthese{\upeta^{1/2} \bar{u}_0} \eqsp,
\end{equation}
where $\gPlan_\upeta : u \mapsto \gPlan (\upeta^{-1/2} u)$ is defined over $\planT_\thetas \Theta$, therefore proving \eqref{eq:u_kernel}.

We show that $\nu^{\upeta}$ and $\bar{\nu}^{\upeta}$ are invariant for $R_\upeta$ and $\bar{R}_{\upeta}$ respectively. Indeed, for any $\msa\in \mcb{\planT_\thetas \Theta}$, we have by \eqref{eq:def_u}, \eqref{eq:u_kernel} and \eqref{eq:u_invar} %and that $\Exp_\thetas$ is a diffeomeorphism
\begin{align}
\nu^\upeta R_\upeta (\msa) &= \int_{\planT_\thetas \Theta} \rmd\nu^\upeta (u) R_\upeta\parenthese{u,\msa}  
= \int_{\Theta} \rmd \mu^{\upeta}(\theta) R_\upeta \parenthese{\Exp^{-1}_\thetas (\theta),\msa} \\
&= \int_{\Theta} \rmd \mu^\upeta(\theta) Q_\upeta\parenthese{\theta, \Exp_\thetas (\msa)} 
=\mu^\upeta\parenthese{\Exp_\thetas (\msa)}
=\nu^\upeta (A) \eqsp.
\end{align}
Therefore $\nu^\upeta$ is invariant for $R_\upeta$. Similarly, we show that $\bar{\nu}^\upeta$ is invariant for $\bar{R}_\upeta$. Using again \eqref{eq:def_u}, \eqref{eq:u_kernel} and \eqref{eq:u_invar}, for any $\msa\in \mcb{\planT_\thetas \Theta}$ we have,
\begin{align}
\bar{\nu}^{\upeta} \bar{R}_{\upeta}(\msa) 
= \int_{\planT_\thetas \Theta} \rmd \nu^\upeta (u) \bar{R}_\upeta \parenthese{\upeta^{-1/2}u,\msa} 
=\int_{\planT_\thetas \Theta} \rmd \nu^\upeta (u) R_\upeta \parenthese{u,\upeta^{1/2}\msa} = \bar{\nu}^\upeta (\msa) \eqsp.
\end{align}
Finally, since $(\theta_n)_{n \in \nset}$, $(U_n)_{n \in \nset}$ and $(\bar{U}_n)_{n \in \nset}$ are deterministic functions of each other and since \Cref{thm:recurrent_ergodic} proves that $(\theta_n)_{n \in \nset}$ is geometrically ergodic and Harris-recurrent, the same holds for $(U_n)_{n \in \nset}$ and $(\bar{U}_n)_{n \in \nset}$ and their invariant distributions are unique.
\end{proof}
For any smooth function with compact support $\gPlan:\planT_\thetas \Theta \to \rset$, $\bar{u}_0\in \planT_\thetas \Theta$ and $\upeta>0$ consider the 2-tensor $(\msc^2(\gPlan,\bar{u}_0,\upeta)_{ij})_{i,j \in \defEnsLigne{1,\ldots,d}}$ defined by, for any $i,j\in \defEnsLigne{1,\ldots,d}$,
\begin{equation}\label{eq:def_c_2}
\msc^2(\gPlan,\bar{u}_0,\upeta)_{ij}= \partial^2_{ij} \gPlan (\bar{u}_0) - \upeta^{1/2} \sum_{k=1}^d \Gamma_{ij}^k\parenthese{\Exp_\thetas(\upeta^{1/2}\bar{u}_0)} \partial_k \gPlan (\bar{u}_0) \eqsp,
\end{equation}
and, similarly consider the 3-tensor $(\msc^3(\gPlan,\bar{u}_0,\upeta)_{ijk})_{i,j,k \in \defEnsLigne{1,\ldots,d}}$ defined by, for any $i,j,k\in \defEnsLigne{1,\ldots,d}$,
\begin{equation}\label{eq:def_c_3}
\begin{aligned}
  \msc^3(\gPlan,\bar{u}_0,\upeta)_{ijk}&= \partial^3_{ijk} \gPlan (\bar{u}_0)\\
&\quad  - \upeta^{1/2} \sum_{l=1}^d \left[ \Gamma_{ij}^l \parenthese{\Exp_\thetas(\upeta^{1/2}\bar{u}_0)} \partial^2_{kl}\gPlan(\bar{u}_0) + \Gamma_{ki}^l \parenthese{\Exp_\thetas(\upeta^{1/2}\bar{u}_0)} \partial^2_{jl}\gPlan(\bar{u}_0) \right.\\
&\quad  \left.+\Gamma_{kj}^l \parenthese{\Exp_\thetas(\upeta^{1/2}\bar{u}_0)} \partial^2_{il}\gPlan(\bar{u}_0) \right] - \upeta \sum_{m=1}^d \partial_k \Gamma_{ij}^m\parenthese{\Exp_\thetas(\upeta^{1/2}\bar{u}_0)} \partial_m \gPlan (\bar{u}_0) \\
&\quad + \upeta \sum_{l,m=1}^d\parentheseDeux{\Gamma_{kj}^l\Gamma_{il}^m+\Gamma_{ki}^l\Gamma_{lj}^m} \parenthese{\Exp_\thetas(\upeta^{1/2} \bar{u}_0)} \partial_m \gPlan (\bar{u}_0) \eqsp,
\end{aligned}
\end{equation}
where $(\Gamma_{ij}^k)_{i,j,k\in \{1,\dots, d\}}$ are the Christoffel symbols 
of the Levi-Civita connection $\nabla$.
We derive the following Taylor formulas.
\begin{llemma}\label{lem:taylor_kernel_2}
Assume  \Cref{ass:had_or_complete}-\ref{ass:had_or_complete_i}-\ref{ass:had_or_complete_ii}, \Cref{ass:0mean_noise},
  \Cref{ass:MD:topo_prop_chain}, \Cref{ass:lyap},
  \Cref{ass:lyap_meanfield} and \Cref{ass:lyap_minorization}$(\msks)$
  for some compact set $\msks \subset \mss$. Suppose in addition that there exists $C_3 > 0$ such that for any $\theta \in \Theta$,  $\normr{h(\theta)}[\theta] \leq C_3 \distT(\thetas,\theta)$ and let $\bupeta = [2 C_2 L (1+\sigmaU)]^{-1} \wedge (4C_3)^{-1}$. 
  Consider normal coordinates $(u^i)_{ i \in \{1,\ldots,d\}}$ centered
  at $\theta^\star$ and define for any $i,j \in\{1,\ldots,d\}$,
  $h^i : \Theta \to \rset$, $\Sigma_{ij} : \Theta \to \rset$ by 
  $h^i =  \rmd u^i (h)$ and
  $\Sigma_{ij} = [\rmd u^i \otimes \rmd u^j]\{\Sigma\}$. For any smooth function with compact support
  $\gPlan:\planT_{\theta^\star} \Theta \to \rset$, any $\upeta\in
  (0,\bar{\upeta}]$ and $\bar{u}_0 \in \planT_{\theta^\star}
  \Theta$, we have
\begin{align}\label{eq:taylor_q_bar}
&\bar{R}_{\upeta}\gPlan(\bar{u}_0) = \gPlan(\bar{u}_0) + \upeta^{1/2} \sum_{i=1}^d \partial_i \gPlan(\bar{u}_0) h^i\parenthese{\Exp_{\theta^\star}(\upeta^{1/2}\bar{u}_0)} \\
&+ \frac{\upeta}{2} \sum_{i,j=1}^d \defEns{\partial_{ij}^2 \gPlan \parenthese{\bar{u}_0} - \upeta^{1/2}\sum_{k=1}^d \Gamma_{ij}^k(\Exp_\thetas(\upeta^{1/2}\bar{u}_0)) \partial_k \gPlan \parenthese{\bar{u}_0}} \parentheseDeux{\Sigma_{ij} +h^ih^j }\parenthese{\Exp_{\theta^\star}(\upeta^{1/2}\bar{u}_0)} \\
& \qquad \qquad \qquad + (\upeta/6) \bar{\scrR}_{\gPlan,\upeta}(\bar{u}_0) \eqsp,
\end{align}
where, setting $\theta_0= \Exp_{\theta^\star} (\upeta^{1/2}\bar{u}_0)$,
\begin{equation}\label{eq:control_remainder_q_bar}
\abs{\bar{\scrR}_{\gPlan,\upeta}(\bar{u}_0) } \leq
8 \upeta^{1/2} \1_{\msk_K}(\theta_0) \expe{\normrLigne{\msc^3(\gPlan,\upeta)}[\upgamma]
\1_{\msa_{\theta_0}^\complement}\normr{H_K}[\theta_0]^3} 
+ 16\normrLigne{\msc^2(\gPlan,\upeta)}
\expe{\normr{Y_K}[\theta_0]^2} \eqsp,
\end{equation}
using the definitions of $H_K,Y_K$, $\msa_{\theta_0}$, $\msk_K$ and $\upgamma$ in  
\Cref{lem:taylor_kernel_1}-\eqref{eq:def_hk}, 
\begin{equation}\label{eq:operator_norm_c}
\begin{aligned}
\normrLigne{\msc^2(\gPlan,\upeta)} &= \sup \defEnsLigne{\absLigne{\msc^2(\gPlan,\bar{u},\upeta) [v^{\otimes 2}]}\,:\, \bar{u} \in \planT_\thetas \Theta, v \in \rset^d, \normr{v}[2]=1} \\
\normrLigne{\msc^3(\gPlan,\upeta)}[\upgamma] &= \sup \defEnsLigne{\absLigne{\msc^3(\gPlan,\bar{u},\upeta) [v^{\otimes 3}]}\,:\, \bar{u} \in \upeta^{-1/2}\Exp^{-1}_{\thetas}(\upgamma([0,1])), v \in \rset^d, \normr{v}[2]=1} \eqsp,
\end{aligned}
\end{equation}
where $\msc^2(\gPlan,\bar{u},\upeta)$ and $\msc^3(\gPlan,\bar{u},\upeta)$ are defined in \eqref{eq:def_c_2} and \eqref{eq:def_c_3}.
%\end{enumerate}
\end{llemma}
\begin{proof}
  Using
  \Cref{ass:had_or_complete}-\ref{ass:had_or_complete_i} and
  \cite[Proposition 12.9]{lee:2019}, $(u^i)_{i \in \{1,\ldots,d\}}$ are global coordinates on
  the Hadamard manifold $\Theta$.  Let
  $\gPlan : \planT_{\thetas}\Theta \to \rset$ be a smooth function
  with compact support and $g : \Theta \to \rset$ defined for any
  $\theta \in\Theta$ by
  $g(\theta) = \gPlan(\Exp_{\thetas}^{-1}(\theta))$. Note that since
  $\normrLigne{\Exp_{\thetas}^{-1}(\theta)}[\thetas] =
  \distT(\thetas,\theta)$, for any $\theta \in \Theta$ by \cite[Corollary 6.12]{lee:2019}, $g$ is a
  smooth function with compact support as well. In addition, by
  definition of the normal coordinates,
  $\gPlan : u \mapsto \gVar(\Exp_{\theta^\star}(u))$ is the expression
  of $\gVar$ in this coordinate system. Using this fact and the
  definitions of the Riemannian gradient and Hessian \cite[Equation
  2.14, Example 4.22]{lee:2019}, we have, for any
  $\theta_0 \in \Theta$,
\begin{align}
\grad \gVar (\theta_0) &= \sum_{i=1}^d \partial_i \gPlan (u_0) \partial u_i \eqsp,\\
\label{eq:hess_coordinates_2} \Hess \gVar(\theta_0) &= \sum_{i,j=1}^d \defEns{\partial^2_{ij} \gPlan(u_0) - \sum_{k=1}^d \Gamma_{ij}^k(\Exp_\thetas(u_0)) \partial_k \gPlan(u_0)}\rmd u^i \otimes \rmd u^j \eqsp,
\end{align}
where $u_0 = \Exp^{-1}_{\theta^\star}(\theta_0)$ and $(\Gamma_{ij}^k)_{i,j,k \in \{1,\dots,d\}}$ are the Christoffel symbols. 
Combining these expressions with \Cref{lem:u_kernels}-\eqref{eq:u_kernel} and \Cref{lem:taylor_kernel_1}-\ref{lem:item:taylor_kernel_1:b}-\eqref{eq:taylor_q} gives
\begin{multline}\label{eq:taylor_q_tilde}
R_{\upeta} \gPlan (u)= \gPlan (u_0) + \upeta \sum_{i=1}^d \partial_i \gPlan(u_0) h^i(\Exp_{\theta^\star} (u_0)) \\
+ (\upeta^2/2) \sum_{i,j=1}^d \defEns{\partial^2_{ij} \gPlan(u_0) - \sum_{k=1}^d \Gamma_{ij}^k(\Exp_\thetas(u_0)) \partial_k \gPlan(u_0)} \parentheseDeux{\Sigma_{ij}\parenthese{\Exp_{\theta^\star}(u_0)} +h^ih^j\parenthese{\Exp_{\theta^\star}(u_0)} } \\
+ (\upeta^2/6) \tilde{\scrR}_{\gPlan,\upeta}(u_0) \eqsp,
\end{multline}
where $\tilde{\scrR}_{\gPlan,\upeta}(u_0)= \scrR_{\gVar,\upeta}(\theta_0)$ is bounded 
using \eqref{eq:control_remainder_q_b}, for $\theta_0=\Exp_{\theta^\star}(u_0)$ and 
$\gVar:\theta \mapsto \gPlan(\Exp^{-1}_{\theta^\star}(\theta))$.

Replacing $\gPlan$ with $\gPlan_\upeta:u \mapsto \gPlan(\upeta^{-1/2} u)$ defined over $\planT_{\theta^\star} \Theta$ and using that for any $i,j\in \defEnsLigne{1,\dots,d}$ and $u_0 \in \planT_{\theta^\star} \Theta$,
\begin{equation}\label{eq:derivative_gplan_eta}
\partial_i \gPlan_{\upeta}(u_0) = \upeta^{-1/2} \partial_i \gPlan (\upeta^{-1/2}u_0) 
\quad\text{ and }\quad
\partial^2_{ij} \gPlan_{\upeta}(u_0) = \upeta^{-1} \partial_{ij} \gPlan (\upeta^{-1/2}u_0) \eqsp,
\end{equation}
we have for any $u_0 \in \planT_\thetas \Theta$,
\begin{multline}\label{eq:taylor_q_tilde_gplan_eta}
R_\upeta \gPlan_\upeta(u_0) = \gPlan(\upeta^{-1/2} u_0) + \upeta^{1/2} \sum_{i=1}^d \partial_i \gPlan(\upeta^{-1/2}u_0) h^i\parenthese{\Exp_{\theta^\star}(u_0)} \\
+ (\upeta/2) \sum_{i,j=1}^d \defEns{\partial_{ij}^2 \gPlan \parenthese{\frac{u_0}{\upeta^{1/2}}} - \upeta^{1/2}\sum_{k=1}^d \Gamma_{ij}^k(\Exp_\thetas(u_0)) \partial_k \gPlan \parenthese{\frac{u_0}{\upeta^{1/2}}}} \parentheseDeux{\Sigma_{ij} +h^ih^j}\parenthese{\Exp_{\theta^\star}(u_0)} \\
+ (\upeta^2/6) \tilde{\scrR}_{\gPlan_\upeta,\upeta}(u_0) \eqsp.
\end{multline}
Expressing $\tilde{\scrR}_{\gPlan_\upeta,\upeta}(u_0)$ using partial derivatives shows explicitly 
the dependency on $\upeta$. Using \eqref{eq:derivative_gplan_eta} and the equivalent 
formula for the third order derivative, we have for any $K>0$,
\begin{equation}\label{eq:remainder_taylor}
\upeta^2 \abs{\tilde{\scrR}_{\gPlan_\upeta,\upeta}(u_0)} \leq 8 \upeta^3 \1_{\msk_K}(\theta_0)
\expe{\normrLigne{\nabla \Hess \gPlan_\upeta}[\upgamma,\infty] \1_{\msa_{\theta_0}^\complement}
\normr{H_K}[\theta_0]^3} 
+ 16 \upeta^2 \normrLigne{ \Hess \gPlan_\upeta}[\infty]\,
\expe{\normr{Y_K}[\theta_0]^2} \eqsp,
\end{equation}
where $\theta_0=\Exp_\thetas (u_0)$, $\upgamma:[0,1]\to \Theta$ is defined by $\upgamma(t)=\Exp_{\theta_0}(t\upeta H_{\theta_0}(X_1))$, $H_K,Y_K$ and $\msa_{\theta_0}$ are defined in \eqref{eq:def_hk}. Using \eqref{eq:hess_coordinates_2} and \Cref{lem:nabla_hess}, we have $\Hess \gPlan_\upeta(u) = \upeta^{-1} \msc^2(\gPlan,\upeta^{-1/2}u_0,\upeta)$ and $\nabla \Hess \gPlan_\upeta(u) = \upeta^{-3/2} \msc^3(\gPlan,\upeta^{-1/2}u_0,\upeta)$, where $\msc^2$ and $\msc^3$ are defined in \eqref{eq:def_c_2} and \eqref{eq:def_c_3} respectively. This gives
\begin{equation}\label{eq:operator_c}
\normr{\nabla \Hess \gPlan_\upeta}[\upgamma,\infty] = \upeta^{-3/2}\normrLigne{\msc^3(\gPlan,\upeta)}[\upgamma] \quad \text{  and  }\quad \normr{\Hess \gPlan_\upeta}[\infty] = \upeta^{-1}\normrLigne{\msc^2(\gPlan,\upeta)}  \eqsp,
\end{equation}
where $\normrLigne{\msc^2(\gPlan,\upeta)}$ and $\normrLigne{\msc^3(\gPlan,\upeta)}[\upgamma]$ are defined in \eqref{eq:operator_norm_c}.  
Setting $u_0=\upeta^{1/2}\bar{u}_0$ in \eqref{eq:taylor_q_tilde_gplan_eta}, we get
\begin{multline}\label{eq:taylor_almost}
R_{\upeta}\gPlan_{\upeta}(\upeta^{1/2}\bar{u}_0) = \gPlan(\bar{u}_0) + \upeta^{1/2} \sum_{i=1}^d \partial_i \gPlan(\bar{u}_0) h^i\parenthese{\Exp_{\theta^\star}(\upeta^{1/2}\bar{u}_0)} \\
+ (\upeta/2) \sum_{i,j=1}^d \defEns{\partial_{ij}^2 \gPlan \parenthese{\bar{u}_0} - \upeta^{1/2}\sum_{k=1}^d \Gamma_{ij}^k(\Exp_\thetas(\upeta^{1/2}\bar{u}_0)) \partial_k \gPlan \parenthese{\bar{u}_0}} \parentheseDeux{\Sigma_{ij} +h^ih^j }\parenthese{\Exp_{\theta^\star}(\upeta^{1/2}\bar{u}_0)} \\
+ \upeta^2 \tilde{\scrR}_{\gPlan_\upeta,\upeta}(\upeta^{1/2}\bar{u}_0) \eqsp.
\end{multline}
Therefore, letting $\bar{\scrR}_{\gPlan,\upeta}(\bar{u}_0)=
\upeta \tilde{\scrR}_{\gPlan_\upeta,\upeta}(\upeta^{1/2}\bar{u}_0)$, and combining \Cref{lem:u_kernels}-\eqref{eq:u_kernel}, \eqref{eq:remainder_taylor}, \eqref{eq:operator_c} and \eqref{eq:taylor_almost} gives the desired result.
%\end{enumerate}
\end{proof}

\begin{llemma}\label{lem:dl_h}% , 
% with respect to the orthonormal basis $(\bfe_i)_{i \in \{1,\dots,d\}}$ of $\planT_\thetas \Theta$
%\item\label{lem:item:dl_sigma} Assume \Cref{ass:moment3}. Then $\Sigma$ can be expressed in this chart as, for any $\upeta>0$, $\bar{u}\in \planT_\thetas \Theta$,
%\begin{equation}\label{eq:dl_sigma}
%\Sigma\parenthese{\Exp_\thetas(\upeta^{1/2}\bar{u})} = \sum_{i,j=1}^d \defEns{\Sigma_\star^{ij} + \scrR^{ij}_\Sigma \parenthese{\upeta^{1/2}\bar{u}}} \partial u_i \otimes \partial u_j \eqsp,
%\end{equation}
%where for any $i,j \in \{ 1, \dots, d \},\scrR^{ij}_{\Sigma}$ is bounded over $\planT_\thetas \Theta$ and $\lim_{\normrLigne{u}[\thetas]\to 0} \scrR^{ij}_{\Sigma}\parentheseLigne{u} =0$.
  Assume
  \Cref{ass:had_or_complete}-\ref{ass:had_or_complete_i}-\ref{ass:had_or_complete_ii}
  and \Cref{ass:driftmatrix}.  Consider normal coordinates
  $(u^i)_{i \in \{1,\dots,d\}}$ centered at $\thetas$ with respect to
  the orthonormal basis $(\bfe_i)_{i \in \{1,\dots,d\}}$ of
  $\planT_\thetas \Theta$. Then $h$ can be expressed in this chart as,
  for any $\upeta>0$, $\bar{u}\in \planT_\thetas \Theta$,
\begin{equation}\label{eq:dl_h}
h\parenthese{\Exp_\thetas(\upeta^{1/2}\bar{u})} = 
\sum_{i=1}^d \defEns{ \upeta^{1/2} \sum_{k=1}^d \bfA^i_k \bar{u}^k + 
\scrR^i_{h}\parenthese{\upeta^{1/2}\bar{u}}} \partial u_i \eqsp,
\end{equation}
where $\bfA$ is defined in \Cref{ass:driftmatrix}, $\bar{u}^k$ are the
components of $\bar{u}$ in $(\bfe_i)_{i\in\{1,\ldots,d\}}$ and for any
$i \in \{1,\dots,d\},\lim_{u \to 0}\{
|\scrR^i_{h}\parentheseLigne{u}|/\normrLigne{u}[\thetas]\} = 0$.
\end{llemma}
\begin{proof}
Since $\Theta$ is a Hadamard manifold, these normal coordinates are defined throughout $\Theta$. 
Thus, for any $\theta \in \Theta$, it is possible to write,
%\begin{equation} \label{eq:directionalderivative}
%  \psrLigne{\grad f(\theta)}{h(\theta)}[\theta] = \sum^d_{i=1}h^i(\theta) \frac{\partial f}{\partial \theta^i}(\theta)
%\end{equation}
%Write, as in \eqref{eq:directionalderivative}, for any $\theta \in \Theta$,
\begin{equation} \label{eq:proofh1}
  h(\theta) = \sum^d_{j=1}h^j(\theta)\partial u_j(\theta) \eqsp.
\end{equation}
Recall the definition of the metric coefficients in the coordinates 
$(u^i)_{i \in \{1,\dots,d\}}$ at $\theta\in \Theta$, for any $i,j\in \{1,\dots,d\}$,
\begin{equation} \label{eq:gij}
\frg_{ij}(\theta) = \psrLigne{\partial u_i(\theta)}{\partial u_j(\theta)}[\theta] \eqsp.
\end{equation}
Then, taking the scalar product of \eqref{eq:proofh1} with each $\partial u_i$, 
we have for any $i \in \{1,\dots,d\}$,
\begin{equation} \label{eq:proofh2}
\sum^d_{j=1}  \frg_{ij}( \theta)\,h^j( \theta) =  \psrLigne{h( \theta)}{\partial u_i(\theta)}[ \theta] \eqsp.
\end{equation}
From the Taylor expansion formula for vector fields given by \Cref{lem:taylor_vector} for the geodesic 
$\upgamma:[0,1]\to \Theta$ given by $\upgamma(0)=\thetas$ and 
$\dot{\upgamma}(0)=\Exp^{-1}_\thetas (\theta)$, it follows that,
\begin{equation}\label{eq:taylor_partial}
\partial u_i( \theta) = \parallelTransport_{01}^{\upgamma}\parentheseDeux{ 
\bfe_i + \nabla (\partial  u_i)_{\thetas}\parenthese{\Exp^{-1}_\thetas(\theta)}} 
+ \scrR_{\partial u_i}(\theta) \eqsp,
\end{equation}
where the remainder is given by
\begin{equation}
\scrR_{\partial  u_i}( \theta)=\int_{0}^1 (1-t) \parallelTransport_{t1}^{\upgamma} \nabla^2(\partial  u_i)_{\upgamma(t)}(\dot{\upgamma}(t),\dot{\upgamma}(t)) \rmd t \eqsp.
\end{equation}
Let $\normrLigne{\nabla^2 \partial u_i}[\infty,\upgamma] = 
\sup\defEnsLigne{\absLigne{\nabla^2(\partial  u_i)_{\upgamma(t)}(v,v)} 
\,:\, t \in [0,1], v \in \rmU_{\upgamma(t)} \Theta }$ which is finite as 
$\upgamma{[0,1]}$ is compact. Then using that for any $t\in [0,1]$, 
$\normrLigne{\dot{\upgamma}(t)}[\upgamma(t)]= \distT(\thetas,\theta)$ by 
\cite[Corollary 5.6]{lee:2019} and that geodesics are length-minimizing curves 
by \Cref{ass:had_or_complete}-\ref{ass:had_or_complete_i}; and that the parallel 
transport map is an isometry \cite[p.108]{lee:2019}, we have
\begin{equation}
\abs{\scrR_{\partial  u_i}( \theta)} \leq 
(1/2)\normrLigne{\nabla^2 \partial u_i}[\infty,\upgamma] \distT^2(\thetas,\theta) \eqsp.
\end{equation}
This proves that $\lim_{\theta \to \thetas}\abs{\scrR_{\partial  u_i}
( \theta)/ \distT(\thetas,\theta)} =0$. By the definition of normal coordinates 
centered at $\thetas$, for any $i,j \in \{1,\dots,d\}, \nabla_{\partial u_j} \partial u_i
= \sum_{k=1}^d\Gamma_{ji}^k \partial u_k$ and $(\Gamma_{ji}^k)_{i,j,k\in \{1,\dots,d\}}$ 
vanishes at $\thetas$ \cite[Proposition 5.24]{lee:2019} so \eqref{eq:taylor_partial} becomes
\begin{equation} \label{eq:proofh3}
\partial u_i(\theta) = \parallelTransport_{01}^{\upgamma}( \bfe_i ) + \scrR_{\partial  u_i}( \theta) \eqsp.
\end{equation}
Taking the scalar product of \eqref{eq:htaylor} and \eqref{eq:proofh3}, it follows that
\begin{equation} \label{eq:proofh4}
 \psrLigne{h(\theta)}{\partial u_i(\theta)}[\theta] =  \psrLigne{\bfA \Exp^{-1}_{\thetas}(\theta)}{\bfe_i}[\thetas] 
 + \tilde{\scrR}^i_{h}(\theta) \eqsp,
\end{equation}
since parallel transport preserves scalar products, where $\lim_{\theta \to \thetas}
\defEnsLigne{\absLigne{\tilde{\scrR}_h^i(\theta)}/\distT(\thetas,\theta)}=0$.
On the other hand, from \eqref{eq:gij} and \eqref{eq:proofh3}, since the 
$(\bfe_i)_{i \in \{1,\dots,d\}}$ are orthonormal,
\begin{equation} \label{eq:proofh6}
  \frg_{ij}(\theta) = \delta_{ij} + \scrR_\frg^{ij}(\theta) \eqsp,
\end{equation}
where $\delta_{ij} = 1$ if $i=j$ and $\delta_{ij}=0$ otherwise and 
$\lim_{\theta \to \thetas} \defEnsLigne{|\scrR_\frg^{ij}(\theta)| / \distT(\thetas,\theta)}=0$. 
Plugging \eqref{eq:proofh4} and \eqref{eq:proofh6} in \eqref{eq:proofh2}, we obtain
\begin{equation} \label{eq:proofh7}
h^i(\theta) = \sum_{j=1}^d\bfA^i_j u^j(\theta) + \scrR^i_h(\theta) \eqsp,
\end{equation}
where $\lim_{\theta \to \thetas} \abs{\scrR^i_h(\theta)} = 0$.
Finally, \eqref{eq:dl_h} is obtained from \eqref{eq:proofh1}-\eqref{eq:proofh7}, by setting 
$\theta = \Exp_\thetas (\upeta^{1/2}\bar{u})$, for $\bar{u } \in\planT_{\thetas}\Theta$, and noting that
\begin{align}
u^j(\Exp_\thetas (\upeta^{1/2}\bar{u})) &= \psrLigne{\Exp_\thetas^{-1}(\Exp_\thetas (\upeta^{1/2}\bar{u}))}{\bfe_j}[\thetas] = \upeta^{1/2}\bar{u}^j \eqsp,\\
 \distT  \parenthese{\Exp_\thetas (\upeta^{1/2}\bar{u}) ,\thetas } &= \upeta^{1/2}\normr{\bar{u}}[\thetas] \eqsp,
\end{align}
which follow from \cite[Corollary 5.6]{lee:2019} and the definition of the coordinates $(u^i)_{i \in \{1,\dots,d\}}$. 
\end{proof}
\begin{llemma} \label{lem:tight_tildemu}
  Assume
  \Cref{ass:had_or_complete}-\ref{ass:had_or_complete_i}-\ref{ass:had_or_complete_ii},
  \Cref{ass:0mean_noise}, \Cref{ass:MD:topo_prop_chain},
  \Cref{ass:moment3}, \Cref{ass:noise_unif_int_h}, \Cref{ass:lyap},
  \Cref{ass:lyap_meanfield}, \Cref{ass:driftmatrix} and
  \Cref{ass:lyap_comparison_distance} hold. Let $\bupeta=[2C_2 L (1+\sigmaU)]^{-1}\wedge (4C_3)^{-1}$. 
  Then the family of distributions $(\bnu^{\upeta})_{\upeta \in (0,\bupeta]}$, defined by 
  \eqref{eq:def_bnu}, is tight.
\end{llemma}
\begin{proof}
For any $\upeta \in \ocint{0,\bupeta}$, the conditions of \Cref{lem:u_kernels} hold, 
thus the Markov chain $(\bar{U}_n)_{n \in \nset}$ is ergodic 
and its invariant distribution $\bnu^{\upeta}$ is given by \eqref{eq:def_bnu}. 
For any $r\geq 0$, let $\bouletan_r=\defEnsLigne{u \in \planTsT \,:\, \normrLigne{u}[\thetas]\leq r}$
be the tangent closed ball at $\thetas$ of center $0$ and radius $r$.
Then, by \eqref{eq:u_invar} and \cite[Corollary 6.13]{lee:2019}, for any $r>0$ and $\upeta \in 
\ocintLigne{0,\bupeta}$, we have
\begin{equation}\label{eq:invariant_tight}
  \bnu^{\upeta}\parenthese{\planTsT \setminus \bouletan_r} 
= \nu^{\upeta}\parenthese{\planTsT \setminus \bouletan_{\upeta^{1/2}r}}
= \mu^{\upeta}\parenthese{\Theta \setminus \boulefermee{\thetas}{\upeta^{1/2}r}} \eqsp.
\end{equation}
However, by \Cref{ass:lyap_comparison_distance},
\begin{align}
  \mu^{\upeta}\parenthese{\Theta \setminus \boulefermee{\thetas}{\upeta^{1/2}r}}&\leq  \phi^{-1}(\upeta^{1/2}r)\int_{\Theta\setminus \{ \thetas \}}\phi(\distT(\thetas,\theta)) \rmd \mu^{\upeta}(\theta)\\
  \label{eq:tight_V}
&  \leq\phi^{-1}(\upeta^{1/2}r)\int_{\Theta\setminus \{ \thetas \}}V(\theta) \rmd \mu^{\upeta}(\theta) \eqsp.
\end{align}
Now, using \Cref{ass:lyap_comparison_distance} and \Cref{lem:majoration_mu_v} taking $\msks = \{ \thetas \}$,
we have,
\begin{equation}
\int_{\Theta\setminus \{ \thetas \}}V(\theta) \rmd \mu^{\upeta}(\theta) \leq 2 \upeta L \defEns{\sigmaZ + C_1 (1+\sigmaU)} / \lambda \eqsp.
\end{equation}
Combining this result and \eqref{eq:tight_V} in \eqref{eq:invariant_tight} implies that for any $r >0$,
\begin{align}
\bnu^{\upeta} \parenthese{\planTsT \setminus \bouletan_r} 
&\leq \left. 2 \upeta L \defEns{\sigmaZ + C_1 (1+\sigmaU)} / \parentheseDeuxLigne{\lambda \phi(\upeta^{1/2}r)}
\right. \\
&\leq \sup_{\upeta \leq \bupeta}\defEnsLigne{ \upeta / \phi(\upeta^{1/2}r)} (2L/\lambda)
\defEns{\sigmaZ + C_1 (1+\sigmaU)} \eqsp, \label{eq:determine_r}
\end{align}
where $\lim_{r\to +\infty}\defEnsLigne{\sup_{\upeta \leq \bupeta} \upeta / \phi(\upeta^{1/2}r)} = 0$
using \Cref{ass:lyap_comparison_distance}. Therefore, for any $\varepsilon >0$, there exists $r>0$  such that for any 
$\upeta \in \ocintLigne{0,\bupeta}$, 
$\bnu^{\upeta} \parentheseLigne{\planTsT \setminus \bouletan_r} \leq \varepsilon$.
This concludes the proof that $(\bnu^{\upeta})_{\upeta\in \ocintLigne{0,\bupeta}}$ is tight.
\end{proof}

\begin{proof}[Proof of \Cref{thm:central_limit}]
  Consider normal coordinates $(u^i)_{ i \in \{1,\ldots,d\}}$ centered
  at $\theta^\star$ with respect to
  the orthonormal basis $(\bfe_i)_{i \in \{1,\dots,d\}}$ of
  $\planT_\thetas \Theta$. Define for any $i,j \in\{1,\ldots,d\}$,
  $h^i : \Theta \to \rset$, $\Sigma_{ij} : \Theta \to \rset$ by 
  $h^i =  \rmd u^i (h)$ and
  $\Sigma_{ij} = [\rmd u^i \otimes \rmd u^j]\{\Sigma\}$.
Let $\gPlan:\planT_\thetas \Theta \to \rset$ be a smooth function with compact support. Applying \Cref{lem:taylor_kernel_2} to $\gPlan$ gives \eqref{eq:taylor_q_bar}. Using \Cref{ass:moment3}, $\Sigma$ is continuous, which implies that for any $\bar{u}_0 \in \planT_\thetas \Theta$,
\begin{equation}\label{eq:dl_sigma}
\Sigma\parenthese{\Exp_\thetas(\upeta^{1/2}\bar{u}_0)} = \sum_{i,j=1}^d \defEns{\Sigma_\star^{ij} + \scrR^{ij}_\Sigma \parenthese{\upeta^{1/2}\bar{u}_0}} \partial u_i \otimes \partial u_j \eqsp,
\end{equation}
where for any $i,j \in \{ 1, \dots, d \}$, $\Sigma_{\star}^{ij} = \Sigma_{ij} (\thetas)$, $\scrR^{ij}_{\Sigma}$ is continuous over $\planT_\thetas \Theta$ 
and $\scrR^{ij}_{\Sigma}\parentheseLigne{0} =0$.
Using \Cref{lem:dl_h}, replacing $\Sigma_{ij}$ and $h^i$ in \eqref{eq:taylor_q_bar} with \eqref{eq:dl_h} and \eqref{eq:dl_sigma} gives for any $\bar{u}_0 \in \planT_\thetas \Theta$,
\begin{equation}\label{eq:dl_R_final}
\begin{aligned}
  \bar{R}_{\upeta}\gPlan(\bar{u}_0) = \gPlan(\bar{u}_0) + \upeta \sum_{i=1}^d \partial_i \gPlan(\bar{u}_0) \sum_{k=1}^d \bfA_k^i \bar{u}_0^k + (\upeta/2) \sum_{i,j=1}^d \partial^2_{ij}\gPlan(\bar{u}_0)&\Sigma_\star^{ij} + \upeta \scrR_{\gPlan,\upeta,\Sigma,h}(\bar{u}_0) \\
 &\, +(\upeta/6)\bar{\scrR}_{\gPlan,\upeta}(\bar{u}_0) \eqsp,
\end{aligned}
\end{equation}
where $\bar{u}^k_0$ are the
components of $\bar{u}_0$ in $(\bfe_i)_{i\in\{1,\ldots,d\}}$, 
\begin{multline}\label{eq:remainder_final}
\scrR_{\gPlan,\upeta,\Sigma,h}(\bar{u}_0) = \upeta^{-1/2} \sum_{i=1}^d \scrR_{h}^i\parenthese{\upeta^{1/2}\bar{u}_0}\partial_i \gPlan(\bar{u}_0) \\
+ (1/2)\sum_{i,j=1}^d \defEns{\partial^2_{ij}\gPlan(\bar{u}_0) - \upeta^{1/2}\sum_{k=1}^d \Gamma_{ij}^k \parenthese{\Exp_\thetas(\upeta^{1/2}\bar{u}_0)} \partial_k \gPlan (\bar{u}_0)} \parentheseDeux{\scrR_\Sigma^{ij}\parenthese{\upeta^{1/2} \bar{u}_0}} \\
+ (1/2)\sum_{i,j=1}^d \defEns{\partial^2_{ij}\gPlan(\bar{u}_0) - \upeta^{1/2}\sum_{k=1}^d \Gamma_{ij}^k \parenthese{\Exp_\thetas(\upeta^{1/2}\bar{u}_0)} \partial_k \gPlan (\bar{u}_0)} \parentheseDeux{  h^ih^j\parenthese{\Exp_\thetas (\upeta^{1/2} \bar{u}_0)}} \\
- (\upeta^{1/2}/2) \sum_{i,j,k=1}^d \Gamma_{ij}^k \parenthese{\Exp_\thetas(\upeta^{1/2}\bar{u}_0)} \partial_k \gPlan(\bar{u}_0) \Sigma^{ij}_\star %\parenthese{\Exp_\thetas \parentheseLigne{\upeta^{1/2} \bar{u}_0}}
\eqsp.
\end{multline}
By \Cref{lem:tight_tildemu},
$(\bnu^{\upeta})_{\upeta \in \ocint{0,\bupeta}}$ is tight and
therefore relatively compact. Therefore, it is enough that for any limit point $\bnu^{\star}$,
$\bnu^\star = \loiGauss(0,\bfV)$ where  $\bfV\in \rset^{d \times d}$ is the solution of the 
Lyapunov equation $  \bfA \bfV + \bfV \bfA^\top = \Sigma(\thetas)$.  Let
$(\upeta_n)_{n \in \nsets}$ be a sequence with values in
$(0,\bar{\upeta}]$, such that $\lim_{n\to +\infty} \upeta_n = 0$, and $(\bnu^{\upeta_n})_{n \in \nsets}$ weakly converges to $\bnu^{\star}$.

First by
\eqref{eq:dl_R_final}, we have
\begin{align}\label{eq:dl_R_final_2}
  &\int_{\planTsT}  \bnu^{\upeta_n}(\rmd \bu_0) \int_{\planTsT} \bar{R}_{\upeta_n}(\bu_0,\rmd \bu_1)\gPlan(\bar{u}_1)\\
&\quad= \int_{\planTsT} \bnu^{\upeta_n}(\rmd \bu_0) \gPlan(\bar{u}_0) 
+ \upeta_n \int_{\planTsT} \bnu^{\upeta_n}(\rmd \bu_0) \sum_{i=1}^d \partial_i \gPlan(\bar{u}_0) \sum_{k=1}^d \bfA_k^i \bar{u}_0^k \\
&\qquad+ (\upeta_n/2) \int_{\planTsT} \bnu^{\upeta_n}(\rmd \bu_0)\sum_{i,j=1}^d \partial^2_{ij}\gPlan(\bar{u}_0)\Sigma_\star^{ij} 
+ \upeta_n \int_{\planTsT} \bnu^{\upeta_n}(\rmd \bu_0)\scrR_{\gPlan,\upeta_n,\Sigma,h}(\bar{u}_0)\\
&  \qquad + (\upeta_n/6)\int_{\planTsT} \bnu^{\upeta_n}(\rmd \bu_0) \bar{\scrR}_{\gPlan,\upeta_n}(\bar{u}_0)\eqsp.
\end{align}
Therefore using that $\bnu^{\upeta_n}$ is stationary with respect to $\bar{R}_{\upeta_n}$, we obtain that 
\begin{multline}
  \label{eq:proof_theo_tcl_1}
\limsup_{n \to \plusinfty} \abs{  \int_{\planTsT} \bnu^{\upeta_n}(\rmd \bu_0) 
\defEns{\sum_{i=1}^d \partial_i \gPlan(\bar{u}_0) \sum_{k=1}^d \bfA_k^i \bar{u}_0^k 
+\sum_{i,j=1}^d \partial^2_{ij}\gPlan(\bar{u}_0)\Sigma_\star^{ij} }}\\
\leq  \limsup_{n \to \plusinfty} \abs{ \int_{\planTsT} \bnu^{\upeta_n}(\rmd \bu_0)
\scrR_{\gPlan,\upeta_n,\Sigma,h}(\bar{u}_0)} + \abs{ \int_{\planTsT} 
\bnu^{\upeta_n}(\rmd \bu_0) \bar{\scrR}_{\gPlan,\upeta_n}(\bar{u}_0)} \eqsp. 
\end{multline}

Consider a sequence of independent random variables
$(\rmY_n)_{n \in \nset}$ such that for any $n \in \nset$,
the law of $\rmY_n$ is $\bar{\nu}^{\upeta_n}$.
By Slutsky's theorem, since $(\rmY_n)_{n \in \nset}$ 
converges in distribution and $\lim_{n \to +\infty} \upeta_n = 0$, we obtain that
$\upeta_n^{1/2} \rmY_n$ converges in distribution towards $0$. Moreover, using the
continuous mapping theorem, we have
\begin{equation}
  \label{eq:proof_theo_tcl_2}
  \limsup_{n \to \plusinfty} \abs{ \PE[\scrR_{\gPlan,\upeta_n,\Sigma,h}(\rmY_n)]} =0 \eqsp. 
\end{equation}
Similarly, we use \eqref{eq:control_remainder_q_bar} to obtain, for any $n \in \nset$ and $K>0$,
\begin{align}
\abs{\bar{\scrR}_{\gPlan,\upeta_n}(\rmY_n) } &\leq
8 \upeta_n^{1/2} \1_{\msk_K}(\uptheta_n) \expe{\normrLigne{\msc^3(\gPlan,\upeta_n)}[\upgamma]
\1_{\msa_{\uptheta_n}^\complement}\normr{H_K}[\uptheta_n]^3 \middle| \uptheta_n} \\
&\quad + 16\normrLigne{\msc^2(\gPlan,\upeta_n)}
\expe{\normr{Y_K}[\uptheta_n]^2 \middle| \uptheta_n} \eqsp,
\end{align}
where for any $n\in \nset$,
$\uptheta_n = \Exp_\thetas (\upeta^{1/2}_n \rmY_n)$ are independent
random variables and by \eqref{eq:u_invar}, the distribution of
$\uptheta_n$ is $\mu^{\upeta_n}$.  Thus we obtain for any $K \geq 0$,
using $\1_{\msk_K}(\uptheta_n)\normr{H_K}[\uptheta_n]$ is almost surely bounded by
$4[K^3+\sup_{\theta \in \msk_K}\normr{h(\theta)}[\theta]^3]$, Markov's
inequality and \Cref{ass:noise_unif_int_h},
\begin{align}
  \limsup_{n \to \plusinfty} \abs{ \PE[\bar{\scrR}_{\gPlan,{\upeta_n}}(\rmY_n)]} 
  &\leq  \limsup_{n\to \plusinfty} 16\normrLigne{\msc^2(\gPlan,\upeta_n)}\PE[\normr{\noise_{\uptheta_n}(X_1)}[\uptheta_n]^2\{1-\chi_K(\normrLigne{\noise_{\uptheta_n}(X_1)}[\uptheta_n])\}] \\
    \label{eq:proof_theo_tcl_3}
  &\leq  16\normrLigne{\msc^2(\gPlan,0)} K^{-\varepsilon}\{\tsigma_0^2 + \tsigma_1^2 \PE[V(\thetas)]\} \eqsp,
\end{align}
using that $(\uptheta_n)_{n \in \nset}$ converges in distribution to $\thetas$.
%where $(\uptheta_n)_{n \in\nset}$ is an independent sequence of random variables,
%such that $\uptheta_n$ has the distribution $\mu^{\upeta_n}$. 
For any smooth function with compact support $\gPlan : \planTsT \to \rset$, combining \eqref{eq:proof_theo_tcl_1}-\eqref{eq:proof_theo_tcl_2}-\eqref{eq:proof_theo_tcl_3}, 
taking $K \to \plusinfty$ and using the weak convergence of 
$(\bnu^{\upeta_n})_{n \in \nset}$ to $\bnu^{\star}$ when $n \to +\infty$ shows that  %concludes the proof.  
\begin{equation}\label{eq:generator_zero}
\int_{\planTsT} \bnu^{\star}(\rmd \bu_0) \defEns{ \sum_{i=1}^d \partial_i \gPlan(\bar{u}_0) 
\sum_{k=1}^d \bfA_k^i \bar{u}_0^k  
+\sum_{i,j=1}^d \partial^2_{ij}\gPlan(\bar{u}_0)\Sigma_\star^{ij} } = 0 \eqsp.
\end{equation}
Finally, by \cite[Theorem~2.2.1]{horn:johnson:1994}, there exists a
unique matrix $\bfV\in \rset^{d \times d}$ solution to
the Lyapunov equation $\bfA \bfV + \bfV \bfA^\top = \Sigma(\thetas)$.
By \cite[Theorem 10.1]{kent:1978},
$\loiGauss(0,\bfV)$ is the unique probability distribution on
$\planTsT$ satisfying \eqref{eq:generator_zero}.  This concludes the proof.
\end{proof}

\section{Proofs for  \Cref{sec:app_sgd}} \label{app:applications_sgd} %\ref{sec:app_sgd}
\subsection{Proof of \Cref{lem:str_conv}}
Recall that $f$ is $\lambda_f$-strongly geodesically convex, 
if and only if for any $\theta_1,\theta_2 \in \Theta$, %~[Page ]\cite{boumal:2020},
\begin{equation} \label{eq:strongconvdef}
 f(\theta_2) \geq f(\theta_1) +
  \psr{\Exp_{\theta_1}^{-1}(\theta_2)}{\grad f(\theta_1)}[\theta_1] +
  \lambda_f\distT^2(\theta_1,\theta_2) \eqsp.
\end{equation}
Put $\theta_1 = \thetas$ and $\theta_2 = \theta$. Since $\thetas$ is a stationary point of $f$, so $\grad f(\thetas) = 0$, it follows from \eqref{eq:strongconvdef} that
\begin{equation} \label{eq:strongconv_id2}
 f(\theta) - f(\thetas) \geq 
  \lambda_f\distT^2(\thetas,\theta) \eqsp,
\end{equation}
which is the second identity in \eqref{eq:strongconv}. To obtain the first identity, put $\theta_1 = \theta$ and $\theta_2 = \thetas$, in \eqref{eq:strongconvdef}, so
\begin{equation} \label{eq:proofstrongconv21}
  f(\thetas) - f(\theta)  \geq \psr{\Exp_{\theta}^{-1}(\thetas)}{\grad f(\theta)}[\theta] +
\lambda_f\distT^2(\thetas,\theta) \eqsp.
\end{equation}
Since $f(\thetas) \leq f(\theta)$, this implies
\begin{equation}
- \psr{\Exp_{\theta}^{-1}(\thetas)}{\grad f(\theta)}[\theta] \geq
\lambda_f\distT^2(\thetas,\theta) = \lambda_f \normr{\Exp_\theta^{-1}(\thetas)}[\theta]^2 \eqsp.
\end{equation}
Or, after using the Cauchy-Schwarz inequality,
\begin{equation} \label{eq:proofstrongconv22}
 \normr{\grad f(\theta)}[\theta] \geq \lambda_f \normr{\Exp_\theta^{-1}(\thetas)}[\theta] \eqsp.
\end{equation}
Finally, using once more the Cauchy-Schwarz inequality, and \eqref{eq:proofstrongconv21} and \eqref{eq:proofstrongconv22},
\begin{equation}
  f(\theta) - f(\thetas)  \leq  - \psr{\Exp_{\theta}^{-1}(\thetas)}{\grad f(\theta)}[\theta] 
 \leq (1/\lambda_f)  \normr{\grad f(\theta)}[\theta]^2 \eqsp,
\end{equation}
which is equivalent to the first identity in \eqref{eq:strongconv}. 
% \paragraph{Proof of Proposition \ref{prop:sgd_bound_error_moment}\,:} by \Cref{ass:lyap_strongconv} and~\cite[Lemma 1]{durmus:jimenez:moulines:said:wai:2020}, it follows from \eqref{eq:sgdschemecss}
% \begin{equation} \label{eq:proof_prop10_1}
% f(\theta_1) - f(\theta_0) \leq -\upeta \normr{\grad f(\theta_0)}[\theta_0]^2 + \upeta^2L_f \left(\normr{\grad f(\theta_0)}[\theta_0]^2 + \normr{e_{\theta_0}(X_{k+1})}[\theta_0] \right)
% \end{equation}
% By taking conditional expectations and using \Cref{ass:0mean_noise},
% $$
% \left(\upeta- \upeta^2L_f(1+\sigmaU)\right) \normr{\grad f(\theta_0)}[\theta_0]^2
%    \leq  f(\theta_0) - Q_\upeta f(\theta_0) + \upeta^2L_f\sigmaZ
% $$
% and, if $\upeta \leq \bar{\upeta}$,
% $$
% (\upeta/2) \normr{\grad f(\theta_0)}[\theta_0]^2
%    \leq f(\theta_0) - Q_\upeta f(\theta_0) + \upeta^2L_f\sigmaZ
% $$
% Then, since $\mu^\upeta$ is the stationary distribution of $Q_\upeta$,
% $$
% (\upeta/2) \int_\Theta \normr{\grad f(\theta)}[\theta]^2 \mu^\upeta(\rmd \theta)
%   \leq   \upeta^2L_f\sigmaZ
% $$
% The required inequality \ref{eq:sgd_bound_error_moment} follows after dividing by $\upeta$.

\subsection{Proof of \Cref{lem:str_huberized}}
Without loss of generality, we assume that $f(\thetas) = 0$.
First, we
show that for any $\theta \in\Theta$,
\begin{equation}
  \label{eq:inequbound_huber_proof_f}
f(\theta) \leq M_f \distT^2(\thetas,\theta) \eqsp.
\end{equation}
Let $\theta \in \Theta$ and $\upgamma : \ccint{0,1} \to \Theta$ the
unique geodesic such that $\upgamma(0)=\thetas$ and
$\upgamma(1)= \theta$. Then since $f$ is continuously differentiable
using \cite[Proposition 4.15-(ii) and Theorem 4.24-(iii)]{lee:2019},
we get that
$f(\theta) = \int_{0}^1 \psr{\grad
  f(\upgamma(t))}{\dot{\upgamma}(t)}[\upgamma(t)] \rmd t$. Therefore,
using the Cauchy-Schwarz inequality and for any $t \in \ccint{0,1}$,
$\normr{\dot{\upgamma}(t)}[\upgamma(t)] = \distT(\thetas,\theta)$
we obtain that $\abs{f(\theta)} \leq \distT(\thetas,\theta) \normr{\grad
  f(\upgamma(t))}[\upgamma(t)]$ which shows that \eqref{eq:inequbound_huber_proof_f} holds by assumption. 

We now proceed with the proof of the main statement. Since $f$ is twice continuously
differentiable, $\tilde{f}$ has this same property. In addition,
for any $\theta \in \Theta$,
\begin{equation}
  \label{eq:grad_tilde_f_str_huberized}
  \grad \tilde{f}(\theta) = \grad f(\theta) / [2(f(\theta)+1)^{1/2}] \eqsp. 
\end{equation}

Therefore, using the assumption that for any $\theta \in \Theta$,
$\normr{\grad f(\theta)}[\theta]^2 \leq M_f \distT^2(\thetas,\theta)$ and the second inequality of \Cref{lem:str_conv},
we get that
\begin{multline}
  \label{eq:5}
  \normrLigne{\grad \tilde{f}(\theta)}[\theta] =   \normr{\grad f(\theta)}[\theta]/ [2(f(\theta)+1)^{1/2}] \leq M_f^{1/2} \distT(\thetas,\theta)/[2(\lambda_f \distT^2(\thetas,\theta) +1)^{1/2}]\\ \leq C_f^{1/2} [1\wedge \distT(\thetas,\theta)] \eqsp,
\end{multline}
with $C_f^{1/2}  \leftarrow (M_f^{1/2}/2)[1\wedge \lambda_f^{-1/2}]$.

It remains to show that for any $\theta \in \Theta$,
$ -\psrLigne{\Exp^{-1}_\theta(\thetas)}{\grad \tilde{f}(\theta)}[\theta] \geq \tlambda_f
V_1(\theta) $, where $V_1$ is defined by \eqref{eq:lyapunovV} with
$\delta =1$ and $\tlambda_f \leftarrow \lambda_f^{1/2}/2$. Using
\eqref{eq:grad_tilde_f_str_huberized} again, \Cref{ass:strongconv} and \eqref{eq:inequbound_huber_proof_f}, we obtain that for any
$\theta \in \Theta$,
\begin{multline}
  \label{eq:7}
  -\psr{\Exp^{-1}_{\theta}(\thetas)}{\grad \tilde{f}(\theta)}[\theta] =   -\psr{\Exp^{-1}_{\theta}(\thetas)}{\grad f(\theta)}[\theta]/[2(f(\theta)+1)^{1/2}] \\ \geq \lambda_f \distT^2(\thetas,\theta)/[2(f(\theta)+1)^{1/2}] 
  \geq \lambda_f \distT^2(\thetas,\theta)/[2(M_f\distT^2(\thetas,\theta)+1)^{1/2}] \eqsp.
\end{multline}
Using that for any $\theta \in \Theta$, $V_1(\theta)= \{\distT^2(\thetas,\theta)+1\}^{1/2} - 1 \leq \distT(\thetas,\theta)$, we get that
\begin{equation}
  \label{eq:8}
  -\psr{\Exp^{-1}_{\theta}(\thetas)}{\grad \tilde{f}(\theta)}[\theta] \geq \lambda_f V_1(\theta) \distT(\thetas,\theta)/[2(M_f\distT^2(\thetas,\theta)+1)^{1/2}] \geq \lambda_f V_1(\theta)/(2M_f^{1/2}) \eqsp. 
\end{equation}

\subsection{Proof of
  \Cref{prop_sgdconstrained}}
\label{app:applications_sgd_constrained} %\ref{sec:app_sgdconstrained}
The proof consists in an application of
\Cref{theo:drift_lyap}-\ref{theo:drift_lyap1}. First, by
\Cref{prop:huber}, $V_1$ defined by \eqref{eq:lyapunovV} with
$\delta=1$, satisfies \Cref{ass:lyap}. In addition, by \cite[Lemma 16]{durmus:jimenez:moulines:said:wai:2020}, $V_1$ is continuously differentiable with gradient given for any $\theta \in \Theta$ by 
\begin{equation}
  \label{eq:grad_V_1}
  \grad V_1(\theta) = -\Exp^{-1}_{\theta}(\thetas)/ \defEnsLigne{ 1+\distT^2(\thetas,\theta)}^{1/2} \eqsp. 
\end{equation}
Therefore, for any $\theta \in \Theta$, by \Cref{ass:quasiconv} we get 
\begin{multline}
  \label{eq:9}
  \psr{\grad V_1(\theta)}{\grad f(\theta)}[\theta] = - \psr{\Exp^{-1}_{\theta}(\thetas)}{\grad f(\theta)}[\theta]/ \defEnsLigne{ 1+\distT^2(\thetas,\theta)}^{1/2}  \\ \geq \tlambda_f V_1(\theta)/ \defEnsLigne{ 1+\distT^2(\thetas,\theta)}^{1/2} \eqsp. 
\end{multline}
In addition, 
$t^2 \wedge 1 -ab\{(t^2+1)^{1/2}-1\}/(1+t^2)^{1/2} \leq 0$ for any
$t \geq 0$, $b > 0$ and $a = 4b^{-1}$ using that $(t^2+1)^{1/2}-1 \geq t^{2}/[2(1+t^2)^{1/2}]$. As a result, using
\Cref{ass:quasiconv}  for any
$t \geq 0$, $b > 0$ and $a = 4b^{-1}$, it follows that
\Cref{ass:lyap_meanfield} is satisfied with
$C_1 \leftarrow 0, C_2 \leftarrow 4 C_f/ \tlambda_f$ for
$h = -\grad f$ and $V \leftarrow V_1$. Therefore, we obtain using
\Cref{theo:drift_lyap}-\ref{theo:drift_lyap1} that for any
$\upeta \in \ocint{0,\upeta}$,
  \begin{equation}\label{eq:theo:drift_lyap0_h_app_sgd}
  n^{-1} \sum_{k=0}^{n-1} \expe{\psr{\grad V_1 (\theta_k)}{\grad f(\theta_k)}[\theta_k]} \leq 
  2V_1(\theta_0)/(n \upeta) +  2 \upeta (1+\kappa )\sigmaZ   \eqsp,
  \end{equation}
where
$\bupeta = [(8  C_f/ \tlambda_f) (1+\kappa ) (1+\sigmaU)]^{-1} $.
Using \eqref{eq:9}, we have
  \begin{equation}\label{eq:theo:drift_lyap0_h_app_sgd_D}
  (\tlambda_f/ n) \sum_{k=0}^{n-1} \expe{V_1 (\theta_k)/ \defEnsLigne{ 1+\distT^2(\thetas,\theta_k)}^{1/2}} \leq 
  2V_1(\theta_0)/(n \upeta) +  2 \upeta (1+\kappa)\sigmaZ   \eqsp,
  \end{equation}
  which concludes the proof since $(t^2+1)^{1/2}-1 \geq t^2/[2(1+t^2)^{1/2}]$ for any $t\geq 0$ implying 
  $V_1 (\theta)/ \defEnsLigne{ 1+\distT^2(\thetas,\theta)}^{1/2} \geq
  D^2_{\Theta}(\thetas,\theta)/2$ for any $\theta \in\Theta$.

\subsection{Proof of  \Cref{prop:br1}} \label{app:applications_br} %\ref{sec:barycentre}
% For the variance function,
% \begin{equation} \label{eq:variancef}
%    f(\theta) =  \frac{1}{2}\,\sum^{M_\pi}_{i=1} \distT^2(\theta,\btheta_i)
% \end{equation}
% the stochastic gradient scheme was given in~\cite{arnaudon:2012} %reference
%  \begin{equation} \label{eq:sgd_barycentre}
%     \theta_{n+1} = \Exp_{\theta_n}\!\left( \upeta H_{\theta_n}(X_{n+1})\right)\hspace{0.25cm};\hspace{0.25cm}
%    H_{\theta}(x) = \Exp^{-1}_\theta(x)
%  \end{equation}
% where $x \in \msx$, which is the set 
Define $\msx = \lbrace \btheta_i\,: \, i\in \{1,\ldots,M_\pi\}\rbrace$ and recall that
$\rmD= \sup \{ \distT(\theta_0,\btheta)\,:\, \btheta \in \msx \}$. Set $\mss = \cball{\theta_0}{\rm D}$. 
Note that the closed ball $\mss$, is compact by \cite[Theorem 1.7.1]{jost:2005},
geodesically convex, and $\msx \subset \mss$, as well as
$\theta_0 \in \mss$. We consider in this section, for any $\theta \in \Theta$ and $x \in \msx$,
$H_\theta (x) = \Exp^{-1}_{\theta}(x)$.

First note that $\theta_n \in \mss$, for all $n \in \nset$ by a
straightforward induction using that $\mss$ is geodesically convex
and $\theta_0 \in\mss$. Indeed, $\theta_0 \in \mss$, and, if $\theta_n \in \mss$, then
$\theta_{n+1}$ lies on the geodesic segment connecting $\theta_n$ and
$X_{n+1}$, two points which belong to $\mss$, and therefore
$\theta_{n+1} \in \mss$. This means that the SGD scheme used here is
equivalent to
\begin{equation}
\theta_{n+1} = \proj_{\mss}\left(\Exp_{\theta_n}\!\left( \upeta H_{\theta_n}(X_{n+1})\right)\right) \eqsp.
\end{equation}

Define $\msh$ and $V_2$ as in Proposition \ref{prop:sqdistance}. 
It is possible to show that $\msh = \mss$. Indeed, for $\theta \in \mss$, and $x \in \msx$, 
since $x \in\mss$, and $\mss$ is convex, the geodesic segment connecting $\theta$ to $x$ 
is entirely contained in $\mss$. However, by definition, this geodesic segment is the set 
of points $\Exp_\theta(tH_\theta(x))$, where $t \in [0,1]$. Now, since 
$\upeta \leq \bar{\upeta} \leq 1$, Proposition \ref{prop:sqdistance} implies that 
$V_2$ verifies \Cref{ass:lyap}-\ref{ass:lyap_contractive}-\ref{ass:lyap_grad_lips} 
where $L\leftarrow C L_\pi$, $L_\pi= (\rmD +1 )(1+ \kappa \coth(\kappa \rmD))$ 
and $C$ is a universal constant. % \alain{preciser constant Lipschitz $V_2$}

The objective function $f$
satisfies \Cref{ass:strongconv} with $\lambda_f = 1/2$ (that is, $f$
is $1/2$-strongly convex), since by \cite[Theorem 5.6.1]{jost:2005}
$f_i(\theta) = \distT^2(\theta,\btheta_i)/2$ is $1$-strongly
geodesically convex for any $i\in\{1,\ldots,M_{\pi}\}$. Thus, by \eqref{eq:strongconvdef} for all $\theta \in \mss$
\begin{equation} \label{eq:V2_inequality1}
\psr{\Exp^{-1}_\theta(\thetas)}{\grad f(\theta)}[\theta] \leq -(1/2) \distT^2(\thetas,\theta) \eqsp.
\end{equation}
 Now, for any $\theta \in \mss$, $v \in \planT_\theta \Theta$,  using \cite[Theorem 5.6.1]{jost:2005}, we have,
\begin{align}
\normr{\Hess f_\theta (v,v)}[\theta] 
&\leq M_\pi^{-1} \sum_{i=1}^{M_\pi} \normr{(\Hess f_i)_\theta (v,v)}[\theta] \\
&\leq M_\pi^{-1} \sum_{i=1}^{M_\pi} \kappa \distT(\theta,\btheta_i) 
\coth(\kappa \distT(\theta,\btheta_i)) \normr{v}[\theta]^2
\leq \tilde{L}_{\pi} \normr{v}[\theta]^2 \eqsp,
\end{align}
where $\tilde{L}_{\pi} = 2\rmD\kappa \coth(2\kappa \rmD)$, since $t \mapsto t \coth(t)$ 
%$\tilde{L}_{\pi} = \rmD(1+2\kappa \coth(2\kappa \rmD))$
is non-decreasing over $\rset_+$. Therefore, by \cite[Lemma 10]{durmus:jimenez:moulines:said:wai:2020},
$\grad f$ is geodesically $\tilde{L}_{\pi}$-Lipschitz continuous on $\mss$.% where $\tilde{L}_{\pi} = (1+\kappa{\rmD})\coth(\kappa{\rmD})$. 
 In particular, for any $\theta \in \mss$,
\begin{equation} \label{eq:V2_inequality2}
  \normr{\grad f(\theta)}[\theta] \leq \tilde{L}_{\pi} \distT(\thetas,\theta)\eqsp.
\end{equation}
By \eqref{eq:V2_inequality1} and \eqref{eq:V2_inequality2}, it is straightforward that $V = V_2$ and $h = -\grad f$ satisfy \Cref{ass:lyap_meanfield}, with $C_1 = 0$ and $C_2 = 2\tilde{L}^2_\pi \leq 2^5 L^2_{\pi}$. In addition, by Proposition \ref{prop:sqdistance}, \eqref{eq:V2_inequality1} implies $V_2$ verifies \Cref{ass:lyap_minorization}-$(\emptyset)$, with $\lambda = 1/2$. 

Finally, \Cref{ass:0mean_noise} holds with $\sigmaZ = {\rm D}^2$ and
$\sigmaU = 0$ since for any $\theta \in \mss$ and
$x \in \msx$,
\begin{equation}
\normr{H_\theta(x)}[\theta] = \normr{\Exp^{-1}_{\theta}(x)}[\theta] \leq 2{\rmD} \eqsp.
\end{equation}

Therefore, we can apply \Cref{theo:drift_lyap}-\ref{theo:drift_lyap2} which implies that for any $\upeta \leq \bar{\upeta}$,
\begin{equation} \label{eq:barycentre_00}
\PE[V_2(\theta_n)] \leq \defEns{1 - \upeta/4}^{n}V_2(\theta_0) + 4\upeta\,L_\pi {\rmD}^2 \eqsp.
\end{equation}
To conclude, it only remains to note that $V_2(\theta_n) = \distT^2(\thetas,\theta_n)$ and $V_2(\theta_0) = \distT^2(\thetas,\theta_0)$, since $(\theta_n)_{n \in \nset}$ and $\thetas$ belong to $\msh = \mss$.

\subsection{Proof of \Cref{theo:drift_bary}}

We consider in this section the recursion
\begin{align}
  \label{eq:karsher_mean_rec_non_compact}
		\theta_{n+1} &= \Exp_{\theta_n}\parentheseDeux{\upeta H_{\theta_{n}}(X_{n+1})} \\
      H_{\theta_{n}}(X_{n+1}) &=  \left.    \Exp^{-1}_{\theta_n}\parenthese{X_{n+1}^{(1)}}\middle/
      \parenthese{2\defEnsLigne{\distT^2(\theta_n,X_{n+1}^{(2)})/2 + 1}^{1/2}} \right.          \eqsp ,
\end{align}
where $X_{n+1} = (X_{n+1}^{(1)}, X_{n+1}^{(2)}$ and $(X_n^{(1)},X_n^{(2)})_{n\in \nsets}$ is an \iid~sequence of pairs of independent random variables with distribution $\pi$. Denote by $\MKer$ the Markov kernel corresponding to \eqref{eq:karsher_mean_rec_non_compact}. 

We give first some additional intuition and motivation behind the scheme \eqref{eq:karsher_mean_rec_non_compact}. 
It can be interpreted as a stochastic optimization method to minimize
                \begin{equation}
                  \label{eq:def_tf_karsher_noncompact}
                  \tf_{\pi}=(f_{\pi}+1)^{1/2} \eqsp,
                \end{equation}
                in place of $f_{\pi}$. First note that $f_{\pi}$ and
                $\tf_{\pi}$ have the same minimizer, but compared
                to $f_{\pi}$ it may be shown that $ \grad \tf_{\pi}$, given for any $\theta \in \Theta$ by
\begin{equation}
	\label{eq:tamed_grad}
	\grad \tf_{\pi} (\theta) = (1/2)\grad f_{\pi}(\theta)  (f_{\pi}(\theta) +1)^{-1/2} \eqsp,
	%=(1/2) \expe[  \eqsp,
      \end{equation}
      is geodesically Lipschitz. However, note that \eqref{eq:karsher_mean_rec_non_compact} is not an 
      unbiased stochastic optimization scheme for the function $\tf_{\pi}$ since
      \begin{equation}
\expe{      H_{\theta_{n}}(X_{n+1})} = (1/2) \{  \grad f_{\pi}(\theta_n)\} \expe{\defEnsLigne{\distT^2(\theta_n,X_{n+1}^{(2)})/2 + 1}^{-1/2} } \eqsp.       
\end{equation}
The proof of \Cref{theo:drift_bary} then consists in adapting the
proof of \Cref{theo:drift_lyap} to deal with this additional
difficulty taking for the Lyapunov function $V$, $V_1$ defined by
\eqref{eq:lyapunovV} with $\delta=1$. A general theory could be
derived but we believe that this is out the scope of the present
document and leave it for future work. We start by preliminary technical results which are needed to establish \Cref{theo:drift_bary}.

      \begin{llemma}
	      \label{lem:psr_convex}
        	Assume
	  \Cref{ass:hadamard_curvature} and \Cref{ass:second_moment_pi}. Let $\thetaspi$ be the Riemannian
	  barycenter of the probability measure $\pi$, \ie~$\thetaspi = \argmin_{\Theta} f_{\pi}$ where $f_\pi$ is defined by \eqref{def:barycenter_noncompact}. Then, for any $\theta \in\Theta$,
          \begin{equation}
            \label{eq:10}
            -\int_{\Theta} \psr{\Exp^{-1}_{\theta}(\thetaspi)}{\Exp^{-1}_{\theta}(\nu)}[\theta] \pi(\rmd \nu) 
	    \leq - \distT^2(\theta,\thetaspi)/2 \eqsp. 
          \end{equation}
        \end{llemma}
\begin{proof}
	Using \Cref{ass:hadamard_curvature} and \cite[Theorem 5.6.1]{jost:2005},
	%and \cite[Lemma 10]{durmus:jimenez:moulines:said:wai:2020}, 
	we have that for any $\nu \in \Theta$, the operator norm of the Riemannian Hessian of 
	$\theta \mapsto \distT^2(\theta,\nu)/2$ is lower bounded by 1.
	Therefore, by \cite[Theorem 11.19]{boumal:2020intromanifolds}, $\theta \mapsto \distT^2(\theta,\nu)/2$ is $1/2$-strongly convex. 
	Applying this to $\theta$ and $\thetaspi \in \Theta$, we have for any $\nu \in \Theta$,
	\begin{equation}
		\distT^2(\thetaspi,\nu)/2 - \distT^2(\theta,\nu)/2 \geq
	-\psr{\Exp^{-1}_\theta (\thetaspi)}{\Exp^{-1}_\theta(\nu)}[\theta]
		+ \distT^2(\theta,\thetaspi)/2 \eqsp.
	\end{equation}
	Using \Cref{ass:second_moment_pi}, we can integrate this inequality \wrt~$\pi$, bringing
	\begin{equation}
		f_\pi(\thetaspi) - f_\pi(\theta) \geq - \int_\Theta 
		\psr{\Exp^{-1}_\theta (\thetaspi)}{\Exp^{-1}_\theta(\nu)}[\theta] \pi(\rmd \nu)
		+ \distT^2(\theta,\thetaspi)/2 \eqsp.
	\end{equation}
	Since by definition of $\thetaspi$, $0 \geq f_\pi(\thetaspi) - f_\pi(\theta)$, 
this completes the proof.
\end{proof}

        \begin{llemma}
		\label{lem:jensen_dist}
          	Assume
	  \Cref{ass:hadamard_curvature} and \Cref{ass:second_moment_pi}. Let $\thetaspi$ be the Riemannian
	  barycenter of the probability measure $\pi$, \ie~$\thetaspi = \argmin_{\Theta} f_{\pi}$ where $f_\pi$ is defined by \eqref{def:barycenter_noncompact}. Then, for any $\theta \in \Theta$,
          \begin{equation}
            \label{eq:jensen_lemme}
            \int_{\Theta} \{\distT^2(\theta,\nu)/2+1 \}^{-1/2} \pi (\rmd \nu) \geq \{\distT^2(\theta,\thetaspi) + 2 f_{\pi}(\thetaspi) +1 \}^{-1/2} \eqsp. 
          \end{equation}
        \end{llemma}

\begin{proof} 
	Let $\theta \in \Theta$. Using Jensen's inequality with the convex function $t \mapsto (t+1)^{-1/2}$ on $\rset_+^*$, we have
	\begin{equation}
		\label{eq:jensen_w}
		\int_\Theta \{\distT^2(\theta,\nu)/2+1 \}^{-1/2} \pi (\rmd \nu) 
		\geq \defEns{ f_\pi(\theta) + 1}^{-1/2} \eqsp.
	\end{equation}
	However, using the triangle and Hölder's inequalities, we have for any $\theta$ and $\nu \in \Theta$,
	$\distT^2(\theta,\nu)/2 \leq \distT^2(\theta,\thetaspi) + \distT^2(\thetaspi,\nu)$. 
	Taking the integral with respect to $\pi$, by \Cref{ass:second_moment_pi} we get 
	$f_\pi (\theta) \leq \distT^2(\theta,\thetaspi) + 2 f_\pi(\thetaspi)$.
	Lastly, combining this result with \eqref{eq:jensen_w} and using that the  function $t \mapsto (t+1)^{-1/2}$  is non-increasing on $\rset_+^*$ completes the proof.
\end{proof}

\begin{llemma}
	\label{lem:descente_barycentre}
	Assume
	  \Cref{ass:hadamard_curvature} and \Cref{ass:second_moment_pi}. Let $\thetaspi$ be the Riemannian
	  barycenter of the probability measure $\pi$, \ie~$\thetaspi = \argmin_{\Theta} f_{\pi}$ where $f_\pi$ is defined by \eqref{def:barycenter_noncompact}. 
	Then, for any $\theta_0 \in \Theta$,
\begin{equation}
	\label{eq:descent_bary}
	\MKer \VHuber (\theta_0) \leq \VHuber (\theta_0) - [\upeta/(4C_\pi^{1/2})] D^2_\Theta(\theta_0,\thetaspi)
	+ 2 \upeta^2 (1+\kappa) \{1+f_\pi(\thetaspi)\}\parenthese{f_{\pi}(\thetaspi) + 2} \eqsp,
\end{equation}
where $\VHuber$ is defined in \eqref{eq:lyapunovV}  with $\delta \leftarrow 1$, $\thetas\leftarrow \thetaspi$, 
$C_\pi = 1+ 2 f_\pi(\thetaspi)$ and $D^2_\Theta:\Theta^2 \to [0,1]$ is defined by \eqref{eq:def_D_theta}.
\end{llemma}

\begin{proof}
	Let $\theta_0 \in \Theta$, and consider 
	\begin{equation}
		H_{\theta_0}(X)=\left. (1/2)  \Exp^{-1}_{\theta_0}\parenthese{X^{(1)}} \middle/ 
		\defEns{\distT^2\parenthese{\theta_0,X^{(2)}}\middle/2 + 1}^{1/2} \right. \eqsp,
            \end{equation}
            where $X^{(1)},X^{(2)}$ are independent random variables with distribution $\pi$.
            
Let $\upgamma : [0,1] \to \Theta$ be the geodesic curve defined by $\upgamma:t \mapsto \Exp_{\theta_0}[t \upeta H_{\theta_0}(X)]$.
Using \cite[Lemma 1]{durmus:jimenez:moulines:said:wai:2020} with $\upgamma$ and $\VHuber$, we get
	\begin{align}
		\VHuber(\upgamma(1)) &\leq \VHuber(\theta_0) + \psr{\grad \VHuber(\theta_0)}{\dot{\upgamma}(0)}[\theta_0]
		+ (L/2) \normr{\dupgamma(0)}[\theta_0]^2 \\
				 &= \VHuber(\theta_0) + \upeta \psr{\grad \VHuber(\theta_0)}{H_{\theta_0}(X)}[\theta_0]
				 + ((1+\kappa)\upeta^2/2)\normr{H_{\theta_0}(X)}[\theta_0]^2 \eqsp,
	\label{eq:descent_huber}
	\end{align}
 by \Cref{prop:huber}.
We now compute the expectation of the terms in \eqref{eq:descent_huber}. 
Using that $(X^{(1)},X^{(2)})$ are independent, we obtain
\begin{equation}
	\label{eq:psr_bary}
	\expe{ \psr{\grad \VHuber (\theta_0)}{H_{\theta_0}(X)}[\theta_0]} = 
	(1/2) \psr{\grad \VHuber (\theta_0)}{\expe{\Exp^{-1}_{\theta_0}\parenthese{X^{(1)}}} 
	\expe{ \defEns{\distT^2\parenthese{\theta_0,X^{(2)}}\middle/2+1}^{-1/2}}}[\theta_0] \eqsp .
\end{equation}
Moreover, using \eqref{eq:Vgrad} and \Cref{lem:psr_convex,lem:jensen_dist} yields
\begin{align}
	&\expe{ \psr{\grad \VHuber (\theta_0)}{H_{\theta_0}(X)}[\theta_0]} \\
	&\quad = -(1/2) 
	\defEns{\distT^2\parenthese{\theta_0,\thetaspi} + 1 }^{-1/2}
	\expe{\psr{\Exp^{-1}_{\theta_0}(\thetaspi)}{\Exp^{-1}_{\theta_0}\parentheseLigne{X^{(1)}}}[\theta_0]} \expe{ \defEns{\distT^2\parentheseLigne{\theta_0,X^{(2)}}\middle/2+1}^{-1/2}} \\
	&\quad \leq -(1/4)\distT^2\parenthese{\theta_0,\thetaspi} \parentheseDeux{
	\defEns{\distT^2\parentheseLigne{\theta_0,\thetaspi} + 1 }
          \defEns{\distT^2\parentheseLigne{\theta_0,\thetaspi} + 2 f_\pi(\thetaspi) + 1 } }^{-1/2} \\
  \label{eq:proof_huber_1}
	&\quad \leq -(16C_{\pi})^{-1/2} D^2_\Theta(\theta_0,\thetaspi)\eqsp,
\end{align}
where $C_\pi = 1+ 2 f_\pi(\thetaspi)$ and $D^2_\Theta:\Theta^2 \to [0,1]$ is defined by \eqref{eq:def_D_theta}.
%and using Jensen's inequality along \eqref{def:barycenter_noncompact}, with the convex function $t \mapsto (t+1)^{-1/2}$ on $\rset_+^*$,
%\begin{equation}
%	\expe{\defEns{\distT^2\parenthese{\theta_0,X^{(2)}}\middle/2+1}^{-1/2}} 
%	\geq \defEns{ \expe{ \distT^2\parenthese{\theta_0,X^{(2)}}\middle/2} +1 }^{-1/2} 
%	= \tf^{-1}(\theta_0) \eqsp.
%\end{equation}
%Plugging both of these expressions into \eqref{eq:psr_bary} brings
%\begin{align}
%	\expe{ \psr{\grad \tf (\theta_0)}{H_{\theta_0}(X)}[\theta_0]} 
%	&= -(\upeta/4) \psr{\grad f(\theta_0)}{\grad f(\theta_0)}[\theta_0](f(\theta_0)+1)^{-1/2}
%	\expe{\defEns{\distT^2\parenthese{\theta_0,X^{(2)}}\middle/2+1}^{-1/2}} \\ 
%	&\leq -(\upeta/4) \normr{\grad f(\theta_0)}[\theta_0]^2 (f(\theta_0)+1)^{-1} = -\upeta \normr{\grad \tf (\theta_0)}[\theta_0]^2 \eqsp .
%	\label{eq:psr_bary2}
%\end{align}
Looking to bound the expectation of the last term in \eqref{eq:descent_huber}, 
%we use the triangle and Hölder's inequalities,
%\begin{align}
%	\expe{\normr{H_{\theta_0}(X)}[\theta_0]^2} 
%	&\leq 2\expe{ \normr{H_{\theta_0}(X) - \grad f(\theta_0)\middle/ \parenthese{2 
%	\defEns{\distT^2\parenthese{\theta_0,X^{(2)}}\middle/2+1}^{-1/2}}}[\theta_0]^2} \\
%	&\quad + (1/2) \normr{\grad f(\theta_0)}[\theta_0]^2 \expe{\defEns{\distT^2\parenthese{\theta_0,X^{(2)}}\middle/2+1}^{-1}}
%\end{align}
we use that $\normrLigne{\Exp^{-1}_{\theta_0} \parentheseLigne{X^{(1)}}}[\theta_0] = \distT(\theta_0,X^{(1)})$ and that $X^{(1)}$ has distribution $\pi$ to obtain,
\begin{align}
	\expe{\normr{H_{\theta_0}(X)}[\theta_0]^2} 
	&= (1/4) \expe{\distT^2\parentheseLigne{\theta_0,X^{(1)}}} \expe{\defEns{\distT^2\parentheseLigne{\theta_0,X^{(2)}}/2+1}^{-1}} \\
	&= (f_\pi(\theta_0)/2) \expe{\defEns{\distT^2\parentheseLigne{\theta_0,X^{(2)}}/2+1}^{-1}} \eqsp.
	\label{eq:hsquared}
\end{align}
Denote by $M = \distT(\thetaspi,\theta_0)/2$. We bound the expectation in \eqref{eq:hsquared} using the event 
$\defEnsLigne{\distT\parentheseLigne{\thetaspi,X^{(2)}} \geq M}$ and its complement.
On $\defEnsLigne{\distT\parentheseLigne{\thetaspi,X^{(2)}} \geq M}$, we use Markov's inequality with the increasing 
map $t\mapsto t^2/2 +1$, 
\begin{align}
	\expe{ \1_{[M,+\infty)}\parentheseLigne{\distT\parentheseLigne{\thetaspi,X^{(2)}}}/
	\parentheseDeuxLigne{\distT^2\parentheseLigne{\theta_0,X^{(2)}}/2+1} }
	&\leq \proba{\distT\parentheseLigne{\thetaspi,X^{(2)}} \geq M} \\
	&\leq \left.\parenthese{\expe{\distT^2\parentheseLigne{\thetaspi,X^{(2)}}}\middle/2 + 1}
		\middle/ \parenthese{M^2/2 + 1}\right. \eqsp .
\label{eq:indicator_1}
\end{align}
On $\defEnsLigne{\distT\parentheseLigne{\thetaspi,X^{(2)}} < M}$, using the triangle inequality, we have 
\begin{equation}
	\distT\parentheseLigne{\theta_0,X^{(2)}} \geq \absLigne{\distT\parentheseLigne{\theta_0,\thetaspi} - \distT\parentheseLigne{\thetaspi,X^{(2)}}}
	\geq \distT(\theta_0,\thetaspi) - M
	= M \eqsp.
\end{equation}
Then, we obtain
\begin{equation}
\label{eq:indicator_2}
\expe{ \1_{[0,M)}\parentheseLigne{\distT\parentheseLigne{\theta,X^{(2)}}}/
	\defEnsLigne{\distT^2\parentheseLigne{\theta_0,X^{(2)}}/2+1} }
	\leq \left. 1 / \parentheseDeuxLigne{M^2/2 + 1}\right. \eqsp.
\end{equation}
Adding \eqref{eq:indicator_1} and \eqref{eq:indicator_2} together and using the definition of $M$ we obtain,
\begin{equation}
	\label{eq:hsquared_2}
	\expe{\defEns{\distT^2\parenthese{\theta_0,X^{(2)}}\middle/2+1}^{-1}} 
	\leq \left. \parenthese{f_{\pi}(\thetaspi) + 2} \middle/ \parentheseDeux{\distT^2(\thetaspi,\theta_0)/8 + 1}\right. \eqsp .
	%\leq  \parenthese{f(\theta) + 2} \eqsp.
      \end{equation}
      Plugging \eqref{eq:hsquared_2} in \eqref{eq:hsquared}, we get
      \begin{equation}
        \label{eq:hsquared_3}
        \expe{\normr{H_{\theta_0}(X)}[\theta_0]^2}  \leq  (f_\pi(\theta_0)/2)\parenthese{f_{\pi}(\thetaspi) + 2} / \parentheseDeux{\distT^2(\thetaspi,\theta_0)/8 + 1} \eqsp.
      \end{equation}
Using the triangle and Hölder's inequalities, we have for any $\theta$ and $\nu \in \Theta$,
	$\distT^2(\theta,\nu)/2 \leq \distT^2(\theta,\thetaspi) + \distT^2(\thetaspi,\nu)$. 
	Taking the integral with respect to $\pi$, by \Cref{ass:second_moment_pi} we get 
	$f_\pi (\theta) \leq \distT^2(\theta,\thetaspi) + 2 f_\pi(\thetaspi)$. Combining this result and \eqref{eq:hsquared_3}, we obtain
\begin{equation}
        \label{eq:hsquared_4}
        \expe{\normr{H_{\theta_0}(X)}[\theta_0]^2}  \leq  \{\distT^2(\thetaspi,\theta_0)/2 +  f_\pi(\thetaspi)\}\parenthese{f_{\pi}(\thetaspi) + 2} / \parentheseDeux{\distT^2(\thetaspi,\theta_0)/8 + 1} \leq 4\{1+f_\pi(\thetaspi)\}\parenthese{f_{\pi}(\thetaspi) + 2} \eqsp.
      \end{equation}
      Combining this result and \eqref{eq:proof_huber_1} in \eqref{eq:descent_huber} concludes the proof. 
\end{proof}
\begin{proof}[Proof of \Cref{theo:drift_bary}]
	Let $\theta_0 \in \Theta, \upeta >0$ and $n \in \nset$. Then, for any $k \in \defEnsLigne{1, \dots, n}$,
	using Markov's property and \Cref{lem:descente_barycentre} we have,
	\begin{align}
		[\upeta / (4C_\pi^{1/2})] \expe{D^2_\Theta(\theta_{k-1},\thetaspi)} 
		&=[\upeta / (4C_\pi^{1/2})]\int_\Theta D^2_\Theta(\theta,\thetaspi) \MKer^{k-1} (\theta_0,\rmd \theta)\\
		&\leq \MKer^{k-1} \VHuber (\theta_0) - \MKer^k \VHuber(\theta_0) 
		+ 2 \upeta^2 (1+ \kappa)(1+f_\pi(\thetaspi))(f_\pi(\thetaspi)+2) \eqsp. 
	\end{align}
	Summing these inequalities for $k \in \defEnsLigne{1, \dots, n}$ implies that 
	\begin{equation}
	[\upeta / (4C_\pi^{1/2})]\sum_{k=0}^{n-1} \expe{D^2_\Theta(\theta_{k},\thetaspi)}
		\leq \VHuber(\theta_0) - \MKer^n \VHuber(\theta_0) 
		+ 2 n \upeta^2 (1+ \kappa)(1+f_\pi(\thetaspi))(f_\pi(\thetaspi)+2) \eqsp. 
	\end{equation}
	Finally, dividing both sides by $[n \upeta / (4C_\pi^{1/2})]$ and using that $\VHuber$ is a non-negative function, we obtain
	\begin{equation}
		n^{-1} \sum_{k=0}^{n-1} \expe{D^2_\Theta(\theta_{k},\thetaspi)}
		\leq \left. 2\VHuber(\theta_0)C_\pi^{1/2} \middle/ (\upeta n) \right. 
		+ 2 \upeta (1+\kappa) (f_\pi(\thetaspi) + 1) 
		(f_\pi(\thetaspi) + 2) (2 f_\pi(\thetaspi) + 1)^{-1/2} \eqsp .
	\end{equation}
	Which concludes the proof by setting 
	$B_\pi =(1+\kappa) (f_\pi(\thetaspi) + 1) 
		(f_\pi(\thetaspi) + 2) (2 f_\pi(\thetaspi) + 1)^{-1/2} $. 
\end{proof}

\begin{figure}
	\centering
	\includegraphics[width=.5\linewidth]{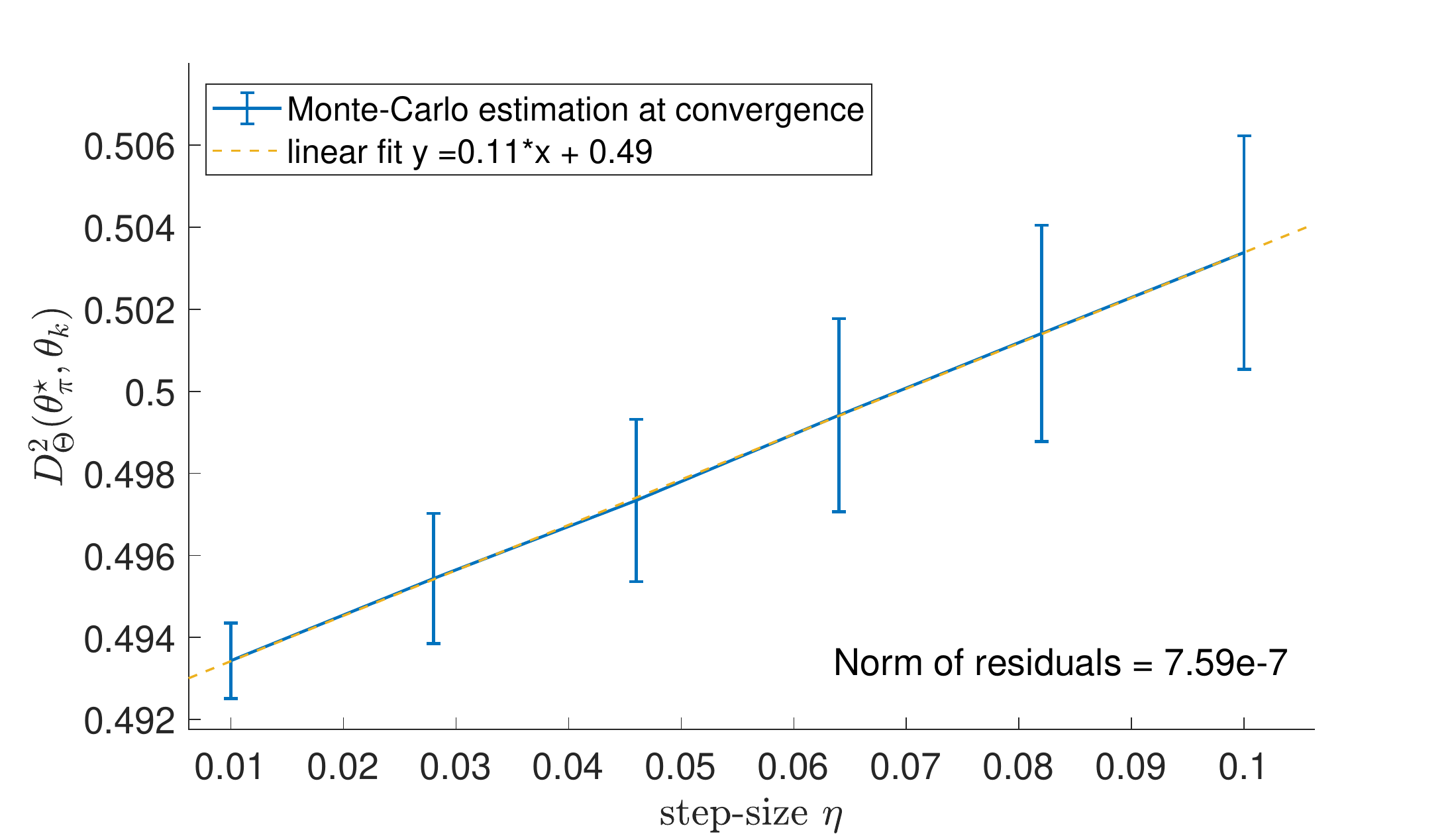}
	\caption{Monte Carlo approximations of the mean distance at convergence in \Cref{theo:drift_bary}}
	\label{fig:monte_supp}
\end{figure}
Similarly to \Cref{fig:monte}, \Cref{fig:monte_supp} illustrates 
\Cref{thm:central_limit}. To this end, 1000 replications of
the experiment
derived for \Cref{fig:path_unbounded} are performed, obtaining
$\defEnsLigne{(\theta_n^{(i)}) \, : \, i \in \defEnsLigne{1, \dots
, 1000}}$ for $n = \lceil 50/\upeta \rceil $ and 
$\upeta \in \defEnsLigne{1,2.8,4.6,6.4,8.2,10} \times 10^{-2}$.
We estimate, with these samples, the mean and the variance of 
$D_{\Theta}^2(\theta,\thetaspi)$, for $\theta$ following the stationary 
distribution $\mu^\upeta$.
We observe that the mean and variance are both linear \wrt~the step-size
$\upeta$, indicating that the iterates of the SA scheme remain in a 
neighborhood of diameter $\bigO(\upeta^{1/2})$ to the ground truth.

Even though the setting of this experiment goes beyond the assumptions 
of \Cref{thm:central_limit}, it suggests that such a result may be 
applicable also in the setting of \Cref{theo:drift_bary}. 
The proof of such a result is left for future work. 
%%% Local Variables:
%%% mode: latex
%%% TeX-master: "main_supplement"
%%% End:

\section{Background on Markov chain theory and Riemannian geometry}\label{app:preliminaries}
We give here some useful definitions and results that are used throughout the paper.
\subsection{Markov chain notions}\label{app:markov}
We refer to \cite{meyn:tweedie:2009} for a general introduction to
Markov chains in general state space.  Let $(\msy,\mcy)$ be a
measurable state space and $P$ be a Markov kernel on
$\msy \times \mcy$. Consider for any $y \in \msy$, the distribution
$\PP_{y}$ of the canonical Markov chain $(Y_n)_{n\in\nset}$
corresponding to $P$ and starting from $y$ on the canonical space
$(\msy^{\nset},\mcy^{\otimes \nset})$. Denote by $\PE_{y}$ the corresponding expectation. 

Denote for any $\msa \in \mcy$, $\tau_\msa = \inf \defEnsLigne{{l} \geq 1 \,:\, Y_l\in \msa}$ and $N_\msa = \sum_{{l} = 1}^{+\infty} \1_{\defEnsLigne{ \msa}}(Y_l )$.

We say that $(Y_n)_{n\in\nset}$ is $\psi$-irreducible if there exists a measure $\psi$ on $\mcy$ such that whenever $\psi(\msa)>0$, we have $\PP_y(\tau_\msa < \infty)>0$ for any $y\in \msy$. Moreover, a set $\msa\in \mcy$ is called Harris-recurrent if $\PP_y(N_\msa=\infty)=1$ for any $y \in \msa$. Finally, a chain $(Y_n)_{n\in\nset}$ is called Harris-recurrent if it is $\psi$-irreducible and every set $\msa\in \mcy$ such that $\psi(\msa)>0$ is Harris-recurrent.

Let $\bar{V} :\msy \to \coint{1,\plusinfty}$. We say that $P$ is
$\bar{V}$-uniformly geometrically ergodic if there exist
$\rho \in \coint{0,1}$ and $C \geq 0$ such that for any $y \in \msy$ and $k \in\nset$,
$\normr{\updelta_y P^k - \mu}[\bar{V}] \leq C \rho^k \bar{V}(y)$,
where $\normr{\cdot}[\bar{V}]$ is defined for two probability measures
$\nu_1,\nu_2$ on $(\msy,\mcy)$ by
$\normr{\nu_1 - \nu_2}[\bar{V}]= \sup \{ \abs{\nu_1(g) - \nu_2(g) } \,
: \, \sup_{\msy} \{ \abs{g} /\bar{V} \} \leq 1\}$.

\subsection{Useful results from Riemannian geometry}\label{app:geometry}

We now give definitions and auxiliary results related to tensor fields along curves, 
their derivatives, and Taylor expansions on Riemannian manifolds.

Let $\msm$ be a smooth manifold with or without boundary. Given a smooth curve $\upgamma:I\to \msm$ defined on an interval $I$, and any $k,l\in \nset$, a $(k,l)$-tensor field along $\upgamma$ is a continuous map $F:I \to \planT^{(k,l)} \planT \msm$, such that $F(t) \in \planT^{(k,l)}(\planT_{\upgamma(t)}\msm)$ for any $t \in I$, where $\planT^{(k,l)} \planT \msm$ is the bundle of $(k,l)$-tensors on $\msm$, see \eg~\cite[Appendix B]{lee:2019}. A vector field $Y$ along $\upgamma$ is a $(1,0)$-tensor field, in which case for any $t\in I$, $Y(t)$ is just a tangent vector in $\planT_{\upgamma(t)} \msm$. We say that a tensor field $F$ along $\upgamma$ is extendible if there exists a tensor field $\tilde{F}$ defined on a neighborhood of $\upgamma(I)$ such that $F=\tilde{F} \circ \upgamma$.

We let $\frX^{k,l}(\upgamma)$ denote the set of smooth $(k,l)$-tensor
fields along $\upgamma$, and $\frX(\upgamma)=\frX^{1,0}(\upgamma)$
denote the set of smooth vector fields along $\upgamma$. In
particular, $\frX^{0,0}(\upgamma)$ is the set of smooth functions
$g : I \to \upgamma(I)\times \rset$ such that for any $t \in I$,
$g(t) = (\upgamma(t),f(t))$ for some smooth function $f : I \to \rset$
and therefore can be identified with the set of smooth functions
$f: I \to \rset$. In the sequel, we adopt if no confusion is possible
this identification. We extend to tensor fields along $\upgamma$ the
following definition of the trace on tensors.  For any $(k,l)$-tensor
$T$, we denote by $\trace_{\square,\triangle}(T)$ the
$(k-1,l-1)$-tensor with component of index
$(i_1,\ldots,i_{k-1},j_1,\ldots,j_{l-1})$, given by
$\sum_{m=1}^{d} T_{i_1,\ldots,i_{\square-1},m, i_{\square},\ldots,
  i_{k-1}}^{j_1,\ldots,j_{\triangle-1},m, j_{\triangle},\ldots,
  j_{l-1}}$. In particular, for any
$\omega\in \frX^{0,1}(\upgamma),Y \in \frX(\upgamma)$,
  \begin{equation}
  \trace_{(1,1)}(\omega \otimes Y) = \omega(Y) \eqsp.
  \end{equation}
Also, for any $F\in \frX^{k,l}(\upgamma)$,
any $\omega^1,\ldots,\omega^{k_0}\in \frX^{0,1}(\upgamma)$ and $Y_1,\dots,Y_{l_0}\in \frX(\upgamma)$, with $k_0 \leq k,l_0\leq l$, denote by
$[F\, :\, \omega^1\otimes \cdots \otimes \omega^{k_0} \otimes Y_1\otimes \cdots
\otimes Y_{l_0}]$, the $(k-k_0,l-l_0)$ smooth tensor field along $\upgamma$ defined by the induction:
\begin{align}
  \label{eq:contraction1}
  [F\, :\, \omega^{\otimes 1:i}]& =   \trace_{(1,l+1)}([F\, :\, \omega^{\otimes 1:(i-1)} ] \otimes\omega^i)\\
  \label{eq:contraction2}[F\, :\, \omega^{\otimes 1:{k_0}} \otimes  Y_{\otimes 1:j}] &=   \trace_{(k-{k_0}+1,1)}([F\, :\, \omega^{\otimes 1:{k_0}} \otimes Y_{\otimes 1:(j-1)}] \otimes Y_j)
 \eqsp,
\end{align}
setting
$\omega^{\otimes 1:i} = \omega^{1} \otimes \cdots \otimes \omega^{i}$,
$Y_{\otimes 1:j} = Y_1\otimes \cdots \otimes Y_j$.  Note that for any $\omega^{k-{k_0}+1},\ldots, \omega^{k} \in \frX^{0,1}(\upgamma)$ and
$Y_{l-{l_0}+1},\ldots,Y_{l} \in \frX(\upgamma)$,
\begin{equation}
  \label{eq:action_contraction}
  [F\, :\, \omega^{\otimes 1:{k_0}}\otimes Y_{\otimes 1:{l_0}}](\omega^{k-{k_0}+1},\ldots, \omega^{k},Y_{l-{l_0}+1},\ldots,Y_{l})
  =F(\omega^1,\ldots,\omega^{k},Y_1,\ldots,Y_{l}) \eqsp.
\end{equation}
\begin{pproposition}\label{prop:ten_cov_der}
Let $\msm$ be a smooth manifold with or without border, $\nabla$ be a connection on $\planT \msm$ and $\upgamma:I \to \msm$ a smooth curve defined on an interval $I$. Then, for any $k,l\in \nset$, $\nabla$ determines an operator $\rmD_t : \frX^{k,l}(\upgamma) \to \frX^{k,l}(\upgamma)$,
satisfying the following conditions.
\begin{enumerate}[wide, labelwidth=!, labelindent=0pt,label=(\alph*),noitemsep,nolistsep]
\item \label{prop:item:ten_cov_der_a}On $\frX(\upgamma)$, $\rmD_t$ is the usual covariant derivative along $\upgamma$, see \cite[Theorem 4.24]{lee:2019}.
\item \label{prop:item:ten_cov_der_b}On $\frX^{0,0}(\upgamma)$,
  $\rmD_t$ is the usual derivative for real functions,
  \ie~for any $f \in \frX^{0,0}(\upgamma)$, $\rmD_t f = \rmd f / \rmd t$.
 \item  For any $F \in \frX^{k,l}(\upgamma)$, any $\omega^1,\dots,\omega^k \in \frX^{0,1}(\upgamma)$ and any $Y_1,\dots,Y_l\in \frX(\upgamma)$,
\begin{equation}\label{eq:ten_cov_der}
\begin{aligned}
(\rmD_t F ) \parenthese{\omega^1,\dots,\omega^k,Y_1,\dots,Y_l} &= \frac{\rmd}{\rmd t} \parentheseDeux{F \parenthese{\omega^1,\dots,\omega^k,Y_1,\dots,Y_l}} \\
&\quad - \sum_{i=1}^k F \parenthese{\omega^1,\dots,\omega^{i-1},\rmD_t \omega^i, \omega^{i+1},\dots, \omega^k,Y_1,\dots,Y_l} \\
&\quad - \sum_{j=1}^l F \parenthese{\omega^1,\dots,\omega^k,Y_1,\dots,Y_{j-1},\rmD_t Y_{j}, Y_{j+1},\dots, Y_l} \eqsp.
\end{aligned}
\end{equation}
\end{enumerate}
In particular, $\rmD_t$ satisfies these additional properties. 
\begin{enumerate}[wide, labelwidth=!, labelindent=0pt,label=(\roman*),noitemsep,nolistsep]
\item \label{prop:item:ten_cov_der_e}$\rmD_t$ satisfies the product rule, \ie~for any $f\in \frX^{0,0}(\upgamma),F\in \frX^{k,l}(\upgamma),$
\begin{equation}
\rmD_t \parenthese{fF} = \parenthese{\frac{\rmd}{\rmd t} f} F + f \rmD_t F \eqsp.
\end{equation}
\item \label{prop:item:ten_cov_der_c}For any $k_1,l_1,k_2,l_2 \in \nset$, and any $F \in \frX^{k_1,l_1}(\upgamma), G \in \frX^{k_2,l_2}(\upgamma)$,
\begin{equation}
\rmD_t (F\otimes G) = \rmD_t F \otimes G + F \otimes \rmD_t G \eqsp. 
\end{equation}
\item \label{prop:item:ten_cov_der_d} For any positive integers $k_0\leq k,l_0\leq l$, $F \in \frX^{k,l}(\upgamma)$,
\begin{equation}
\rmD_t \defEns{\trace_{(k_0,l_0)}(F)} = \trace_{(k_0,l_0)}\parenthese{\rmD_t F} \eqsp.
\end{equation}
\item \label{prop:item:ten_cov_der_extension} Let $F\in \frX^{k,l}$ be an extendible tensor field, \ie, 
such that there exists a $(k,l)$-tensor field $\tilde{F}$ defined on a neighborhood of $\upgamma(I)$ 
satisfying for any $t \in I$, $F(t)=\tilde{F}(\upgamma(t))$. Then, for any $t\in I$,
\begin{equation}\label{eq:ten_cov_der_extension}
\rmD_t F(t) = \nabla_{\dot{\upgamma}(t)} \tilde{F}({\upgamma(t)}) \eqsp.
\end{equation}
%to do... \alain{to do}
\end{enumerate}
Finally, if $\tilde{\rmD}_t : \frX^{k,l}(\upgamma) \to \frX^{k,l}(\upgamma)$ is another operator satisfying \ref{prop:item:ten_cov_der_a},\ref{prop:item:ten_cov_der_b},\ref{prop:item:ten_cov_der_e},\ref{prop:item:ten_cov_der_c} and \ref{prop:item:ten_cov_der_d}, then $\rmD_t = \tilde{\rmD}_t$. 
\end{pproposition}
\begin{proof}
Let $k,l \in \nset$.  Note first that
\ref{prop:item:ten_cov_der_a}-\ref{prop:item:ten_cov_der_b} and
\eqref{eq:ten_cov_der} define
$\rmD_t F$ for any $F \in  \frX^{k,l}(\upgamma)$, setting for
any $\omega \in \frX^{0,1}(\upgamma)$ and $Y \in \frX(\upgamma)$,
\begin{equation}
  \label{eq:def_rmd_t_omega}
\parentheseDeux{\rmD_t \omega } (Y) = {\rmd} \parentheseDeux{\omega(Y)}/{\rmd t} - \omega(\rmD_t Y) \eqsp.
\end{equation}
We now show that $\rmD_t F \in \frX^{k,l}$, which will imply that
$\rmD_t : \frX^{k,l} \to \frX^{k,l}$. Second, we establish that
\ref{prop:item:ten_cov_der_e}-\ref{prop:item:ten_cov_der_c}-\ref{prop:item:ten_cov_der_d}-\ref{prop:item:ten_cov_der_extension}
are satisfied. We conclude the proof by proving uniqueness of
$\rmD_t$.

Using \cite[Lemma
B.6]{lee:2019}, to show that $\rmD_t F \in \frX^{k,l}$ it is enough to prove that
$\rmD_t F$ is multilinear over $\frX^{0,0}(\upgamma)$. 
  For that, we start proving
 \ref{prop:item:ten_cov_der_e} on $\frX^{0,1}(\upgamma)$. Let
 $\omega \in \frX^{0,1}(\upgamma), f \in \frX^{0,0}(\upgamma)$ and
 $Y \in \frX(\upgamma)$, then by \eqref{eq:def_rmd_t_omega},
 \begin{equation}
     \label{eq:linear_rmd_t_omega}
\parentheseDeux{\rmD_t (f \omega)}(Y) =  {\rmd} \parentheseDeux{f \omega (Y)}/{\rmd t} - f \omega \parenthese{\rmD_t Y} 
%&= \parentheseDeux{\frac{\rmd}{\rmd t} f} \omega (Y) + f \frac{\rmd}{\rmd t} \parentheseDeux{\omega (Y)} - f \omega \parenthese{\rmD_t Y} \\
%&= \parentheseDeux{\frac{\rmd}{\rmd t} f } \omega (Y) + f \frac{\rmd}{\rmd t} \parentheseDeux{\omega(Y)} -  f \omega \parenthese{\rmD_t Y} \\
= \parentheseDeux{{\rmd} f/ {\rmd t}} \omega(Y) + f \parentheseDeux{\rmD_t \omega} (Y) \eqsp,
\end{equation}
which proves \ref{prop:item:ten_cov_der_e} on
$\frX^{0,1}(\upgamma)$. Now, let
$k,l\in \nset, F\in \frX^{k,l}(\upgamma), \omega^1,\dots,\omega^k \in
\frX^{0,1}(\upgamma),Y_1,\dots,Y_l \in \frX(\upgamma)$. Let
$f \in \frX^{0,0}(\upgamma)$ and $k_0 \in \nsets, k_0\leq k$. We have,
using the multilinearity of $F$ over $\frX^{0,0}(\upgamma)$, the
definition of $\rmD_t$ \eqref{eq:ten_cov_der}, and
\eqref{eq:linear_rmd_t_omega} 
\begin{align}
&\parentheseDeux{\rmD_t F} \parenthese{\omega^1,\dots, \omega^{k_0-1}, f \omega^{k_0},\omega^{k_0+1}, \dots, \omega^{k}, Y_1, \dots, Y_l} \\
&\qquad= \frac{\rmd}{\rmd t} \parentheseDeux{F\parenthese{\omega^1,\dots, \omega^{k_0-1}, f \omega^{k_0},\omega^{k_0+1}, \dots, \omega^{k}, Y_1, \dots, Y_l}} \\
&\qquad \qquad \qquad \quad - \sum_{i=1,i\neq k_0}^k f F\parenthese{\omega^1,\dots,\omega^{i-1},\rmD_t \omega^{i},\omega^{i+1},\dots\omega^k,Y_1,\dots,Y_l} \\
&\qquad \qquad \qquad \quad- F \parenthese{\omega^1,\dots,\omega^{k_0-1},\rmD_t (f\omega^{k_0}),\omega^{k_0+1},\dots,\omega^k,Y_1,\dots,Y_l} \\
&\qquad\qquad \qquad  \quad- \sum_{j=1}^l f F\parenthese{\omega^1,\dots,\omega^k,Y_1,\dots,Y_{j-1},\rmD_t Y_{j},Y_{j+1},\dots,Y_{l}} \\
&\qquad= \parentheseDeux{\frac{\rmd}{\rmd t} f } \defEns{F \parenthese{\omega^1,\dots,\omega^k,Y_1,\dots,Y_k}- F \parenthese{\omega^1,\dots,\omega^k,Y_1,\dots,Y_k}} \\
&\qquad \qquad \qquad \quad +f \parentheseDeux{\rmD_t F} \parenthese{\omega^1,\dots,\omega^k,Y_1,\dots,Y_l} \\
& \qquad  =f \parentheseDeux{\rmD_t F} \parenthese{\omega^1,\dots,\omega^k,Y_1,\dots,Y_l} \eqsp.
\end{align}
% where we have used 
% \begin{align}
% \frac{\rmd}{\rmd t} \parentheseDeux{f F\parenthese{\omega^1,\dots,\omega^k,Y_1,\dots,Y_l} } &= \parentheseDeux{\frac{\rmd}{\rmd t} f} F\parenthese{\omega^1,\dots,\omega^k,Y_1,\dots,Y_l} \\
% &\quad + f \frac{\rmd}{\rmd t}\parentheseDeux{ F\parenthese{\omega^1,\dots,\omega^k,Y_1,\dots,Y_l} } \eqsp,
% \end{align}
%and the product rule for $f \omega^{k_0}$.
The same arguments apply if we replace $Y_{l_0}$ with $fY_{l_0}$, for some $l_0 \leq l$. Thus, using \cite[Lemma B.6]{lee:2019}, $\rmD_t F \in \frX^{k,l}$.

Next, regarding \ref{prop:item:ten_cov_der_e}, using the definition of $\rmD_t$,
\begin{align}
 \parentheseDeux{\rmD_t fF} \parenthese{\omega^1,\dots,\omega^k,Y_1,\dots,Y_l} &= \parentheseDeux{\frac{\rmd}{\rmd t} f} F \parenthese{\omega^1,\dots,\omega^k,Y_1,\dots,Y_l} \\
 &\quad + f \parentheseDeux{\rmD_t F} \parenthese{\omega^1,\dots,\omega^k,Y_1,\dots,Y_l} \eqsp,
\end{align} 
thus proving \ref{prop:item:ten_cov_der_e}. Moreover, we prove \ref{prop:item:ten_cov_der_c}. Let $k_1,l_1,k_2,l_2 \in \nset$ and $F \in \frX^{k_1,l_1}(\upgamma),G\in \frX^{k_2,l_2}(\upgamma)$, $\omega^1,\dots,\omega^{k_1+k_2}\in \frX^{0,1}(\upgamma),Y_1,\dots,Y_{l_1+l_2}\in \frX(\upgamma)$. Setting $$f = F\parentheseLigne{\omega^1,\dots,\omega^{k_1},Y_1,\dots,Y_{l_1}} \text{ and } g = G \parentheseLigne{\omega^{k_1+1},\dots,\omega^{k_1+k_2},Y_{l_1+1},\dots,Y_{l_1+l_2}} \eqsp,$$ we have
\begin{align}
&\parentheseDeux{\rmD_t (F\otimes G)} \parenthese{\omega^1,\dots,\omega^{k_1+k_2},Y_1,\dots,Y_{l_1+l_2}} \\
&\quad= \frac{\rmd}{\rmd t} \parentheseDeux{fg }  - \left[ \sum_{i=1}^{k_1} F \parenthese{\omega^1,\dots ,\omega^{i-1},\rmD_t \omega^{i},\omega^{i+1},\dots,\omega^{k_1},Y_1,\dots,Y_{l_1}}\right. \\
&\quad\quad + \left. \sum_{j=1}^{l_1} F \parenthese{\omega^1,\dots,\omega^{k_1},Y_1,\dots,Y_{j-1},\rmD_t Y_j ,Y_{j+1},\dots,Y_{l_1}} \right] g\\
&\qquad - f \left[ \sum_{i=1}^{k_2} G \parenthese{\omega^{k_1+1},\dots,\omega^{k_1+i-1},\rmD_t \omega^{k_1+i},\omega^{k_1+i+1},\dots,\omega^{k_1+k_2},Y_{l_1+1},\dots,Y_{l_1+l_2}} \right.\\
&\qquad + \left. \sum_{j=1}^{l_2} G \parenthese{\omega^{k_1+1},\dots,\omega^{k_1+k_2},Y_{l_1+1},\dots,Y_{l_1+j-1},\rmD_t Y_{l_1+j},Y_{l_1+j+1},\dots,Y_{l_1+l_2}} \right] \\
&\quad = \parentheseDeux{\rmD_t F} \parenthese{\omega^1,\dots,\omega^{k_1},Y_1,\dots,Y_{l_1}} g   + f \parentheseDeux{\rmD_t G}\parenthese{\omega^{k_1+1},\dots,\omega^{k_1+k_2},Y_{l_1+1},\dots,Y_{l_1+l_2}} \\
&\quad = \parentheseDeux{\rmD_t F \otimes G + F \otimes \rmD_t G}\parenthese{\omega^{1},\dots,\omega^{k_1+k_2},Y_{1},\dots,Y_{l_1+l_2}} \eqsp,
\end{align}
which proves \ref{prop:item:ten_cov_der_c}.
Furthermore, to prove \ref{prop:item:ten_cov_der_d}, let $t_0 \in I$ and $(\bfb_i)_{ i \in \{1,\dots,d\}}$ be a basis of $\planT_{\upgamma(t_0)} \Theta$. Using \ref{prop:item:ten_cov_der_a} and \cite[Theorem 4.32]{lee:2019}, define for any $i \in \{1,\dots,d\}$ and $t\in I$,
\begin{equation}\label{eq:parallel_frame}
e_i(t) = \parallelTransport_{t_0,t}^{\upgamma} \bfb_i \eqsp,
\end{equation}
where $\parallelTransport_{t_0,t}^{\upgamma}$ denotes the parallel transport map along $\upgamma$ from $\planT_{\upgamma(t_0)} \Theta$ to $\planT_{\upgamma(t)} \Theta$. As the parallel transport map is an isomorphism, $(e_i(t))_{ i \in \{1,\dots,d\}}$ is a basis of $\planT_{\upgamma(t)} \Theta$, for any $t\in I$. Therefore the family of smooth vector fields $(e_i)_{ i \in \{1,\dots,d\}}$ is a parallel frame along $\upgamma$ (with respect to $\nabla$). Denote $(\varepsilon^j)_{ j \in \{1,\dots,d\}}$ its dual coframe. Using \eqref{eq:def_rmd_t_omega} on $Y=e_i,\omega = \varepsilon^j$, for any $i,j \in \defEnsLigne{1,\dots,d}$,  shows that the coframe $(\varepsilon^j)_{ j \in \{1,\dots,d\}}$ is parallel along $\upgamma$.
Note that for $(e_i)_{ i \in \{1,\dots,d\}}$ and $(\varepsilon^j)_{ j \in \{1,\dots,d\}}$ to be well defined, we have used $\nabla$, as well as the operator $\rmD_t$ on $\frX(\upgamma)$ and $\frX^{0,1}(\upgamma)$. 

Let $k,l \in \nsets$ such that $k_0\leq k,l_0\leq l$, and let $F\in \frX^{k,l}(\upgamma)$. There exist a family of functions $\defEnsLigne{F^{j_1,\dots,j_l}_{i_1,\dots,i_k} \in \frX^{0,0}(\upgamma) \,:\, i_1,\dots,i_k,j_1,\dots,j_l \in \defEnsLigne{1,\dots,d}}$ such that
\begin{equation}
F = \sum_{i_1,\dots,i_k=1}^d \sum_{j_1,\dots,j_l=1}^d F^{j_1,\dots,j_l}_{i_1,\dots,i_k} \bigotimes_{\triangle = 1}^{k} e_{i_\triangle} \bigotimes_{\square = 1}^{l} \varepsilon^{j_\square} \eqsp.
\end{equation}
Since the frame and its dual coframe are parallel along $\upgamma$, for any $i\in \defEnsLigne{1,\dots,d}$ $\rmD_t e_i = 0$ and $\rmD_t \varepsilon^i = 0$. Combining this fact with \ref{prop:item:ten_cov_der_e} and \ref{prop:item:ten_cov_der_c} gives
\begin{equation} \label{eq:dt_frame}
\rmD_t F = \sum_{i_1,\dots,i_k=1}^d \sum_{j_1,\dots,j_l=1}^d \parentheseDeux{\frac{\rmd}{\rmd t}F^{j_1,\dots,j_l}_{i_1,\dots,i_k}} \bigotimes_{\triangle = 1}^{k} e_{i_\triangle} \bigotimes_{\square = 1}^{l} \varepsilon^{j_\square} \eqsp.
\end{equation}
Let $k_0,l_0\in\nsets$ such that $k_0\leq k, l_0 \leq l$, then by definition of $\trace_{(k_0,l_0)}$, for any $i_1,\dots,i_{k-1},j_1,\dots,j_{l-1} \in \defEnsLigne{1,\dots,d}$,
\begin{equation} \label{eq:trace_tensor_field}
\trace_{(k_0,l_0)}(F)^{j_1,\dots,j_{l-1}}_{i_1,\dots,i_{k-1}} = \sum_{m=1}^d F^{j_1,\dots,j_{l_0-1},m,j_{l_0},\dots,j_{l-1}}_{i_1,\dots,i_{k_0-1},m,i_{k_0},\dots,i_{k-1}}\eqsp.
\end{equation}
We remind the reader that $\trace_{(k_0,l_0)}(F)$ does not depend on the choice of coordinates \cite[Appendix B]{lee:2019}.
Thus, using \eqref{eq:dt_frame} and \eqref{eq:trace_tensor_field}, we have 
\begin{align}
\rmD_t \parentheseDeux{\trace_{(k_0,l_0)} (F)} &= \sum_{i_1,\dots,i_{k-1}=1}^d \sum_{j_1,\dots,j_{l-1}=1}^d \frac{\rmd}{\rmd t}\parentheseDeux{\trace_{(k_0,l_0)}(F)^{j_1,\dots,j_{l-1}}_{i_1,\dots,i_{k-1}}} \bigotimes_{\triangle = 1}^{k-1} e_{i_\triangle} \bigotimes_{\square = 1}^{l-1} \varepsilon^{j_\square} \\
&= \sum_{i_1,\dots,i_{k-1}=1}^d \sum_{j_1,\dots,j_{l-1}=1}^d \sum_{m=1}^d \frac{\rmd}{\rmd t} F^{j_1,\dots,j_{l_0-1},m,j_{l_0},\dots,j_{l-1}}_{i_1,\dots,i_{k_0-1},m,i_{k_0},\dots,i_{k-1}} \bigotimes_{\triangle = 1}^{k-1} e_{i_\triangle} \bigotimes_{\square = 1}^{l-1} \varepsilon^{j_\square} \\
&= \trace_{(k_0,l_0)}\parenthese{\rmD_t F} \eqsp,
\end{align}
thus proving \ref{prop:item:ten_cov_der_d}.

To prove \ref{prop:item:ten_cov_der_extension}, first for any $f \in \frX^{(0,0)}(\upgamma)$, 
extendible in $\tilde{f}$,
we have by composition and definition of the covariant derivative, that for any $t \in [0,1]$, 
\begin{equation}\label{eq:extendible_f}
(\rmd f / \rmd t )(t) = \rmd \tilde{f}_{\upgamma(t)} (\dot{\upgamma}(t))
=\nabla_{\dot{\upgamma}(t)}\tilde{f} (\upgamma(t)) \eqsp. 
\end{equation}
Also, using \cite[Theorem 4.24-(iii)]{lee:2019} gives
\ref{prop:item:ten_cov_der_extension} for any $Y \in \frX(\upgamma)$. Combining \eqref{eq:extendible_f}, 
\eqref{eq:def_rmd_t_omega}, its counterpart for tensor fields defined over a manifold 
\cite[Proposition 4.15-(a)]{lee:2019} and \ref{prop:item:ten_cov_der_extension} over $\frX(\upgamma)$, 
proves \ref{prop:item:ten_cov_der_extension} over $\frX^{(0,1)}(\upgamma)$.
Now, for any $k,l \in \nset$,
using \ref{prop:item:ten_cov_der_extension} over $\frX(\upgamma)$ and $\frX^{(0,1)}(\upgamma)$ 
combined with \eqref{eq:ten_cov_der} and its counterpart for tensor fields 
defined over a manifold \cite[Equation (4.12)]{lee:2019} gives 
\ref{prop:item:ten_cov_der_extension} over $\frX^{(k,l)}(\upgamma)$.

Finally, we address uniqueness. Suppose now that $\tilde{\rmD}_t$ is an operator on $\frX^{k,l}(\upgamma)$ that satisfies \ref{prop:item:ten_cov_der_a},\ref{prop:item:ten_cov_der_b},\ref{prop:item:ten_cov_der_e},\ref{prop:item:ten_cov_der_c} and \ref{prop:item:ten_cov_der_d}. First, \ref{prop:item:ten_cov_der_a} and \ref{prop:item:ten_cov_der_b} show that $\rmD_t$ and $\tilde{\rmD}_t$ coincide on $\frX^{0,0}(\upgamma)$ and $\frX(\upgamma)$. Second, for any $Y\in \frX(\upgamma), \omega \in \frX^{0,1}(\upgamma)$, writing $\omega(Y)=\trace_{(1,1)}(Y \otimes \omega)$ and using \ref{prop:item:ten_cov_der_d} gives
\begin{equation}
\tilde{\rmD}_t \omega = \rmd [\omega(Y)] / \rmd t - \omega (\tilde{\rmD}_t Y) = \rmD_t \omega \eqsp,
\end{equation}
using \eqref{eq:def_rmd_t_omega}. Thus, $\tilde{\rmD}_t$ and $\rmD_t$ also agree on $\frX^{0,1}(\upgamma)$. Therefore, the frame $(e_i)_{ i \in \{1,\dots,d\}}$ and its dual coframe $(\varepsilon^j)_{ j \in \{1,\dots,d\}}$ are also parallel with respect to $\tilde{\rmD}_t$ along $\upgamma$.
Let $F\in \frX^{k,l}(\upgamma)$, then using \ref{prop:item:ten_cov_der_e} and \ref{prop:item:ten_cov_der_c} shows that  \eqref{eq:dt_frame} holds for the operator $\tilde{\rmD}_t$, proving that $\rmD_t F = \tilde{\rmD}_t F$. This concludes the proof. 

\end{proof}

\begin{llemma}\label{lem:tensor_derivative}
  Let $\msm$ be a smooth manifold and $\nabla$ be a connection on $\planT \msm$. Let $\upgamma:[0,1]\to \msm$ be a smooth
  curve and denote $\rmD_t$ the covariant derivative operator along
  $\upgamma$ associated with $\nabla$,  defined in \Cref{prop:ten_cov_der}.
Let $F\in \frX^{k,l}(\upgamma)$,
$\omega^1,\ldots,\omega^{k_0}\in \frX^{0,1}(\upgamma)$ and $Y_1,\dots,Y_{l_0}\in\frX(\upgamma)$, with ${k_0} \leq k,{l_0}\leq l$.  Then, we have 
\begin{equation}\label{eq:tensor_derivative}
\begin{aligned}
  \rmD_t \parenthese{[F : \omega^{\otimes 1:{k_0}} \otimes Y_{\otimes 1:{l_0}}]} &= [\rmD_t F : \omega^{\otimes 1:{k_0}} \otimes Y_{\otimes 1:{l_0}}] \\
  &\qquad+ \sum_{i=1}^{k_0} [F : \omega^{\otimes 1:(i-1)} \otimes \rmD_t \omega^{i} \otimes \omega^{(i+1):{k_0}} \otimes Y_{\otimes 1:{l_0}}] \\
&\qquad +  \sum_{j=1}^{l_0} [F : \omega^{\otimes 1:{k_0}} \otimes Y_{\otimes 1:(j-1)} \otimes \rmD_t Y_j \otimes Y_{\otimes (j+1):{l_0}}] \eqsp.
\end{aligned}
\end{equation}
\end{llemma}
\begin{proof}
Let $F$ be a smooth $(k,l)$-tensor field along $\upgamma$. We show \eqref{eq:tensor_derivative} by induction. Following the recursive definition of the contraction in \eqref{eq:contraction1}, we prove it by induction on ${k_0}\in\nsets, {k_0} \leq k$, for any $\omega^1,\dots,\omega^{k_0}\in \frX^{0,1}(\upgamma)$.

The case ${k_0}=1$ follows from \Cref{prop:ten_cov_der}-\ref{prop:item:ten_cov_der_c} and \ref{prop:item:ten_cov_der_d}, combined with the definition in \eqref{eq:contraction1},
\begin{align}
\rmD_t \parentheseDeux{F:\omega^1} &= \rmD_t \trace_{(1,l+1)}(F\otimes \omega^1) \\
&= \trace_{(1,l+1)}(\rmD_t [F \otimes \omega^1]) \\
&= \trace_{(1,l+1)} (\rmD_t F \otimes \omega^1 + F \otimes \rmD_t \omega^1) \\
&= \parentheseDeux{\rmD_t F : \omega^1} + \parentheseDeux{F : \rmD_t \omega^1} \eqsp,
\end{align}
where we have used the linearity of $\trace$. Now assume there exists ${k_0}\in \{1,\dots,k-1 \}$ such that \eqref{eq:tensor_derivative} holds for any smooth $1$ forms $\omega^1,\dots,\omega^{k_0}$ and ${l_0}=0$. Moreover, consider any smooth $1$ forms $\omega^1,\dots,\omega^{{k_0}+1}$. Then, using the same arguments as for the case ${k_0}=1$ and the induction hypothesis, we obtain
\begin{align}
\rmD_t \parentheseDeux{F:\omega^{\otimes 1 : ({k_0}+1)}} &= \rmD_t \trace_{(1,l+1)} \parenthese{\parentheseDeux{F: \omega^{\otimes 1 : {k_0}}} \otimes \omega^{{k_0}+1}} \\
&= \trace_{(1,l+1)} \parenthese{\rmD_t \parentheseDeux{F: \omega^{\otimes 1 : {k_0}}} \otimes \omega^{{k_0}+1} } + \trace_{(1,l+1)} \parenthese{ \parentheseDeux{F: \omega^{\otimes 1 : {k_0}}} \otimes \rmD_t \omega^{{k_0}+1} } \\
&= \trace_{(1,l+1)} \parenthese{\parentheseDeux{\rmD_t F:\omega^{\otimes 1 : {k_0}} }\otimes \omega^{{k_0}+1}} + \parentheseDeux{F:\omega^{\otimes 1 : {k_0}} \otimes \rmD_t \omega^{{k_0}+1}}\\
&\quad + \sum_{i=1}^{{k_0}} \trace_{(1,l+1)} \parenthese{\parentheseDeux{F:\omega^{\otimes 1 : (i-1)} \otimes \rmD_t \omega^i \otimes \omega^{\otimes (i+1) : {k_0}} }\otimes \omega^{{k_0}+1}}\\
%&\quad + + \parentheseDeux{F:\omega^{\otimes 1 : {k_0}} \otimes \rmD_t \omega^{{k_0}+1}} \\
& = \parentheseDeux{\rmD_t F:\omega^{\otimes 1 : ({k_0}+1)} } + \sum_{i=1}^{{k_0}+1} \parentheseDeux{F:\omega^{\otimes 1 : (i-1)} \otimes \rmD_t \omega^i \otimes \omega^{\otimes (i+1) : ({k_0}+1)}} \eqsp.
\end{align}
Subsequently, using the recursive definition of the contraction in \eqref{eq:contraction2}, we prove \eqref{eq:tensor_derivative} by induction on ${l_0} \in \nsets,{l_0}\leq l$ for any ${k_0}\leq k$ and any $\omega^1,\dots,\omega^{k_0}\in \frX^{0,1}(\upgamma)$. Let $Y_1\in \frX(\upgamma)$. Then, using once again \Cref{prop:ten_cov_der}-\ref{prop:item:ten_cov_der_c} and \ref{prop:item:ten_cov_der_d}, \eqref{eq:contraction2}, and \eqref{eq:tensor_derivative} in the case ${l_0}=0$ justified above, the case ${l_0}=1$ is proven as follows,
\begin{align}
\rmD_t \parentheseDeux{F: \omega^{\otimes 1:{k_0}} \otimes Y_1} &= \trace_{(k-{k_0}+1,1)}\parenthese{\rmD_t \defEns{\parentheseDeux{F:\omega^{\otimes 1:{k_0}}} \otimes Y_1}} \\
&= \trace_{(k-{k_0}+1,1)}\parenthese{\parentheseDeux{\rmD_t F : \omega^{\otimes 1: {k_0}}}\otimes Y_1} + \parentheseDeux{F: \omega^{\otimes 1:{k_0}} \otimes \rmD_t Y_1} \\
&\quad + \sum_{i=1}^{k_0} \trace_{(k-{k_0}+1,1)}\parenthese{\parentheseDeux{F: \omega^{\otimes 1:(i-1)} \otimes \rmD_t \omega^{i}\otimes \omega^{\otimes (i+1):{k_0}}} \otimes Y_1} \\
%&\quad + \parentheseDeux{F: \omega^{\otimes 1:{k_0}} \otimes \rmD_t Y_1} \\
&= \parentheseDeux{\rmD_t F : \omega^{\otimes 1:{k_0}} \otimes Y_1} 
 + \sum_{i=1}^{{k_0}} \parentheseDeux{F:\omega^{\otimes 1:(i-1)} \otimes \rmD_t \omega^{i} \otimes \omega^{\otimes (i+1):{k_0}} \otimes Y_1} \\
&\quad + \parentheseDeux{F: \omega^{\otimes 1:{k_0}}\otimes \rmD_t Y_1} \eqsp.
\end{align} 
Furthermore, assume there exists ${l_0}\in \{1,\dots, l-1 \}$ such that \eqref{eq:tensor_derivative} holds for any ${k_0} \leq k$, any $\omega^1,\dots , \omega^{k_0}\in \frX^{0,1}(\upgamma)$ and any $Y_1,\dots,Y_{l_0}\in \frX(\upgamma)$. Let $Y_1,\dots, Y_{{l_0}+1}\in \frX(\upgamma)$. Then using the same arguments as for the case ${l_0}=1$ and the induction hypothesis, we obtain
\begin{align}
&\rmD_t \parentheseDeux{F: \omega^{\otimes 1:{k_0}} \otimes Y_{\otimes 1:({l_0}+1)}} \\
&\qquad= \trace_{(k-{k_0}+1,1)} \parenthese{ \rmD_t \defEns{\parentheseDeux{F: \omega^{\otimes 1:{k_0}} \otimes Y_{\otimes 1:{l_0}}} \otimes Y_{{l_0}+1}}} \\
&\qquad= \trace_{(k-{k_0}+1,1)} \parenthese{ \parentheseDeux{\rmD_t F : \omega^{\otimes 1: {k_0}} \otimes Y_{\otimes 1:{l_0}}}\otimes Y_{{l_0}+1}} \\
&\qquad\quad + \sum_{i=1}^{k_0} \trace_{(k-{k_0}+1,1)} \parenthese{ \parentheseDeux{F : \omega^{\otimes 1: (i-1)} \otimes \rmD_t \omega^i \otimes \omega^{\otimes (i+1):{k_0}} \otimes Y_{\otimes 1:{l_0}}}\otimes Y_{{l_0}+1} } \\
&\qquad\quad + \sum_{j=1}^{l_0} \trace_{(k-{k_0}+1,1)} \parenthese{ \parentheseDeux{F : \omega^{\otimes 1: {k_0}} \otimes Y_{\otimes 1:(j-1)} \otimes \rmD_t Y_j  \otimes Y_{\otimes (j+1):{l_0}}}\otimes Y_{{l_0}+1} } \\
&\qquad \quad + \trace_{(k-{k_0}+1,1)} \parenthese{\parentheseDeux{F : \omega^{\otimes 1: {k_0}} \otimes Y_{\otimes 1:{l_0}}}\otimes \rmD_t Y_{{l_0}+1}} \\
&\qquad = \parentheseDeux{\rmD_t F : \omega^{\otimes 1: {k_0}} \otimes Y_{\otimes 1:({l_0}+1)}} + \sum_{i=1}^{k_0} \parentheseDeux{F : \omega^{\otimes 1: (i-1)} \otimes \rmD_t \omega^i \otimes \omega^{\otimes (i+1):{k_0}} \otimes Y_{\otimes 1:({l_0}+1)}} \\
&\qquad \quad + \sum_{j=1}^{{l_0}+1}  \parentheseDeux{F : \omega^{\otimes 1: {k_0}} \otimes Y_{\otimes 1:(j-1)} \otimes \rmD_t Y_j  \otimes Y_{\otimes (j+1):({l_0}+1)}} \eqsp,
\end{align}
which concludes the proof.
\end{proof}

\begin{theorem}\label{lem:taylor_vector}
Let $\msm$ be a smooth manifold and $\nabla$ be a connection on $\planT \msm$. Let $\upgamma:[0,1] \to \msm$ be a geodesic and $Y:\msm \to \planT \msm$ a smooth vector field. Then, for any $t\in [0,1],n \in \nset$,
\begin{equation}\label{eq:taylor_vector}
\begin{aligned}
\parallelTransport_{t0}^\upgamma Y (\upgamma (t)) &=   \sum_{{k_0}=0}^{n} (t^{k_0}/{k_0}!) \left. \nabla^{k_0} Y_{\upgamma(0)} \parenthese{\dot{\upgamma}(0),\dots,\dot{\upgamma}(0)} \right. \\
&\quad + \int_0^t \left. [(t-s)^n/n!] \parallelTransport_{s0}^\upgamma \nabla^{n+1}Y_{\upgamma(s)} \parenthese{\dot{\upgamma}(s),\dots,\dot{\upgamma}(s)}  \right. \rmd s \eqsp,
\end{aligned}
\end{equation}
where $\parallelTransport_{t0}^\upgamma : \planT_{\upgamma(t)} \msm \to \planT_{\upgamma(0)} \msm$ is the parallel transport map along $\upgamma$, and the $(1,{k_0})$-tensor field $\nabla^{k_0} Y$ is the total derivative of order ${k_0}$ of the $(1,0)$-tensor field $Y$.
\end{theorem}
For a definition of the total covariant derivative, see \cite[Proposition 4.15]{lee:2019}. Also, in \eqref{eq:taylor_vector}, remark that even though $\dot{\upgamma}$ is only a vector field along $\upgamma$, and not a vector field, the value of a vector field $\nabla_X Y$ evaluated at $\theta \in \msm$ only depends on $X(\theta)$ and on values of $Y$ along smooth curves $\rmc:[0,1]\to \msm$ satisfying $\rmc(0)=\theta$ and $\dot{\rmc}(0)=X(\theta)$; by \cite[Proposition 4.26]{lee:2019}. Therefore the expression $\nabla^{k_0} Y_{\upgamma(t)} (\dot{\upgamma}(t),\dots,\dot{\upgamma}(t))$ in \Cref{lem:taylor_vector} is well defined for any ${k_0}\in \nset,t \in [0,1]$.
\begin{proof}
Consider $\vt:[0,1]\to \msm$ the smooth vector field along $\upgamma$ and the function $\varphi:[0,1] \to \planT_{\upgamma(0)} \msm$ defined by
\begin{equation}\label{eq:taylor_function}
\vt=Y\circ \upgamma \text{ and } \eqsp\varphi:t\mapsto \parallelTransport_{t0}^\upgamma \vt(t)\eqsp.
\end{equation}
Then we check by induction on $n\in \nsets$ that  $\varphi$ is $n$-times
 differentiable with derivative of order $n$ given for any
$t \in [0,1]$ by
$\varphi^{(n)}(t)= \parallelTransport_{t0}^\upgamma [\rmD_t^n \vt(t)]$ and  $\rmD_t^n \vt (t) = \nabla^n Y_{\upgamma(t)}(\dot{\upgamma}(t),\dots,\dot{\upgamma}(t))$,
where $\rmD_t$ is the covariant derivative operator along $\upgamma$ with respect to the connection $\nabla$, defined in \Cref{prop:ten_cov_der}.

First, the case $n=1$ is a  direct application of \cite[Theorem 4.34, Theorem 4.24]{lee:2019} since $Y$ is an extension of $\vt$. 
Assume now that the property holds for $n \in \nsets$. Then, for any $t_0,t\in [0,1], t\neq t_0$,  we have
\begin{equation}
\left. \parentheseDeux{\varphi^{(n)}(t)-\varphi^{(n)}(t_0)} \middle/ \parenthese{t-t_0} \right. = \parallelTransport_{t_00}^\upgamma \left. \parentheseDeux{\parallelTransport_{tt_0}^\upgamma \rmD^n_t\vt(t)- \rmD^n_t\vt(t_0)} \middle/ \parenthese{t-t_0} \right. \eqsp.
\end{equation}
Now \cite[Theorem 4.34]{lee:2019} ensures that the limit of the quantity above exists when $t \to t_0$ and in addition this limit is
\begin{equation}
\varphi^{(n+1)}(t_0)= \parallelTransport_{t_00}^\upgamma \rmD^{n+1}_t \vt(t_0) \eqsp,
\end{equation}
which shows that $\varphi$ is $n+1$ times differentiable on $\ccint{0,1}$. We now show that for any $t \in \ccint{0,1}$, $\rmD_t^{n+1} \vt (t) = \nabla^{n+1} Y_{\upgamma(t)}(\dot{\upgamma}(t),\dots,\dot{\upgamma}(t))$. Using \Cref{lem:tensor_derivative} on the smooth $(1,n)$-tensor field along $\upgamma$ $F=(\nabla^n Y) \circ \upgamma$, taking ${k_0}=0$ and $n$ times the vector field $\dot{\upgamma}$, we have
\begin{equation}
\rmD_t \parentheseDeux{F : \dot{\upgamma} \otimes \cdots \otimes \dot{\upgamma}} = \parentheseDeux{\rmD_t F : \dot{\upgamma} \otimes \cdots \otimes \dot{\upgamma}}\eqsp,
\end{equation}
%\begin{equation}
%\rmD_t \parentheseDeux{\defEns{(\nabla^n Y)\circ \upgamma} (\dot{\upgamma},\dots,\dot{\upgamma})} = \rmD_t \parentheseDeux{(\nabla^n Y ) \circ \upgamma} (\dot{\upgamma},\dots,\dot{\upgamma}) \eqsp,
%\end{equation}
since $\rmD_t \dot{\upgamma}=0$ because $\upgamma$ is a geodesic. Also, by \eqref{eq:action_contraction}, $[\rmD_tF : \dot{\upgamma} \otimes \cdots \otimes \dot{\upgamma}] = \rmD_t F (\dot{\upgamma},\dots,\dot{\upgamma})$.
Finally, as $\nabla^n Y$ is an extension of $F$, using the induction hypothesis and the definition of the total derivative give for any $t \in [0,1]$,
\begin{multline}
\rmD^{n+1}_t\vt (t) = \rmD_t F\parenthese{\dot{\upgamma},\dots,\dot{\upgamma}} (t) = \nabla_{\dot{\upgamma}(t)} (\nabla^n Y)_{\upgamma(t)} \parenthese{\dot{\upgamma}(t),\dots,\dot{\upgamma}(t)}\\
=   (\nabla^{n+1} Y)_{\upgamma(t)} \parenthese{\dot{\upgamma}(t),\dots,\dot{\upgamma}(t)}\eqsp,
\end{multline}
concluding the induction.

Finally, \eqref{eq:taylor_vector} is simply a consequence of Taylor's formula with integral remainder of the vectorial valued function $\varphi$ identifying $\planT_{\upgamma(0)}\msm$ with $\rset^d$.
\end{proof}
\begin{pproposition}\label{lem:nabla_hess}
Let $\msm$ be a smooth manifold, $\nabla$ be a symmetric connection defined over the smooth vector fields of $\msm$. For any smooth function $f:\msm \to \rset$ and any local coordinates $(u_i)_{ i \in \{1,\dots,d\}}$, we have %the corresponding local frame is denoted $(\partial u_i)_{ i \in \{1,\dots,d\}}$ and its dual coframe $(\rmd u^j)_{ j \in \{1,\dots,d\}}$,
\begin{multline}\label{eq:nabla_hess}
  \nabla \Hess f  = \sum_{i,j,k=1}^d
  \left\{
    \partial^3_{kij}f
    - \sum_{l=1}^d \parentheseDeux{\Gamma_{ij}^l \partial^2_{kl}f
      + \Gamma_{ki}^l \partial^2_{jl}f + \Gamma_{kj}^l\partial^2_{il}f}
    - \sum_{m=1}^d \partial_k \Gamma_{ij}^{m} \partial_m f  \right. \\
  \left.+ \sum_{l,m=1}^d \parentheseDeux{\Gamma_{kj}^l \Gamma_{il}^m
      + \Gamma_{ki}^l \Gamma_{lj}^m} \partial_m f \right\}
  \rmd u^i \otimes \rmd u^j \otimes \rmd u^k \eqsp,
\end{multline}
where $(\Gamma_{ij}^k)_{ i,j,k \in \{1,\dots,d\}}$ are the Christoffel symbols in these local coordinates, the local frame and its dual coframe are denoted by $(\partial u_i)_{ i \in \{1,\dots,d\}}$ and $(\rmd u^j)_{ j \in \{1,\dots,d\}}$.
\end{pproposition}
\begin{proof}
Let $(u_i)_{i \in \{1,\ldots,d\}}$ be local coordinates. By \cite[Example 4.22]{lee:2019}, in this chart, we have
\begin{equation}\label{eq:hess_coordinates}
\Hess f = \sum_{i,j=1}^d F_{ij} \rmd u^i \otimes \rmd u^j \eqsp,\text{where for any } i,j\in \{1,\dots,d\}\eqsp,\eqsp  F_{ij}=\partial_{ij}^2 f - \sum_{m=1}^d \Gamma_{ij}^m \partial_m f\eqsp. 
\end{equation}
Applying \cite[Proposition 4.18]{lee:2019} on $\Hess f$, we obtain that
$\nabla \Hess f = \sum_{i,j,k=1}^d  G_{ijk} \rmd u^i \otimes \rmd u^j \otimes \rmd u^k$,
% \begin{equation}
%  \eqsp,
% \end{equation}
where for any $i,j,k \in \{1, \ldots, d \}$,
\begin{equation}
  G_{ijk}= \partial_k F_{ij}
  - \sum_{l=1}^d \parenthese{\Gamma_{kj}^l F_{il} + \Gamma_{ki}^l F_{lj}} \eqsp.
\end{equation}
Expanding the expression above using \eqref{eq:hess_coordinates} gives for any $i,j,k \in \{1, \dots, d \}$,
\begin{multline}
  G_{ijk} = \partial^3_{ijk} f -
  \sum_{m=1}^d \parenthese{\partial_k \Gamma_{ij}^m \partial_{m} f
    + \Gamma_{ij}^m \partial^2_{k m} f}
  - \sum_{l=1}^d \Gamma_{kj}^l \parenthese{\partial^2_{il} f
    - \sum_{m=1}^d \Gamma_{il}^m \partial_m f} \\
- \sum_{l=1}^d \Gamma_{ki}^l \parenthese{\partial^2_{lj} f - \sum_{m=1}^d \Gamma_{lj}^m \partial_m f} \eqsp.
\end{multline}
The desired result is obtained by reordering this equation, which concludes the proof.
\end{proof}

%\alain{Give a reminder of Riemannian the measure $\mu_\Theta$ associated with the volume form.}

%%% Local Variables:
%%% mode: latex
%%% TeX-master: "main_supplement"
%%% End:

%\hypersetup{citecolor=white, linkcolor=white}
%\textcolor{white}{\eqref{eq:SigmaT}, \eqref{eq:strongconv} \eqref{def:barycenter_noncompact} \eqref{eq:def_D_theta}}
%
% \appendix
% \input{app_notations}

% \input{app_constrained}

% \input{app_unconstrained}

% \input{app_applications}

% \input{app_background}
\end{document}